%% file: elsarticle-template-num.tex
\documentclass[preprint,12pt]{elsarticle}
\makeatletter
\def\ps@pprintTitle{%
  \let\@oddhead\@empty
  \let\@evenhead\@empty
  \let\@oddfoot\@empty
  \let\@evenfoot\@oddfoot}
\makeatother
\usepackage{amssymb}
\usepackage{amsmath}
\journal{}

\usepackage[utf8]{inputenc} % allow utf-8 input
\usepackage[T1]{fontenc}    % use 8-bit T1 fonts
\usepackage{hyperref}       % hyperlinks
\usepackage{url}            % simple URL typesetting
\usepackage{booktabs}       % professional-quality tables
\usepackage{amsfonts}       % blackboard math symbols
\usepackage{nicefrac}       % compact symbols for 1/2, etc.
\usepackage{microtype}      % microtypography
\usepackage{xcolor}         % colors

\usepackage{algorithmic}
\usepackage{algorithm}
\usepackage[english]{babel}
\usepackage{amsmath}
\usepackage{amsthm}
\newtheorem{theorem}{Theorem}
\newtheorem{lemma}{Lemma}
\newtheorem{corollary}{Corollary}
\newtheorem{remark}{Remark}

\newtheorem{definition}{Definition}
\usepackage{amssymb}
\usepackage{graphicx}
\usepackage{epstopdf}
\usepackage{subfigure}
\usepackage{makecell}
\usepackage{wrapfig}
\usepackage{enumitem}

\begin{document}

\begin{frontmatter}



\title{When a Reinforcement Learning Agent Encounters Unknown Unknowns}


\author[uoe,ustc]{Juntian Zhu%\corref{thanks}
} %% Author name
\ead{%zjt1229@mail.ustc.edu.cn; 
%s2760647@ed.ac.uk
J.Zhu-98@sms.ed.ac.uk}
%\cortext[thanks]{This work was completed when J. Zhu was visiting the University of Edinburgh.}
\author[uoe]{Miguel de Carvalho}
% \ead{Miguel.deCarvalho@ed.ac.uk}
\author[ustc]{Zhouwang Yang}
% \ead{yangzw@ustc.edu.cn}
\author[uoe]{Fengxiang He}
% \ead{F.He@ed.ac.uk}


\affiliation[uoe]{organization={University of Edinburgh},%Department and Organization
%            city={Edinburgh},
            country={United Kingdom}
}
            
\affiliation[ustc]{organization={University of Science and Technology of China},%Department and Organization
%            city={Hefei},
            country={China}
}

%% Abstract
\begin{abstract}
An AI agent might surprisingly find she has reached an unknown state which she has never been aware of -- an unknown unknown. 
We mathematically ground this scenario in reinforcement learning: an agent, after taking an action calculated from value functions $Q$ and $V$ defined on the {\it {aware domain}}, reaches a state out of the domain.
To enable the agent to handle this scenario, we propose an {\it episodic Markov decision {process} with growing awareness} (EMDP-GA) model, taking a new {\it noninformative value expansion} (NIVE) approach to expand value functions to newly aware areas: when an agent arrives at an unknown unknown, value functions $Q$ and $V$ whereon are initialised by noninformative beliefs -- the averaged values on the aware domain. This design is out of respect for the complete absence of knowledge in the newly discovered state. The upper confidence bound momentum Q-learning is then adapted to the growing awareness for training the EMDP-GA model. We prove that (1) the regret of our approach is asymptotically consistent with the state of the art (SOTA) without exposure to unknown unknowns in an extremely uncertain environment,
and (2) our computational complexity and space complexity are comparable with the SOTA
-- these collectively suggest that though an unknown unknown is surprising, it will be asymptotically properly discovered with decent speed and an affordable cost.
\end{abstract}


% \begin{highlights}
% \item We propose pisodic Markov decision process with growing awareness (EMDP-GA) to model the unknown unknowns in reinforcement learning.
% \item We design a noninformative value expansion (NIVE) to expand the domain of value functions.
% \item We embed NIVE into upper confidence bound momentum Q-learning (UCBMQ) for training EMDP-GA model, achieving a sublinear regret bound with respect to the number of training episodes $T$.
% \end{highlights}

%% Keywords

\end{frontmatter}

\input{1.introduction}
\input{2.related_works}
\input{3.preliminaries}

\input{4.awareness_of_agent_in_an_uncertain_environment}

\input{5.UCBMQ-GA_algorithm}

\input{6.theory}
% \input{sections/7. experiments}
\input{8.conclusion}

\bibliographystyle{plain}
\bibliography{ref}

\newpage
\appendix
\section{Additional Notations and Preliminaries}
\label{notation_appendix}
\begin{table}[!h]
  \renewcommand{\arraystretch}{1.3}
  \caption{List of Notation}
  \label{notations}
  \centering
  \begin{tabular}{lll}
    \toprule
    Description     & Notation    \\
    \midrule
    Size of the state space  & $S=|\mathcal{S}|$ \\
    Size of the action space  & $A=|\mathcal{A}|$ \\
    Number of learning episodes  & $T$ \\
    Horizon (length of an episode)  & $H$ \\
    Aware domain & $S_t,\ t \in [T]$\\
    immediate reward & $r_h$\\
    Exploration threshold & $\zeta=log(\frac{96eHSA(2T+1)}{\delta})$\\
    Indicator function for the occurrence of $(s,a)$  & $\chi_h^t(s,a)$\\
    Visitation count for $(s,a)$ & $n_h^t(s,a)$\\
    First aware domain including $s$  & $t(s)$ \\
    
    Policy at episode $t$  & $\pi^t$\\
    Optimal policy  & $\pi^*$\\
    Optimal Q-value function  & $Q_h^*(s,a),\ s \in \mathcal{S}, a \in \mathcal{A}$\\
    Optimal value function  & $V_h^*(s),\ s \in \mathcal{S}$\\

    Dirac distribution  concentrated at $(s_{h+1}^t)$  & \thead[l]{$p_h^t(s'|s,a),\ s,s' \in \mathcal{S}_t$ \\ $\mathring{p}_h^t(s'|s,a),\ s,s' \in \mathcal{S}$}\\
    Estimated Q-value function & \thead[l]{$Q_h^t(s,a),\ s \in \mathcal{S}_t,\ a \in \mathcal{A}$ \\ $\tilde{Q}_h^t(s,a),\ s \in \mathcal{S}_{t+1},\ a \in \mathcal{A}$ \\ $\mathring{Q}_h^t(s,a),\ s \in \mathcal{S},\ a \in \mathcal{A}$} \\
    
    Bias-value function & \thead[l]{$V_{h,s,a}^t(s'),\ s,s' \in \mathcal{S}_t$ \\ $\tilde{V}_{h,s,a}^t(s'),\ s,s' \in \mathcal{S}_{t+1}$ \\ $\mathring{V}_{h,s,a}^t(s'),\ s,s' \in \mathcal{S}$}  \\

     Upper bound on Q-value function & \thead[l]{$\overline{Q}_h^t(s,a),\ s \in \mathcal{S}_t,\ a \in \mathcal{A}$ \\ $\tilde{\overline{Q}}_h^t(s,a),\ s \in \mathcal{S}_{t+1},\ a \in \mathcal{A}$ \\ $\mathring{\overline{Q}}_h^t(s,a),\ s \in \mathcal{S},\ a \in \mathcal{A}$}\\

     Upper bound on value function & \thead[l]{$\overline{V}_h^t(s),\ s \in \mathcal{S}_t$ \\ $\tilde{\overline{V}}_h^t(s),\ s \in \mathcal{S}_{t+1}$ \\ $\mathring{\overline{V}}_h^t(s),\ s \in \mathcal{S}$}\\
     
    \bottomrule
  \end{tabular}
\end{table}

For $t \in [T]$, we expand $p_h^t(s'|s,a)$, $Q_h^t(s,a)$, $\overline{V}_h^t(s)$, $V_{h,s,a}^t(s)$, $\overline{Q}_h^t(s,a)$ to the whole state space $\mathcal{S}$.
$$
\mathring{p}_h^t(s'|s,a)=
    \begin{cases}
    p_h^t(s'|s,a), & s,s' \in \mathcal{S}_t,\\
    0, & s \in \mathcal{S}_t, s' \in \mathcal{S}-\mathcal{S}_t,\\
    p_h(s'|s,a), & s \in \mathcal{S}-\mathcal{S}_t,
    \end{cases}
$$

$$
\mathring{Q}_h^t(s,a)=
    \begin{cases}
    Q_h^t(s,a), &s\in \mathcal{S}_t,\\
    Q_{h,avg}^t(a), & s \in \mathcal{S}-\mathcal{S}_t,
    \end{cases}
$$
$$
\mathring{\overline{V}}_h^t(s)=
    \begin{cases}
    \overline{V}^{t}_h(s), & s\in \mathcal{S}_{t},\\
    \overline{V}_{h,avg}^{t}, & s \in \mathcal{S}-\mathcal{S}_{t},
    \end{cases}
$$
$$
\mathring{V}^t_{h,s,a}(s')=
    \begin{cases}
    V^{t}_{h,s,a}(s'), & s,s' \in \mathcal{S}_{t},\\
    V_{h,s,a,avg}^{t}, & s\in \mathcal{S}_{t}, s'\in \mathcal{S}-\mathcal{S}_{t},\\
    V_{h,a,avg}^{t}(s'), & s\in \mathcal{S}-\mathcal{S}_{t}, s'\in \mathcal{S}_{t},\\
    V_{h,a,avg}^{t}, & s,s'\in \mathcal{S}-\mathcal{S}_{t},
    \end{cases}
$$

$$
\mathring{\overline{Q}}_h^t(s,a)=
    \begin{cases}
    Q_h^t(s,a)+\beta_h^t(s,a),& n_h^t(s,a)>0\  \text{(which implies $s\in \mathcal{S}_t$)},\\
    \mathring{Q}_h^t(s,a)+H ,& n_h^t(s,a)=0\ \text{or}\ s \in \mathcal{S}-\mathcal{S}_t.
    \end{cases}
$$
From the definitions, it is easy to see that for any $s, s'$ in the aware domain $ \mathcal{S}_{t+1}$, we have
$$
\mathring{Q}_h^t(s,a)=\tilde{Q}_h^t(s,a), \  
\mathring{\overline{V}}_h^t(s)=\tilde{\overline{V}}_h^t(s), \  
\mathring{V}^t_{h,s,a}(s')=\tilde{V}^t_{h,s,a}(s'),
$$
and for any state $s \in \mathcal{S}_t$,
$$
\mathring{\overline{Q}}_h^t(s,a)=\overline{Q}_h^t(s,a).
$$

%\section{Additional Preliminaries}

When the count of a state-action pair $n_h^t(s,a) > 0$, we can obtain explicit formulas for the estimate of the Q-function:
\begin{align*}
&\ \ \ \ n_h^t(s,a)Q_h^t(s,a)\\
&=\chi_h^t(s,a)[r_h(s,a)+p_h^t \tilde{\overline{V}}_{h+1}^{t-1}(s,a)+\mathring{\gamma}_h^t(s,a)p_h^t(\tilde{\overline{V}}_{h+1}^{t-1}-\tilde{V}_{h,s,a}^{t-1})(s,a)]\\
&+(n_h^t(s,a)-\chi_h^t(s,a))\tilde{Q}_h^{t-1}(s,a),
\end{align*}
where $\mathring{\gamma}_h^t(s,a)=H\frac{n_h^t(s,a)-1}{n_h^t(s,a)+H}$.

For $s \in \mathcal{S}_t-\mathcal{S}_{t-1}$, which means $t(s)=t$, we have
$
n_h^t(s,a)=\chi_h^t(s,a)=1
$, hence $\mathring{\gamma}_h^t(s,a)=0$.
Therefore, the following holds
$$
Q_h^t(s,a)=r_h(s,a)+p_h^t \tilde{\overline{V}}_{h+1}^{t-1}(s,a).
$$
For $s \in \mathcal{S}_{t-1}$, we have
\begin{align*}
&\ \ \ \ n_h^t(s,a)Q_h^t(s,a)\\
&=\chi_h^t(s,a)[r_h(s,a)+p_h^t \tilde{\overline{V}}_{h+1}^{t-1}(s,a)+\mathring{\gamma}_h^t(s,a)p_h^t(\tilde{\overline{V}}_{h+1}^{t-1}-\tilde{V}_{h,s,a}^{t-1})(s,a)]\\
&\ \ \ \ +n_h^{t-1}(s,a)Q_h^{t-1}(s,a),
\end{align*}
\begin{align*}
Q_h^t(s,a)=&r_h(s,a)+\frac{1}{n_h^t(s,a)}\sum_{k=t(s)}^t\chi_h^k(s,a)[p_h^k \tilde{\overline{V}}_{h+1}^{k-1}(s,a)+\\
&\mathring{\gamma}_h^k(s,a)p_h^k(\tilde{\overline{V}}_{h+1}^{k-1}-\tilde{V}_{h,s,a}^{k-1})(s,a)].\\
\end{align*}
From all above, we can obtain explicit formulas for $Q_h^t$,
\begin{align*}
Q_h^t(s,a)=&r_h(s,a)+\frac{1}{n_h^t(s,a)}\sum_{k=t(s)}^t\chi_h^k(s,a)[\mathring{p}_h^k \mathring{\overline{V}}_{h+1}^{k-1}(s,a)\\
&+\mathring{\gamma}_h^k(s,a)\mathring{p}_h^k(\mathring{\overline{V}}_{h+1}^{k-1}-\mathring{V}_{h,s,a}^{k-1})(s,a)].
\end{align*}

We can do the same with the bias value function when $n_h^t(s,a)>0$.
For $s,s' \in \mathcal{S}_t$,
\begin{align*}
\mathring{V}_{h,s,a}^t(s') &=V_{h,s,a}^t(s') \\
&= \eta_h^t(s,a)\tilde{\overline{V}}_{h+1}^{t-1}(s')+(1-\eta_h^t(s,a))\tilde{V}_{h,s,a}^{t-1}(s')\\
&=\eta_h^t(s,a)\mathring{\overline{V}}_{h+1}^{t-1}(s')+(1-\eta_h^t(s,a))\mathring{V}_{h,s,a}^{t-1}(s').
\end{align*}
For $s \in \mathcal{S}, s' \in \mathcal{S}-\mathcal{S}_t$,
\begin{align*}
\mathring{V}_{h,s,a}^t(s') &=V_{h,s,a,avg}^t \\
&=\frac{1}{|\mathcal{S}_t|} \sum_{s' \in \mathcal{S}_t} V_{h,s,a}^t(s')\\
&=\frac{1}{|\mathcal{S}_t|} \sum_{s' \in \mathcal{S}_t} \eta_h^t(s,a)\tilde{\overline{V}}_{h+1}^{t-1}(s')+(1-\eta_h^t(s,a))\tilde{V}_{h,s,a}^{t-1}(s')\\
&=\eta_h^t(s,a)\mathring{\overline{V}}_{h+1}^{t-1}(s')+(1-\eta_h^t(s,a))\mathring{V}_{h,s,a}^{t-1}(s').
\end{align*}
For $s\in \mathcal{S}-\mathcal{S}_t$, we get $n_h^t(s,a)=0$.

We can conclude

\begin{align}
\mathring{V}_{h,s,a}^t(s')
&=\eta_h^t(s,a)\mathring{\overline{V}}_{h+1}^{t-1}(s')+(1-\eta_h^t(s,a))\mathring{V}_{h,s,a}^{t-1}(s') \nonumber \\
&=\frac{1}{n_h^t(s,a)} \sum_{k=t(s)}^{t} \chi_h^k(s,a)[\mathring{\overline{V}}_{h+1}^{k-1}(s')+\mathring{\gamma}_h^k(s,a)(\mathring{\overline{V}}_{h+1}^{k-1}-\mathring{V}_{h,s,a}^{k-1})(s')] \label{V_induction}\\
&=\sum_{k=t(s)}^t \tilde{\eta}_h^{t,k}(s,a)\mathring{\overline{V}}_{h+1}^{k-1}(s'), \label{V_induction_2}
\end{align}

where $\tilde{\eta}_h^{t,k}(s,a)=\eta_h^k(s,a) \prod \limits_{l=k+1}^t (1-\eta_h^l(s,a))$.

We prove the following lemma to have a basic knowledge about $\overline{V}^t_h(s)$ and $V^t_{h,s,a}(s')$.
\begin{lemma}
\label{V_bound}
For $s,s' \in \mathcal{S}_t$, the following holds almost surely:
\begin{itemize}
    \item the sequence $(\overline{V}^t_h(s))_{t \ge t(s)}$ is non-increasing,
    \item $0 \le \overline{V}_h^t(s) \le H$,
    \item $\overline{V}^t_{h+1}(s') \le V^t_{h,s,a}(s') \le H$.
\end{itemize}
\end{lemma}
When $t \ge t(s)$, we have $\tilde{\overline{V}}_h^t=\overline{V}_h^t$.

We can get from the construction of $\overline{V}_h^t(s)$:
$$
\overline{V}_h^t(s)=clip(\mathop{max}\limits_{a \in \mathcal{A}} \overline{Q}_h^t(s,a),0,\tilde{\overline{V}}_h^{t-1}(s))
$$
that 
$$
\tilde{\overline{V}}_h^{t(s)-1}(s) \ge \overline{V}_h^{t(s)}(s) \ge \overline{V}_h^{t(s)+1}(s) \ge \cdots \ge 0.
$$

To prove that $\overline{V}_h^{t}(s) \le H, s \in \mathcal{S}_t$, we proceed by induction.
We have $\overline{V}_h^0=H$, hence this claim holds when $t=0$.
Assuming that $\overline{V}_h^{t-1}(s) \le H, s \in \mathcal{S}_{t-1}$.
For $s \in \mathcal{S}_{t-1} \subset \mathcal{S}_t$, we have
$$
\overline{V}_h^{t}(s) \le \overline{V}_h^{t-1}(s) \le H.
$$
For $s \in \mathcal{S}_t-\mathcal{S}_{t-1}$, we have
$$
\overline{V}_h^{t}(s) \le \tilde{\overline{V}}_h^{t-1}(s)=\frac{1}{|\mathcal{S}_{t-1}|} \sum_{\hat{s} \in \mathcal{S}_{t-1}} \overline{V}_h^{t-1}(\hat{s}) \le H.
$$

To prove the third claim, we proceed by induction. 
When $t=0$, the claims holds.
Assuming $\overline{V}^{t-1}_{h+1}(s') \le V^{t-1}_{h,s,a}(s') \le H$ for $s, s' \in \mathcal{S}_{t-1}$.
We consider two different cases of $s \in \mathcal{S}_t$.

Case 1: $s \in \mathcal{S}_t-\mathcal{S}_{t-1}$, which means $t(s)=t$ and $\chi_h^t(s,a)=n_h^t(s,a) \le 1$.
Since 
$$\eta_h^t(s,a)=\chi_h^t(s,a) \frac{H}{H+n_h^t(s,a)} \frac{n_h^t(s,a)-1}{n_h^t(s,a)}+\chi_h^t(s,a) \frac{1}{n_h^t(s,a)}=0,$$
the algorithm updates the bias-value function $V_{h,s,a}^t(s')=\tilde{V}_{h,s,a}^{t-1}(s')$.

For $s' \in \mathcal{S}_{t-1}$, we have
$$
V_{h,s,a}^t(s')=\tilde{V}_{h,s,a}^{t-1}(s')
=\frac{1}{|\mathcal{S}_{t-1}|} \sum_{\hat{s} \in \mathcal{S}_{t-1}} {V_{h,\hat{s},a}^{t-1}}(s') 
\ge \overline{V}_{h+1}^{t-1}(s')
\ge \overline{V}_{h+1}^{t}(s').
$$
For $s' \in \mathcal{S}_{t}-\mathcal{S}_{t-1}$, which means $t=t(s')$, we have
\begin{align*}
V_{h,s,a}^t(s') &= \tilde{V}_{h,s,a}^{t-1}(s') \\
&= \frac{1}{|\mathcal{S}_{t-1}|^2} \sum_{\hat{s}, \hat{s'} \in \mathcal{S}_{t-1}} {V_{h,\hat{s},a}^{t-1}}(\hat{s'}) \\
&\ge \frac{1}{|\mathcal{S}_{t-1}|} \sum_{\hat{s} \in \mathcal{S}_{t-1}}{\overline{V}_{h+1}^{t-1}(\hat{s})}\\
&=\tilde{\overline{V}}_{h+1}^t(s') \ge \overline{V}_{h+1}^t(s').
\end{align*}

Case 2: $s \in \mathcal{S}_{t-1}$, the algorithm updates the bias-value function
$$
V_{h,s,a}^t(s')=\eta_h^t(s,a) \tilde{\overline{V}}_{h+1}^{t-1}(s')+(1-\eta_h^t(s,a)) \tilde{V}_{h,s,a}^{t-1}(s').
$$
Then we consider the cases of $s' \in \mathcal{S}_{t-1}$ and $s' \in \mathcal{S}_t-\mathcal{S}_{t-1}$.
For $s' \in \mathcal{S}_{t-1}$, we have
\begin{align*}
V_{h,s,a}^t(s')&=\eta_h^t(s,a) \tilde{\overline{V}}_{h+1}^{t-1}(s')+(1-\eta_h^t(s,a)) \tilde{V}_{h,s,a}^{t-1}(s')\\
&=\eta_h^t(s,a) \overline{V}_{h+1}^{t-1}(s')+(1-\eta_h^t(s,a)) V_{h,s,a}^{t-1}(s')\\
&\ge \eta_h^t(s,a) \overline{V}_{h+1}^{t-1}(s')+(1-\eta_h^t(s,a)) \overline{V}_{h+1}^{t-1}(s')\\
&=\overline{V}_{h+1}^{t-1}(s') \ge \overline{V}_{h+1}^{t}(s').
\end{align*}
For $s' \in \mathcal{S}_t-\mathcal{S}_{t-1}$, we have
\begin{align*}
V_{h,s,a}^t(s')&=\eta_h^t(s,a) \tilde{\overline{V}}_{h+1}^{t-1}(s')+(1-\eta_h^t(s,a)) \tilde{V}_{h,s,a}^{t-1}(s')\\
&=\eta_h^t(s,a) \frac{1}{|\mathcal{S}_{t-1}|} \sum_{\hat{s} \in \mathcal{S}_{t-1}} \overline{V}_{h+1}^{t-1}(\hat{s}) +(1-\eta_h^t(s,a)) \frac{1}{|\mathcal{S}_{t-1}|} \sum_{\hat{s} \in \mathcal{S}_{t-1}} V_{h,s,a}^{t-1}(\hat{s}) \\
&\ge \eta_h^t(s,a) \frac{1}{|\mathcal{S}_{t-1}|} \sum_{\hat{s} \in \mathcal{S}_{t-1}} \overline{V}_{h+1}^{t-1}(\hat{s})+(1-\eta_h^t(s,a)) \frac{1}{|\mathcal{S}_{t-1}|} \sum_{\hat{s} \in \mathcal{S}_{t-1}} \overline{V}_{h+1}^{t-1}(\hat{s})\\
&=\tilde{\overline{V}}_{h+1}^{t}(s') \ge \overline{V}_{h+1}^{t}(s').
\end{align*}

%\end{proof}

\section{A Brief Introduction to Upper Confidence Bound Momentum Q-learning}
\label{details_of_UCBMQ}

This Appendix gives an overview of the procedure of UCBMQ.
The notations $\chi_h^t(s,a)$, $n_h^t(s,a)$, $\alpha_h^t(s,a)$ and $\gamma_h^t(s,a)$ are the same in UCBMQ and UCBMQ-GA.
For all $(s,a,h,s') \in \mathcal{S} \times \mathcal{A} \times [H] \times \mathcal{S}$, UCBMQ is initialized as follows:
$$
\overline{V}_h^0(s)=H,\ 
V_{h,s,a}^{0}(s')=H,\ 
\overline{V}_{H+1}^t(s)=0,\ 
Q_h^0(s,a)=0.
$$
In episode $t$, the agent follows a greedy policy based on $\overline{Q}_h^t(s,a)$ to obtain an episode of length $H$.
Next, $Q_h^t(s,a)$, $V_{h,s,a}^{t}(s')$, $\overline{Q}_h^t(s,a)$ and $\overline{V}_h^t(s)$ are updated using the data collected in this episode.
The algorithm terminates after $T$ episodes.

The updates in UCBMQ described as follows.
$Q_h^t(s,a)$ is updated using the momentum term $\gamma_h^t(s,a)$:
\begin{align*}
Q_h^t(s,a)=&\alpha_h^t(s,a)[r_h(s,a)+p_h^t \overline{V}_{h+1}^{t-1}(s,a)]\\
&+\gamma_h^t(s,a)p_h^t(\overline{V}_{h+1}^{t-1}-V_{h,s,a}^{t-1})(s,a)+(1-\alpha_h^t(s,a))Q_h^{t-1}(s,a),
\end{align*}
where $p_h^t$ denotes the Dirac distribution  concentrated at $(s_{h+1}^t)$ and, for any transition function $p$ and any function $f: \mathcal{S} \rightarrow \mathbb{R}$, we define 
$$
pf(s,a)=\mathbb{E}_{s'\sim p(\cdot|s,a)}[f(s')].
$$
$V_{h,s,a}^{t}(s')$ is updated according to
$$
V_{h,s,a}^t(s')=\eta_h^t(s,a) \overline{V}_{h+1}^{t-1}(s')+(1-\eta_h^t(s,a))V_{h,s,a}^{t-1}(s'),
$$
where $$
\eta_h^t(s,a)=\alpha_h^t(s,a)+\gamma_h^t(s,a).
$$
$\overline{Q}_h^t(s,a)$ is used to upper-bound the optimal Q-value function, which is defined as 
$$
\overline{Q}_h^t(s,a)=Q_h^t(s,a)+\beta_h^t(s,a),
$$
where $\beta_h^t(s,a)$ denotes the bonus term.
In particular, if $n_h^t(s,a)=0$ then $\beta_h^t(s,a)=H$; otherwise 
\begin{align*}
\beta_h^t(s,a)=&2 \sqrt{\frac{\zeta W_h^t(s,a)}{n_h^t(s,a)}}+53H^3\frac{\zeta log(T)}{n_h^t(s,a)}\\
&+\sum_{k=1}^t\frac{\chi_h^k(s,a)\mathring{\gamma}_h^k(s,a)}{Hlog(T)n_h^t(s,a)}p_h^k(V_{h,s,a}^{k-1}-\overline{V}_{h+1}^{k-1})(s,a).
\end{align*}
Here, $$
\mathring{\gamma}_h^t(s,a)=H\frac{n_h^t(s,a)-1}{n_h^t(s,a)+H}.
$$ 
$\zeta$ is an exploration threshold and $W_h^t(s,a)$ is a proxy for the variance term defined as 
$$
W_h^t(s,a)=\sum_{k=1}^t \frac{\chi_h^k(s,a)}{n_h^k(s,a)}p_h^k(\overline{V}_{h+1}^{k-1}-\sum_{l=1}^t \frac{\chi_h^l(s,a)}{n_h^l(s,a)}p_h^l\overline{V}_{h+1}^{l-1})^2(s,a).
$$
$\overline{V}_h^t(s)$ is used to upper-bound the optimal value function, which is updated as 
$$
\overline{V}_h^t(s)=clip(\mathop{\max}_{a\in \mathcal{A}}\overline{Q}_h^t(s,a),0,\overline{V}_h^{t-1}(s)).
$$

\section{Algorithm Chart of Upper Confidence Bound Momentum Q-learning with Growing Awareness}
\label{UCBMQ-GA_appendix}
\begin{algorithm}[h]
    \caption{UCBMQ-GA for Training EMDA-GA.}
    \label{alg:UCBMQ_GA}
    \renewcommand{\algorithmicrequire}{\textbf{Input:}}
    \renewcommand{\algorithmicensure}{\textbf{Output:}}
    \begin{algorithmic}[1]
        % \REQUIRE $A$, $B$, $C$  %%input
        % \ENSURE EEEEE    %%output
        
        \STATE  \textbf{Initialise:} For all $(s,a,h)$, where $s \in \mathcal{S}_0 \subset \mathcal{S}$, $V^0_{h,s,a}=\overline{V}^0_h=H$, $Q^0_h=0$

        \FOR{$t \in [T]$}
            \FOR{$h \in [H]$}
                \IF{$s_h^t \in \mathcal{S}_{t-1}$}
                    \STATE Play $a_h^t \in \mathop{\arg\max}{\overline{Q}_h^{t-1}(s_h^t, a)}$
                \ELSE
                    \STATE Play $a_h^t \in \mathop{\arg\max}{{Q}_{h,avg}^{t-1}(a)}$ 
                \ENDIF
                \STATE Observe $s_{h+1}^t$
            \ENDFOR
            
            \STATE $\mathcal{S}_t=\mathcal{S}_{t-1} \cup \{ s_1^t, \cdots , s_H^t \}$

            \STATE Expand $Q^{t-1}_h$, $\overline{V}^{t-1}_h$ and $V^{t-1}_{h,s,a}$:
            
            \begin{align*}
            \tilde{Q}^{t-1}_h(s,a)=
                \begin{cases}
                    Q^{t-1}_h(s,a), & s \in \mathcal{S}_{t-1},\\
                    Q^{t-1}_{h,avg}(a), & s \in \mathcal{S}_t-\mathcal{S}_{t-1},
                \end{cases} 
            % \label{expanded_1}
            \end{align*}
            
            \begin{align*}
            \tilde{\overline{V}}^{t-1}_h(s)=
                \begin{cases}
                    \overline{V}^{t-1}_h(s), & s\in \mathcal{S}_{t-1},\\
                    \overline{V}_{h,avg}^{t-1}, & s \in \mathcal{S}_t-\mathcal{S}_{t-1},
                \end{cases}
            % \label{expanded_2}
            \end{align*}

            \begin{align*}
            \tilde{V}^{t-1}_{h,s,a}(s')=
                \begin{cases}
                    V^{t-1}_{h,s,a}(s'), & s,s' \in \mathcal{S}_{t-1},\\
                    V_{h,s,a,avg}^{t-1}, & s\in \mathcal{S}_{t-1}, s'\in \mathcal{S}_t-\mathcal{S}_{t-1},\\
                    V_{h,a,avg}^{t-1}(s'), & s\in \mathcal{S}_t-\mathcal{S}_{t-1}, s'\in \mathcal{S}_{t-1},\\
                    V_{h,a,avg}^{t-1}, & s,s'\in \mathcal{S}_t-\mathcal{S}_{t-1}.
                \end{cases}
            % \label{expanded_3}
            \end{align*}

            \FOR {all $s \in S_t, a, h$}
                \STATE $Q_h^t(s,a) = \alpha_h^t(s,a)(r_h(s,a)+p_h^t\tilde{\overline{V}}_{h+1}^{t-1}(s,a))+\gamma_h^t(s,a)p_h^t(\tilde{\overline{V}}_{h+1}^{t-1}-\tilde{V}_{h,s,a}^{t-1})(s,a)+(1-\alpha_h^t(s,a))\tilde{Q}_h^{t-1}(s,a)$
                \STATE $V_{h,s,a}^t(s') = \eta_h^t(s,a)\tilde{\overline{V}}_{h+1}^{t-1}(s')+(1-\eta_h^t(s,a))\tilde{V}_{h,s,a}^{t-1}(s')$
                \STATE $\overline{Q}_h^t(s,a) = Q_h^t(s,a)+\beta_h^t(s,a)$
                \STATE $\overline{V}_h^t(s)=clip(\mathop{max}\limits_{a \in \mathcal{A}} \overline{Q}_h^t(s,a),0,\tilde{\overline{V}}_h^{t-1}(s))$
            \ENDFOR
        \ENDFOR    
    \end{algorithmic}
\end{algorithm}

\section{When a New Observation Hurts Awareness Confidence -- A Reverse Bayesian Perspective}
\label{confidence_appendix}
Encountering an unknown unknown clearly tells an agent her currently understanding to the environment is problematic. If the agent were a human, this surprise should hurt severely her confidence. We want to model this scenario in agent, and below we show our approach exactly has this property.

We first define a {\it awareness confidence} below, as a quantitative measure of a reinforcement learning agent's confidence to her understanding to the environment, which underpins her confidence to her decisions.

\begin{definition}[awareness confidence of an agent]
Suppose an agent has an estimator to the hidden optimal value function $V^*$ over a subset of the state space $\mathcal{S}$, defined as function $F: \mathcal{S}' \rightarrow \mathbb{R}$, where $ \mathcal{S}' \subset \mathcal{S}$.
The awareness confidence of this agent is defined as the average absolute difference between $F$ and $V^*$ over $\mathcal{S}'$: $$
AC(F)=-\frac{1}{|\mathcal{S}'|}\sum_{s \in \mathcal{S}'}|F(s)-V^*(s)|.
$$
\end{definition}

\begin{remark}
The awareness confidence quantifies the potential for improving the estimator. 
% \fh{xx}
{
By definition, low awareness confidence indicates that the estimator deviates substantially from the optimal value function, implying that the agent should remain not confident of its estimators.
Conversely, high awareness confidence indicates that the estimator closely approximates the optimal value function, thereby justifying the agent’s confidence.
}

\end{remark}

We now prove that the awareness confidence drastically increases when an unknown unknown is encountered.

\begin{theorem}
\label{uncertain_lemma}
For $t \in [T]$ and $h \in [H]$, where $T$ is the number of learning episodes and $H$ is the horizon, with probability at least $1-\frac{\delta}{2}$, the awareness confidence associated with $\overline{V}_h^t$ is lower than that associated with $\tilde{\overline{V}}_h^t$:
$$
AC(\overline{V}_h^t)<AC(\tilde{\overline{V}}_h^t).
$$
\end{theorem}

\begin{proof}

From Lemma \ref{optimism_lemma}, we have
$$AC(\tilde{\overline{V}}_h^t)=\frac{1}{|\mathcal{S}_{t+1}|}\sum_{s \in \mathcal{S}_{t+1}}|\tilde{\overline{V}}_h^t(s)-V^*(s)|=\frac{1}{|\mathcal{S}_{t+1}|}\sum_{s \in \mathcal{S}_{t+1}}\tilde{\overline{V}}_h^t(s)-V^*(s),$$
$$AC(\overline{V}_h^t)=\frac{1}{|\mathcal{S}_t|}\sum_{s \in \mathcal{S}_t}|\overline{V}_h^t(s)-V^*(s)|=\frac{1}{|\mathcal{S}_t|}\sum_{s \in \mathcal{S}_t}\overline{V}_h^t(s)-V^*(s).$$
With Homeland condition in Section \ref{sec:EMDP-GA}, we have
$$\forall s \in \mathcal{S}_{t+1}-\mathcal{S}_t, V^*_h(s) \le \frac{1}{|\mathcal{S}_{t}|}\sum_{\hat{s} \in \mathcal{S}_{t}} V^*_h(\hat{s}) \le \frac{1}{|\mathcal{S}_{t}|}\sum_{\hat{s} \in \mathcal{S}_{t}} \overline{V}_h^t(\hat{s})=\tilde{\overline{V}}_h^t(s).$$
Then we get
\begin{align*}
&\ \ \ \ \frac{1}{|\mathcal{S}_{t+1}-\mathcal{S}_t|}\sum_{\hat{s} \in \mathcal{S}_{t+1}-\mathcal{S}_t} \tilde{\overline{V}}_h^t(\hat{s})-V^*_h(\hat{s})\\
&\ge \frac{1}{|\mathcal{S}_{t+1}-\mathcal{S}_t|}\sum_{\hat{s} \in \mathcal{S}_{t+1}-\mathcal{S}_t} [\frac{1}{|\mathcal{S}_{t}|}\sum_{s' \in \mathcal{S}_{t}} \overline{V}_h^t(s')-V^*_h(s')]\\
&=\frac{1}{|\mathcal{S}_t|}\sum_{s \in \mathcal{S}_t}\overline{V}_h^t(s)-V^*(s).
\end{align*}
Thus, we can conclude $AC(\tilde{\overline{V}}_h^t) \ge AC(\overline{V}_h^t)$.
\end{proof}

\begin{remark}
This theorem provides a mathematical characterisation to our intuition mentioned in the beginning of this section.
In the view of the agent, the environment becomes larger but the agent's knowledge about the environment is not able to increase at the same time.
The knowledge acquired is thus diluted,
leading to a degradation in average estimation accuracy, reflecting in the awareness confidence.
\end{remark}

\section{Proof of Corollary \ref{corollary:scalar}}
\label{proof_of_scalar}
%We prove this lemma under Assumption \ref{optimal_policy_assumption} and on the event $\mathcal{E}$.

\begin{proof}[Proof of Corollary \ref{corollary:scalar}]
Here we give a special case where $d \overline{V}_{h,avg}^t\ (d < 1)$ no longer upper-bounds the optimal function on a newly discovered state.
Consider an EMDP with a fixed immediate reward function: $\forall (s,a,) \in \mathcal{S} \times \mathcal{A}, r_h(s,a)=1$.
It is easy to get that for every state $s \in \mathcal{S}$, the optimal value function is $V_h^*(s)=H$.
According to Lemma \ref{V_bound}, we have $\overline{V}_{h,avg}^t \le H$ for all $h \in [H]$.
Therefore, for any $d<1$, we get: 
$\forall s \in \mathcal{S},\ d \overline{V}_{h,avg}^t < H = V_h^*(s)$.
\end{proof}

% \section{Proof of Theorem \ref{first_existence_lemma}}
% \label{proof_of_first_existence_lemma}
% %\begin{lemma}
% %\label{first_existence_lemma}
% %With probability at least $1-\frac{\delta}{3}$, the following holds
% %$$\forall s \in \mathcal{S}, t(s)\le \sqrt{T}.$$
% %\end{lemma}

% \begin{proof}[Proof of Theorem \ref{first_existence_lemma}]
% We prove the lemma under Assumption \ref{transition_function_assumption}.
% For $s' \in \mathcal{S}_0$, we have $t(s')=0 \le \sqrt{T}$.
% For $s' \in \mathcal{S}-\mathcal{S}_0$, we have $\forall (s,a) \in \mathcal{S} \times \mathcal{A}, P(s'|s,a) \ge 1-(\frac{\delta}{3})^{\frac{1}{H[\sqrt{T}]}} \triangleq \theta$.
% Thus, the following holds $\mathbb{P}(s' \notin O_{[\sqrt{T}]}) \le (1-\theta)^{H[\sqrt{T}]}$.
% From the relationship between $O_t$ and $\mathcal{S}_t$, we get $$\mathbb{P}(s' \notin \mathcal{S}_{[\sqrt{T}]}) \le (1-\theta)^{H[\sqrt{T}]}=\frac{\delta}{3}.$$
% Finally, we get $t(s') \le [\sqrt{T}] \le \sqrt{T}$ with probability at least $1-\frac{\delta}{3}$.
% \end{proof}

\section{Proof of Lemma \ref{computation_lemma} and Lemma \ref{space_lemma}}
\label{complexity_proof}
\begin{proof}[Proof of Lemma \ref{computation_lemma}]
We prove the computational complexity of UCBMQ-GA by analysing the cost of each time-step.
The computational complexity of UCBMQ-GA per time step comprises function updates and average‐value maintenance, which can be performed concurrently.
At time-step $h$ of an episode $t$, we only need to update the functions concerning $(s_h^t, a_h^t)$ since for other $(s,a)$, $\alpha_h^t(s,a)=0$ and $\gamma_h^t(s,a)=0$.
Thus, updating ${Q}_h^t$ and $\overline{Q}_h^t$ takes $\mathcal{O}(1)$ time
Updating the bias-value function $V_{h,s,a}^t$ takes $\mathcal{O}(S)$ time because the update is performing linear interpolation over a vector of length $S$.
Updating $\overline{V}_h^t$ on $s_h^t$  involves taking the maximum over all $A$, which takes $\mathcal{O}(A)$ time.
We can calculate average values every time the functions are updated, and the calculation takes $\mathcal{O}(1)$ time. 
Hence, each step of updating functions costs $\mathcal{O}(S+A)$, each step of calculating average values costs $\mathcal{O}(1)$ and the total cost is $\mathcal{O}HT(S+A)$.
\end{proof}

\begin{proof}[Proof of Lemma \ref{space_lemma}]
UCBMQ-GA needs space to store the functions.
${Q}_h^t$ and $\overline{Q}_h^t$ both require $\mathcal{O}(SAH)$ space.
$\overline{V}_h^t$ needs $\mathcal{O}(HS)$ space, while $V_{h,s,a}^t$ needs $\mathcal{O}(HS^2A)$ space.
So the space complexity is $\mathcal{O}(HS^2A)$.
\end{proof}

\section{Concentration Events}

We firstly introduce the Bernstein-type concentration inequality.

\begin{theorem}[Bernstein-type concentration inequality \cite{domingues2020regret}]
\label{concentration_inequality}
Let $(Y_t)_{t\in \mathbb{N}^*}, (\omega_t)_{t\in \mathbb{N}^*}$ be two sequences of random variables adapted to a filtration $(\mathcal{F}_t)_{t \in \mathbb{N}}$. 
We assume that the weights are in the unit interval $\omega_t \in [0,1]$ and predictable, i.e. $\mathcal{F}_{t-1}$ measurable.
We also assume that the random variables $Y_t$ are bounded $|Y_t| \le b$ and centered $\mathbb{E}[Y_t|\mathcal{F}_{t-1}]=0$.
Consider the following quantities
$$
S_t\triangleq \sum_{s=1}^t \omega_sY_s,\ 
V_t\triangleq \sum_{s=1}^t \omega_s^2 \mathbb{E}[Y_t^2|\mathcal{F}_{s-1}],\ 
W_t\triangleq\sum_{s=1}^t\omega_s,
$$
and let $h(s)\triangleq (x+1)log(x+1)-x$ be the Cr\'{a}mer transform of a Poisson distribution of parameter $1$.
For all $\delta >0$, 
$$
\mathbb{P}[\exists t \ge 1, (\frac{V_t}{b^2})h(\frac{b|S_t|}{V_t+b^2}) \ge log(\frac{1}{\delta})+log(4e(2t+1))] < \delta.
$$
The previous inequality can be weakened to obtain a more explicit bound: if $b\ge1$ with probability at least $1-\delta$, for all $t\ge1$,
$$
|S_t| \le \sqrt{2V_tlog(\frac{4e(2t+1)}{\delta})}+3blog(\frac{4e(2t+1)}{\delta}).
$$
\end{theorem}

We define the favorable events as follows:
\begin{align*}
\mathcal{E}^{v_1}\triangleq\{ &\forall t \in [T], \forall h \in [H], \forall (s,a) \in \mathcal{S}_t \times \mathcal{A}:\\
&|\sum_{k=t(s)}^t \chi_h^k(s,a)(\mathring{p}_h^k-p_h)\mathring{\overline{V}}_{h+1}^{k-1}(s,a)| \\
&\le \sqrt{2 \zeta \sum_{k=t(s)}^t\chi_h^k(s,a)Var_{p_h}(\mathring{\overline{V}}_h^{k-1})(s,a)}+6H\zeta\},
\end{align*}

\begin{align*}
\mathcal{E}^{v_2}\triangleq\{ &\forall t \in [T], \forall h \in [H], \forall (s,a) \in \mathcal{S}_t \times \mathcal{A}: \\
&|\sum_{k=t(s)}^t \chi_h^k(s,a)(\mathring{p}_h^k-p_h)(\mathring{\overline{V}}_{h+1}^{k-1})^2(s,a)|\\
& \le \sqrt{8H^2 \zeta \sum_{k=t(s)}^t\chi_h^k(s,a)Var_{p_h}(\mathring{\overline{V}}_{h+1}^{k-1})(s,a)}+12H^2\zeta\},
\end{align*}

\begin{align*}
\mathcal{E}^{m_\omega}\triangleq\{ &\forall t \in [T], \forall h \in [H], \forall (s,a) \in \mathcal{S}_t \times \mathcal{A}:\\
&|\sum_{k=t(s)}^t \chi_h^k(s,a)\mathring{\gamma}_h^k(s,a)(\mathring{p}_h^k-p_h)(\mathring{\overline{V}}_{h+1}^{k-1}-\mathring{V}^{k-1}_{h,s,a})(s,a)| \\
&\le \sqrt{2 \zeta \sum_{k=t(s)}^t\chi_h^k(s,a)\mathring{\gamma}_h^k(s,a)^2Var_{p_h}(\mathring{\overline{V}}_{h+1}^{k-1}-\mathring{V}^{k-1}_{h,s,a})(s,a)}+6H^2\zeta\},
\end{align*}

\begin{align*}
\mathcal{E}^{m}\triangleq\{ &\forall t \in [T], \forall h \in [H], \forall (s,a) \in \mathcal{S}_t \times \mathcal{A}:\\
&|\sum_{k=t(s)}^t \chi_h^k(s,a)(\mathring{p}_h^k-p_h)(\mathring{\overline{V}}_{h+1}^{k-1}-\mathring{V}^{k-1}_{h,s,a})(s,a)| \\
&\le \sqrt{2 \zeta \sum_{k=t(s)}^t\chi_h^k(s,a)Var_{p_h}(\mathring{\overline{V}}_{h+1}^{k-1}-\mathring{V}^{k-1}_{h,s,a})(s,a)}+6H\zeta\}.
\end{align*}

We define $\mathcal{E}=\mathcal{E}^{v_1} \cap \mathcal{E}^{v_2} \cap \mathcal{E}^{m_\omega} \cap \mathcal{E}^{m}$ the intersection of all the events above.
$\mathcal{E}$ holds with high probability.
 \begin{lemma}
 \label{P(E)}
 For the choice
 $$
 \zeta=log(\frac{96e(2T+1)}{\delta}),
 $$
 the following holds $\mathbb{P}(\mathcal{E})\ge 1-\frac{\delta}{6}$.
 \end{lemma}
\begin{proof}
From Lemma \ref{V_bound} we can get
$$
\forall s,s' \in \mathcal{S},\ \forall t \in [T],\ 0 \le \mathring{\overline{V}}_{h+1}^{t}(s) \le H,\ 
0 \le \mathring{V}^{t}_{h,s,a}(s')-\mathring{\overline{V}}_{h+1}^{t}(s') \le H.
$$
Therefore, for $s\in \mathcal{S}_k$
$$
|(\mathring{p}_h^k-p_h)\mathring{\overline{V}}_{h+1}^{k-1}(s,a)|\le 2H,
$$
$$
|(\mathring{p}_h^k-p_h)(\mathring{\overline{V}}_{h+1}^{k-1})^2(s,a)|\le 4H^2,
$$
$$
|(\mathring{p}_h^k-p_h)(\mathring{\overline{V}}_{h+1}^{k-1}-\mathring{V}^{k-1}_{h,s,a})(s,a)| \le 2H.
$$
For $\mathcal{E}^{v_2}$, based on Lemma \ref{Variance_inequality}, we have 
$$Var_{p_h}(\mathring{\overline{V}}_{h+1}^{k-1})^2(s,a) \le 2H^2 Var_{p_h}(\mathring{\overline{V}}_{h+1}^{k-1})(s,a).$$
Based on $\zeta=log(\frac{32e(2T+1)}{\delta})$ and Theorem \ref{concentration_inequality}, we get
$$
\mathbb{P}((\mathcal{E}^{v_1})^c)\le \frac{\delta}{24},\ \ \ 
\mathbb{P}((\mathcal{E}^{v_2})^c)\le \frac{\delta}{24},\ \ \ 
\mathbb{P}((\mathcal{E}^{m_{\omega}})^c)\le \frac{\delta}{24},\ \ \ 
\mathbb{P}((\mathcal{E}^{m})^c)\le \frac{\delta}{24}. 
$$
Combine all above, we can conclude $\mathbb{P}(\mathcal{E})\ge 1-\frac{\delta}{6}$.
\end{proof}

\begin{lemma}
\label{E_lemma}
On the event $\mathcal{E}$, $\forall t \in [T],\ \forall h \in [H],\ \forall (s,a)\in \mathcal{S}_t\times \mathcal{A}$, the following holds
\begin{align*}
&\ \ \ \ |\sum_{k=t(s)}^t \chi_h^k(s,a)\mathring{\gamma}_h^k(s,a) (\mathring{p}_h^k-p_h)(\mathring{\overline{V}}_{h+1}^{k-1}-\mathring{V}^{k-1}_{h,s,a})(s,a)|\\
&\le \frac{1}{4Hlog(T)}\sum_{k=t(s)}^t \chi_h^k(s,a)\mathring{\gamma}_h^k(s,a) p_h(\mathring{V}^{k-1}_{h,s,a}-\mathring{\overline{V}}_{h+1}^{k-1})(s,a)+14H^3log(T)\zeta,
\end{align*}
\begin{align*}
&\ \ \ \ |\sum_{k=t(s)}^t \chi_h^k(s,a)(\mathring{p}_h^k-p_h)(\mathring{\overline{V}}_{h+1}^{k-1}-\mathring{V}^{k-1}_{h,s,a})(s,a)| \\
&\le \frac{1}{4} \sum_{k=t(s)}^t \chi_h^k(s,a)p_h(\mathring{V}^{k-1}_{h,s,a}-\mathring{\overline{V}}_{h+1}^{k-1})(s,a)+14H\zeta,
\end{align*}
and
\begin{align*}
&\ \ \ \ |\sum_{k=t(s)}^t \chi_h^k(s,a)(\mathring{p}_h^k-p_h)(\mathring{\overline{V}}_{h+1}^{k-1})^2(s,a)| \\
&\le \frac{1}{4} \sum_{k=t(s)}^t \chi_h^k(s,a)Var_{p_h}(\mathring{\overline{V}}_{h+1}^{k-1})(s,a)+44H^2\zeta.
\end{align*}
\end{lemma}
\begin{proof}
For the first claim, we use the fact that $\mathcal{E}^{m_\omega}$ holds and for all $s \in \mathcal{S}_t, s' \in \mathcal{S}$ we have
$$
\mathring{\gamma}_h^t(s,a) \le H,\ 
0 \le \mathring{V}_{h,s,a}^{t-1}(s')-\mathring{V}_{h+1}^{t-1}(s')\le H,\ 
\sqrt{xy} \le x+y,
$$
and get
\begin{align*}
&\ \ \ \ |\sum_{k=t(s)}^t \chi_h^k(s,a)\mathring{\gamma}_h^k(s,a) (\mathring{p}_h^k-p_h)(\mathring{\overline{V}}_{h+1}^{k-1}-\mathring{V}^{k-1}_{h,s,a})(s,a)| \\
&\le \sqrt{2 \zeta H\sum_{k=t(s)}^t\chi_h^k(s,a)\mathring{\gamma}_h^k(s,a)Var_{p_h}(\mathring{\overline{V}}_{h+1}^{k-1}-\mathring{V}^{k-1}_{h,s,a})(s,a)}+6H^2\zeta\\
&\le \sqrt{2 \zeta H^2\sum_{k=t(s)}^t\chi_h^k(s,a)\mathring{\gamma}_h^k(s,a)p_h(\mathring{V}^{k-1}_{h,s,a}-\mathring{\overline{V}}_{h+1}^{k-1})(s,a)}+6H^2\zeta\\
& \le \frac{1}{4Hlog(T)} \sum_{k=t(s)}^t\chi_h^k(s,a)\mathring{\gamma}_h^k(s,a)p_h(\mathring{V}^{k-1}_{h,s,a}-\mathring{\overline{V}}_{h+1}^{k-1})(s,a)+14H^3log(T)\zeta.
\end{align*}
For the second claim, on the event $\mathcal{E}^m$, we have
\begin{align*}
&\ \ \ \ |\sum_{k=t(s)}^t \chi_h^k(s,a)(\mathring{p}_h^k-p_h)(\mathring{\overline{V}}_{h+1}^{k-1}-\mathring{V}^{k-1}_{h,s,a})(s,a)|\\
&\le \sqrt{2 \zeta \sum_{k=t(s)}^t\chi_h^k(s,a)Var_{p_h}(\mathring{\overline{V}}_{h+1}^{k-1}-\mathring{V}^{k-1}_{h,s,a})(s,a)}+6H\zeta\\
&\le \sqrt{2 \zeta H\sum_{k=t(s)}^t\chi_h^k(s,a)p_h(\mathring{\overline{V}}_{h+1}^{k-1}-\mathring{V}^{k-1}_{h,s,a})(s,a)}+6H\zeta\\
&\le \frac{1}{4} \sum_{k=t(s)}^t\chi_h^k(s,a)p_h(\mathring{V}^{k-1}_{h,s,a}-\mathring{\overline{V}}_{h+1}^{k-1})(s,a) +14 H\zeta.
\end{align*}
For the third claim, we use event $\mathcal{E}^{v_2}$: 
\begin{align*}
&\ \ \ \ |\sum_{k=t(s)}^t \chi_h^k(s,a)(\mathring{p}_h^k-p_h)(\mathring{\overline{V}}_{h+1}^{k-1})^2(s,a)|\\
& \le \sqrt{8H^2 \zeta \sum_{k=t(s)}^t\chi_h^k(s,a)Var_{p_h}(\mathring{\overline{V}}_{h+1}^{k-1})(s,a)}+12H^2\zeta\\
&\le \frac{1}{4}\sum_{k=t(s)}^t\chi_h^k(s,a)Var_{p_h}(\mathring{\overline{V}}_{h+1}^{k-1})(s,a)+44H^2\zeta.
\end{align*}
This concludes the proof.
\end{proof}

For $s \in \mathcal{S}$, Let $\overline{p}_h^t(s,a)$ and $\overline{p}_h^t(s)$ denote the probabilities to reach $(s,a)$ and $s$, respectively, at the time-step $h$ under the policy $\pi^t$ in the algorithm.
We define the following events:
\begin{align*}
\mathcal{G}^{Var}\triangleq\{& \forall t \in [T], \forall h \in [H], \forall (s,a) \in  \mathcal{S}_t \times \mathcal{A}: \\
&|\sum_{k=t(s)}^t(\chi_h^k-\overline{p}_h^k)(s,a)Var_{p_h}(V_{h+1}^{\pi^t})(s,a)|\\
&\le \sum_{k=t(s)}^t \overline{p}_h^k(s,a)Var_{p_h}(V_{h+1}^{\pi^t})(s,a)+8H^2\zeta
\},
\end{align*}

\begin{align*}
\mathcal{G}^{v_1}\triangleq\{& \forall t \in [T], \forall h \in [H], \forall (s,a) \in  \mathcal{S}_t \times \mathcal{A}: \\
&|\sum_{k=t(s)}^t (\chi_h^k-\overline{p}_h^k)(s,a)p_h(\mathring{\overline{V}}_{h+1}^{k-1}-V_{h+1}^{\pi^t})(s,a)| \\
&\le \frac{1}{4H} \sum_{k=t(s)}^t\overline{p}_n^k(s,a)p_h|\mathring{\overline{V}}_{h+1}^{k-1}-V_{h+1}^{\pi^t}|(s,a)+14H^2\zeta\},
\end{align*}

\begin{align*}
\mathcal{G}^{v_2} \triangleq\{& \forall t \in [T], \forall h \in [H], \forall s \in  \mathcal{S}_t : \\
&|\sum_{k=t(s)}^t (\chi_h^k-\overline{p}_h^k)(s)(\mathring{\overline{V}}_{h}^{k-1}-V_{h}^{\pi^t})(s)| \\
&\le \frac{1}{4H} \sum_{k=t(s)}^t\overline{p}_n^k(s)|\mathring{\overline{V}}_{h+1}^{k-1}-V_{h+1}^{\pi^t}|(s)+14H^2\zeta\},
\end{align*}
We define $\mathcal{G}=\mathcal{G}^{Var} \cap \mathcal{G}^{v_1} \cap \mathcal{G}^{v_2}$ the intersection of the events.
This event holds with high probability.
\begin{lemma}
\label{P(G)}
For the choice
$$
\zeta=log(\frac{96eHSA(2T+1)}{\delta}),
$$
the following holds $\mathbb{P}(\mathcal{G}) \ge 1-\frac{\delta}{6}$.
\end{lemma}

\begin{proof}
Based on Theorem \ref{concentration_inequality}, with probability at least $1-\frac{\delta}{24}$, $\forall t \in [T], \forall h \in [H], \forall (s,a) \in  \mathcal{S}_t \times \mathcal{A}$, using the facts that for $X \sim Ber(q)$ the following holds $Var(X)=q(1-q) \le q$ and $\sqrt{xy} \le x+y$,
\begin{align*}
&\ \ \ \ |\sum_{k=t(s)}^t(\chi_h^k-\overline{p}_h^k)(s,a)Var_{p_h}(V_{h+1}^{\pi^t})(s,a)|\\
&\le \sqrt{2\zeta\sum_{k=t(s)}^t \overline{p}_h^k(s,a)Var_{p_h}(V_{h+1}^{\pi^t})(s,a)^2}+6H^2\zeta\\
&\le\sum_{k=t(s)}^t \overline{p}_h^k(s,a)Var_{p_h}(V_{h+1}^{\pi^t})(s,a)+8H^2\zeta.
\end{align*}
So, we have
\begin{align*}
&\ \ \ \ |\sum_{k=t(s)}^t (\chi_h^k-\overline{p}_h^k)(s,a)p_h(\mathring{\overline{V}}_{h+1}^{k-1}-V_{h+1}^{\pi^t})(s,a)|\\
&\le\sqrt{2\zeta \sum_{k=t(s)}^t \overline{p}_n^k(s,a)p_h|\mathring{\overline{V}}_{h+1}^{k-1}-V_{h+1}^{\pi^t}|(s,a)^2}+6\zeta H^2\\
&\le \frac{1}{4H} \sum_{k=t(s)}^t\overline{p}_n^k(s,a)p_h|\mathring{\overline{V}}_{h+1}^{k-1}-V_{h+1}^{\pi^t}|(s,a)+14H^2\zeta\}.
\end{align*}
Further,
\begin{align*}
&\ \ \ \ |\sum_{k=t(s)}^t (\chi_h^k-\overline{p}_h^k)(s)(\mathring{\overline{V}}_{h}^{k-1}-V_{h}^{\pi^t})(s)|\\ 
&\le \sqrt{2\zeta \sum_{k=t(s)}^t\overline{p}_n^k(s)|\mathring{\overline{V}}_{h+1}^{k-1}-V_{h+1}^{\pi^t}|(s)^2}+6\zeta H^2\\
&\le \frac{1}{4H} \sum_{k=t(s)}^t\overline{p}_n^k(s)|\mathring{\overline{V}}_{h+1}^{k-1}-V_{h+1}^{\pi^t}|(s)+14H^2\zeta.
\end{align*}

Therefore, we get 
$$
\mathbb{P}((\mathcal{G}^{Var})^c) \le \frac{\delta}{24},\ \mathbb{P}((\mathcal{G}^{V_1})^c) \le \frac{\delta}{24},\ \mathbb{P}((\mathcal{G}^{V_2})^c) \le \frac{\delta}{24}.
$$
We can conclude that $\mathbb{P}(\mathcal{G}) \ge 1-\frac{\delta}{6}$.
\end{proof}
We define the event $\mathcal{D}=\mathcal{E} \cap \mathcal{G}$.
Combine Lemma \ref{P(E)} and Lemma \ref{P(G)}, we get:
\begin{lemma}
\label{P(D)}
For the choice
$$
\zeta=log(\frac{96eHSA(2T+1)}{\delta}),
$$
the following holds $\mathbb{P}(\mathcal{D}) \ge 1-\frac{\delta}{3}$.
\end{lemma}

\section{Optimism}
\label{optimism}
We use $r_h(s,a)+p_h\mathring{V}_{h,s,a}^t(s,a)$ to estimate $Q_h^t(s,a)$.
The gap between $r_h(s,a)+p_h\mathring{V}_{h,s,a}^t(s,a)$ and $Q_h^t(s,a)$ proves to be controlled by $\beta_h^t(s,a)$.
Finally, we can use $\mathring{\overline{Q}}_h^t(s,a)$ to upper-bound $Q_h^*(s,a)$ and $\mathring{\overline{V}}_h^t(s)$ to upper-bound $V_h^*(s,a)$.

\begin{lemma}
\label{Q_estimate}
On the event $\mathcal{E}$, $\forall t \in[T],\ \forall h \in [H],\ \forall (s,a) \in \mathcal{S} \times \mathcal{A}$, if $n_n^t(s,a)>0$, the following holds,
\begin{align*}
&\ \ \ \ |Q_h^t(s,a)-r_h(s,a)-p_h\mathring{V}_{h,s,a}^t(s,a)|\\
&\le \sqrt{\frac{2}{n_h^t(s,a)} \sum_{k=t(s)}^t\chi_h^k(s,a)Var_{p_h}(\mathring{\overline{V}}_h^{k-1})(s,a) \frac{\zeta}{n_h^t(s,a)}}+20H^3 \frac{\zeta log(T)}{n_h^t(s,a)}\\
&\ \ \ \ +\frac{1}{4Hlog(T)n_h^t(s,a)}\sum_{k=t(s)}^t \chi_h^k(s,a)\mathring{\gamma}_h^k(s,a) p_h(\mathring{V}^{k-1}_{h,s,a}-\mathring{\overline{V}}_{h+1}^{k-1})(s,a),
\end{align*}
where $\mathring{\gamma}_h^t(s,a)=H\frac{n_h^t(s,a)-1}{n_h^t(s,a)+H}$.
\end{lemma}

\begin{proof}
$n_h^t(s,a)>0$ implies $s\in \mathcal{S}_t$. So, we have
\begin{align*}
&\ \ \ \ Q_h^t(s,a)-r_h(s,a)-p_h\mathring{V}_{h,s,a}^t(s,a)\\
&=\frac{1}{n_h^t(s,a)}\sum_{k=t(s)}^t\chi_h^k(s,a)[\mathring{p}_h^k \mathring{\overline{V}}_{h+1}^{k-1}(s,a)+\mathring{\gamma}_h^k(s,a)\mathring{p}_h^k(\mathring{\overline{V}}_{h+1}^{k-1}-\mathring{V}_{h,s,a}^{k-1})(s,a)]\\
&\ \ \ \ -p_h\mathring{V}_{h,s,a}^t(s,a)\\
&=\frac{1}{n_h^t(s,a)}\sum_{k=t(s)}^t\chi_h^k(s,a)[(\mathring{p}_h^k-p_h) \mathring{\overline{V}}_{h+1}^{k-1}(s,a)\\
&\ \ \ \ +\mathring{\gamma}_h^k(s,a)(\mathring{p}_h^k-p_h)(\mathring{\overline{V}}_{h+1}^{k-1}-\mathring{V}_{h,s,a}^{k-1})(s,a)]\\
&\le |\frac{1}{n_h^t(s,a)}\sum_{k=t(s)}^t\chi_h^k(s,a)(\mathring{p}_h^k-p_h) \mathring{\overline{V}}_{h+1}^{k-1}(s,a)|\\
&\ \ \ \ +|\frac{1}{n_h^t(s,a)}\sum_{k=t(s)}^t\chi_h^k(s,a)\mathring{\gamma}_h^k(s,a)(\mathring{p}_h^k-p_h)(\mathring{\overline{V}}_{h+1}^{k-1}-\mathring{V}_{h,s,a}^{k-1})(s,a)|.
\end{align*}
We can upper-bound the first term of the right-hand by the definition of $\mathcal{E}$:
\begin{align*}
&\ \ \ \ |\frac{1}{n_h^t(s,a)}\sum_{k=t(s)}^t\chi_h^k(s,a)(\mathring{p}_h^k-p_h) \mathring{\overline{V}}_{h+1}^{k-1}(s,a)|\\
&\le \sqrt{\frac{2}{n_h^t(s,a)} \sum_{k=t(s)}^t\chi_h^k(s,a)Var_{p_h}(\mathring{\overline{V}}_h^{k-1})(s,a) \frac{\zeta}{n_h^t(s,a)}}+6H\frac{\zeta}{n_h^t(s,a)}.
\end{align*}

For the second term, we use Lemma \ref{E_lemma}, and get
\begin{align*}
&\ \ \ \ |\frac{1}{n_h^t(s,a)}\sum_{k=t(s)}^t\chi_h^k(s,a)\mathring{\gamma}_h^k(s,a)(\mathring{p}_h^k-p_h)(\mathring{\overline{V}}_{h+1}^{k-1}-\mathring{V}_{h,s,a}^{k-1})(s,a)|\\
&\le \frac{1}{4Hlog(T)n_h^t(s,a)}\sum_{k=t(s)}^t \chi_h^k(s,a)\mathring{\gamma}_h^k(s,a) p_h(\mathring{V}^{k-1}_{h,s,a}-\mathring{\overline{V}}_{h+1}^{k-1})(s,a)\\
&\ \ \ \ +14H^3 \frac{\zeta log(T)}{n_h^t(s,a)}. 
\end{align*}
This concludes the proof.
\end{proof}

We design a bonus to compensate for the bias in the previous lemma.

\begin{lemma}
\label{bonus_lower_bound}
On the event $\mathcal{E}$, $\forall t \in[T],\ \forall h \in [H],\ \forall (s,a) \in \mathcal{S} \times \mathcal{A}$, if $n_n^t(s,a)>0$, the following holds
\begin{align*}
\beta_h^t(s,a) \ge &\sqrt{\frac{2}{n_h^t(s,a)} \sum_{k=t(s)}^t \chi_h^k(s,a)Var_{p_h}(\mathring{\overline{V}}_{h+1}^{k-1})(s,a) \frac{\zeta}{n_h^t(s,a)}}+20H^3\frac{\zeta log(T)}{n_h^t(s,a)}\\
&+\frac{1}{4Hlog(T)n_h^t(s,a)} \sum_{k=t(s)}^{t} \chi_h^t(s,a) \mathring{\gamma}_h^k(s,a)p_h(\mathring{V}_{h,s,a}^{k-1}-\mathring{\overline{V}}_{h+1}^{k-1})(s,a).
\end{align*}
\end{lemma}
\begin{proof}
The definition of bonus is: for $s \in \mathcal{S}_t$
\begin{align*}
\beta_h^t(s,a)=&2 \sqrt{W_h^t(s,a) \frac{\zeta}{n_h^t(s,a)}}+53H^3 \frac{\zeta log(T)}{n_h^t(s,a)}\\
&+\frac{1}{H log(T) n_h^t(s,a)} \sum_{k=t(s)}^{t} \chi_h^k(s,a)\mathring{\gamma}_h^k(s,a)\mathring{p}_h^k(\mathring{V}_{h,s,a}^{k-1}-\mathring{\overline{V}}_{h+1}^{k-1})(s,a),
\end{align*}
where 
\begin{align*}
W_h^t(s,a)=&\frac{1}{n_h^t(s,a)} \sum_{k=t(s)}^t \chi_h^k(s,a) \mathring{p}_h^k(\mathring{\overline{V}}_{h+1}^{k-1})^2(s,a)\\
&-\left(\frac{1}{n_h^t(s,a)}\sum_{k=t(s)}^t \chi_h^k(s,a) \mathring{p}_h^k\mathring{\overline{V}}_{h+1}^{k-1}(s,a) \right)^2.
\end{align*}
Since we do not know the true transitions $p_h$, the design of the bonus only needs the observed data.

Based on Lemma \ref{E_lemma}, we can control the correction term as follows,
\begin{equation}
\begin{aligned}
&\ \ \ \ |\frac{1}{n_h^t(s,a)}\sum_{k=t(s)}^t \chi_h^k(s,a)\mathring{\gamma}_h^k(s,a) (\mathring{p}_h^k-p_h)(\mathring{\overline{V}}_{h+1}^{k-1}-\mathring{V}^{k-1}_{h,s,a})(s,a)|\\
&\le \frac{1}{4Hlog(T)n_h^t(s,a)}\sum_{k=t(s)}^t \chi_h^k(s,a)\mathring{\gamma}_h^k(s,a) p_h(\mathring{V}^{k-1}_{h,s,a}-\mathring{\overline{V}}_{h+1}^{k-1})(s,a)\\
&\ \ \ \ +14H^3 \frac{log(T)\zeta}{n_h^t(s,a)}.
\label{concentration_term}
\end{aligned}
\end{equation}

As for the proxy of the variance term $W_h^t(s,a)$, applying Lemma \ref{E_lemma} and the definition of $\mathcal{E}$ and $\sqrt{xy} \le x+y$, we get
\begin{equation}
\begin{aligned}
&\ \ \ \ \frac{1}{n_h^t(s,a)}|\sum_{k=t(s)}^t \chi_h^k(s,a)(\mathring{p}_h^k-p_h)(\mathring{\overline{V}}_{h+1}^{k-1})^2(s,a)| \\
&\le \frac{1}{4n_h^t(s,a)} \sum_{k=t(s)}^t \chi_h^k(s,a)Var_{p_h}(\mathring{\overline{V}}_{h+1}^{k-1})(s,a)+44H^2 \frac{\zeta}{n_h^t(s,a)}, 
\label{var_1}
\end{aligned}
\end{equation}
and
\begin{align}
&\ \ \ \ |(\frac{1}{n_h^t(s,a)}\sum_{k=t(s)}^t \chi_h^k(s,a) \mathring{p}_h^k\mathring{\overline{V}}_{h+1}^{k-1}(s,a))^2\\
&\ \ \ \ - (\frac{1}{n_h^t(s,a)}\sum_{k=t(s)}^t \chi_h^k(s,a) {p}_h\mathring{\overline{V}}_{h+1}^{k-1}(s,a))^2| \nonumber\\
&\le |\frac{1}{n_h^t(s,a)}\sum_{k=t(s)}^t \chi_h^k(s,a) (\mathring{p}_h^k+p_h)\mathring{\overline{V}}_{h+1}^{k-1}(s,a)|\\
&\ \ \ \ \cdot|\frac{1}{n_h^t(s,a)}\sum_{k=t(s)}^t \chi_h^k(s,a) (\mathring{p}_h^k-p_h)\mathring{\overline{V}}_{h+1}^{k-1}(s,a)| \nonumber\\
&\le \frac{2H}{n_h^t(s,a)} |\sum_{k=t(s)}^t \chi_h^k(s,a) (\mathring{p}_h^k-p_h)\mathring{\overline{V}}_{h+1}^{k-1}(s,a)| \nonumber \\
&\le H \sqrt{8 \frac{1}{n_h^t(s,a)}\sum_{k=t(s)}^{t} \chi_h^k(s,a)Var_{p_h}(\mathring{\overline{V}}_{h+1}^{k-1})(s,a) \frac{\zeta}{n_h^t(s,a)}}+12H^2\frac{\zeta}{n_h^t(s,a)} \nonumber \\
&\le \frac{1}{4 n_h^t(s,a)} \sum_{k=t(s)}^{t} \chi_h^k(s,a)Var_{p_h}(\mathring{\overline{V}}_{h+1}^{k-1})(s,a) +44H^2\frac{\zeta}{n_h^t(s,a)} .\label{var_2}
\end{align}

Then, combining eq. (\ref{var_1}), eq. (\ref{var_2}) and Jensen's inequality, we get
\begin{align}
W_h^t(s,a)&=\frac{1}{n_h^t(s,a)} \sum_{k=t(s)}^t \chi_h^k(s,a) \mathring{p}_h^k(\mathring{\overline{V}}_{h+1}^{k-1})^2(s,a)\\
&\ \ \ \ -(\frac{1}{n_h^t(s,a)}\sum_{k=t(s)}^t \chi_h^k(s,a) \mathring{p}_h^k\mathring{\overline{V}}_{h+1}^{k-1}(s,a))^2 \nonumber\\
&\ge \frac{1}{n_h^t(s,a)} \sum_{k=t(s)}^t \chi_h^k(s,a)p_h(\mathring{\overline{V}}_{h+1}^{k-1})^2(s,a)\\
&\ \ \ \ -(\frac{1}{n_h^t(s,a)}\sum_{k=t(s)}^t \chi_h^k(s,a) p_h\mathring{\overline{V}}_{h+1}^{k-1}(s,a))^2 \nonumber\\
&\ \ \ \ -\frac{1}{2n_h^t(s,a)}\sum_{k=t(s)}^t \chi_h^k(s,a)Var_{p_h}(\mathring{\overline{V}}_{h+1}^{k-1})(s,a)-88H^2\frac{\zeta}{n_h^t(s,a)} \nonumber\\
&\ge \frac{1}{n_h^t(s,a)} \sum_{k=t(s)}^t \chi_h^k(s,a)p_h(\mathring{\overline{V}}_{h+1}^{k-1})^2(s,a)\\
&\ \ \ \ -\frac{1}{n_h^t(s,a)}\sum_{k=t(s)}^t \chi_h^k(s,a) (p_h\mathring{\overline{V}}_{h+1}^{k-1}(s,a))^2 \nonumber\\
&\ \ \ \ -\frac{1}{2n_h^t(s,a)}\sum_{k=t(s)}^t \chi_h^k(s,a)Var_{p_h}(\mathring{\overline{V}}_{h+1}^{k-1})(s,a)-88H^2\frac{\zeta}{n_h^t(s,a)} \nonumber\\
&=\frac{1}{2n_h^t(s,a)}\sum_{k=t(s)}^t \chi_h^k(s,a)Var_{p_h}(\mathring{\overline{V}}_{h+1}^{k-1})(s,a)-88H^2\frac{\zeta}{n_h^t(s,a)}.
\label{W_term}
\end{align}

Finally, with eq. (\ref{concentration_term}) and eq. (\ref{W_term}), we can conclude
\begin{align*}
\beta_h^t(s,a)&\ge2 \sqrt{(W_h^t(s,a)+88H^2 \frac{\zeta}{n_h^t(s,a)})\frac{\zeta}{n_h^t(s,a)}}\\
&\ \ \ \ +(53-2\sqrt{88}-14)H^3\frac{\zeta log(T)}{n_h^t(s,a)}\\
&\ \ \ \ +\frac{3}{4Hlog(T)n_h^t(s,a)}\sum_{k=t(s)}^{t} \chi_h^k(s,a) \mathring{\gamma}_h^k(s,a)p_h(\mathring{V}_{h,s,a}^{k-1}-\mathring{\overline{V}}_{h+1}^{k-1})(s,a)\\
&\ge2 \sqrt{\frac{2}{n_h^t(s,a)} \sum_{k=t(s)}^{t} \chi_h^k(s,a)Var_{p_h}(\mathring{\overline{V}}_{h+1}^{k-1})(s,a) \frac{\zeta}{n_h^t(s,a)}}\\
&\ \ \ \ +20H^3\frac{\zeta log(T)}{n_h^t(s,a)}\\
&\ \ \ \ +\frac{1}{4Hlog(T)n_h^t(s,a)}\sum_{k=t(s)}^{t} \chi_h^k(s,a) \mathring{\gamma}_h^k(s,a)p_h(\mathring{V}_{h,s,a}^{k-1}-\mathring{\overline{V}}_{h+1}^{k-1})(s,a),
\end{align*}
which also relies on the fact that $\sqrt{x+y} \le \sqrt{x}+\sqrt{y}$.

This concludes the proof.
\end{proof}

Now we are ready to prove Lemma \ref{optimism_lemma} on the event $\mathcal{E}$ and with the Homeland condition in Section \ref{sec:EMDP-GA}.
\begin{lemma}
\label{optimism_lemma}
% On the event $\mathcal{E}$ and with the Assumption \ref{optimal_policy_assumption}
The following holds with probability at least $1-\delta$, $\forall t \in [T],\ \forall(s,a,h) \in \mathcal{S} \times \mathcal{A} \times [H]$ (also holds for $H+1$ for the value function), we have
$$
\mathring{\overline{Q}}_h^t(s,a) \ge  Q^*_h(s,a), \ \ \ \text{and} \ \ \ \mathring{\overline{V}}_h^t(s) \ge V_h^*(s).
$$
% where $\mu=\mathop{min}\limits_{s,a,h \in S \times A \times [H]}\frac{r_h(s,a)}{r_h(s,a)+H\sum_{s' \in S-S_0}p_h(s'|s,a)}$
\end{lemma}

% The preceding analysis indicates that NIVE is an advanced %robust and stable 
% domain‐expansion technique.
% It provides reasonable and reliable estimators of the %$Q$-value function and the 
% value functions.

%Consequently, we incorporate NIVE into our framework to address challenges arising from the expansion of the awareness domain.
% \begin{lemma}
% \label{optimism_lemma}
% On the event $\mathcal{E}$ and with the Assumption \ref{optimal_policy_assumption}, the following holds with probability at least $1-\delta$, $\forall t \in [T],\ \forall(s,a,h) \in \mathcal{S} \times \mathcal{A} \times [H]$ (also holds for $H+1$ for the value function), we have
% $$
% \mathring{\overline{Q}}_h^t(s,a) \ge  Q^*_h(s,a), \ \ \ \text{and} \ \ \ \mathring{\overline{V}}_h^t(s) \ge V_h^*(s).
% $$
% % where $\mu=\mathop{min}\limits_{s,a,h \in S \times A \times [H]}\frac{r_h(s,a)}{r_h(s,a)+H\sum_{s' \in S-S_0}p_h(s'|s,a)}$
% \end{lemma}
\begin{proof}
We proceed by induction on $t$. The two claims hold trivially for $t=0$ as a result of the initialisation of the algorithm.
Assume the two claims hold for $k \le t-1$.
We consider the case of $t$.
For $h=H+1$, we have $$
\mathring{\overline{Q}}_h^t(s,a) = Q^*_h(s,a)=0,
$$ 
and 
$$
\mathring{\overline{V}}_h^t(s) = V_h^*(s)=0.
$$ 
We assume 
$$
\mathring{\overline{Q}}_{h+1}^t(s,a) \ge  Q^*_{h+1}(s,a),
$$ 
and 
$$
\mathring{\overline{V}}_{h+1}^t(s) \ge  V_{h+1}^*(s).
$$

If $n_h^t(s,a)=0$ and $s \in \mathcal{S}_t$, we have 
\begin{align*}
\mathring{\overline{Q}}_h^t(s,a)&=\frac{1}{|\mathcal{S}_{t(s)-1}|}\sum_{\hat{s} \in \mathcal{S}_{t(s)-1}}Q_h^{t(s)-1}(\hat{s},a) +H \\
&\ge \frac{1}{|\mathcal{S}_{t(s)-1}|}\sum_{\hat{s} \in \mathcal{S}_{t(s)-1}}\overline{Q}_h^{t(s)-1}(\hat{s},a) \\
&\ge \frac{1}{|\mathcal{S}_{t(s)-1}|}\sum_{\hat{s} \in \mathcal{S}_{t(s)-1}}{Q}_h^{*}(\hat{s},a) \\
&\ge Q^*_{h}(s,a).
\end{align*}

If $s \in \mathcal{S}-\mathcal{S}_t$, we have
\begin{align*}
\mathring{\overline{Q}}_h^t(s,a)&=\frac{1}{|\mathcal{S}_{t}|}\sum_{\hat{s} \in \mathcal{S}_{t}}Q_h^{t}(\hat{s},a) +H \\
&\ge \frac{1}{|\mathcal{S}_{t}|}\sum_{\hat{s} \in \mathcal{S}_{t}}\overline{Q}_h^{t}(\hat{s},a) \\
&\ge \frac{1}{|\mathcal{S}_{t}|}\sum_{\hat{s} \in \mathcal{S}_{t}}{Q}_h^{*}(\hat{s},a) \\
&\ge Q^*_{h}(s,a).
\end{align*}

If $n_h^t(s,a)>0$, applying Lemmas \ref{Q_estimate} and \ref{bonus_lower_bound}, and because the event $\mathcal{E}$ holds, we get
\begin{align*}
\mathring{\overline{Q}}_h^t(s,a)&=Q_h^t(s,a)+\beta_h^t(s,a) \\
&\ge r_h(s,a)+p_h\mathring{V}_{h,s,a}^t(s,a)\\
&= r_h(s,a)+p_h(\sum_{k=t(s)}^t \tilde{\eta}_h^{t,k}(s,a)\mathring{\overline{V}}_{h+1}^{k-1})(s,a)\\
&\ge r_h(s,a)+p_hV^*_h(s,a)=Q^*_h(s,a).
\end{align*}

For $s \in \mathcal{S}_t$, note that
$$
\mathring{\overline{V}}_h^t(s)=\overline{V}_h^t(s)\ge clip(\mathop{max}\limits_{a \in \mathcal{A}} \overline{Q}_h^t(s,a),0,H) \ge \mathop{max}\limits_{a \in \mathcal{A}}  Q_h^*(s,a)= V^*_{h}(s).
$$

For $s \in \mathcal{S}-\mathcal{S}_t$, we have
\begin{align*}
\mathring{\overline{V}}_h^t(s)&=\frac{1}{|\mathcal{S}_t|}\sum_{\hat{s} \in \mathcal{S}_t} \overline{V}_h^t(\hat{s})\\
&\ge\frac{1}{|\mathcal{S}_t|}\sum_{\hat{s} \in \mathcal{S}_t} clip(\mathop{max}\limits_{a \in \mathcal{A}} \overline{Q}_h^t(\hat{s},a),0,H)\\
&\ge \frac{1}{|\mathcal{S}_t|}\sum_{\hat{s} \in \mathcal{S}_t} \mathop{max}\limits_{a \in \mathcal{A}} {Q}_h^*(\hat{s},a)\\
&\ge \mathop{max}\limits_{a \in \mathcal{A}} \frac{1}{|\mathcal{S}_t|}\sum_{\hat{s} \in \mathcal{S}_t}{Q}_h^*(\hat{s},a)\\
&\ge\mathop{max}\limits_{a \in \mathcal{A}} {Q}_h^*(s,a)=V^*_h(s).
\end{align*}

The proof is completed.
\end{proof}

\section{Proof of Theorem \ref{theorem_1}}
\label{proof_of_regret_bound}
Let $\tilde{n}_h^t(s,a)=\mathop{max} \{n_h^t(s,a),1\}$.
We first give two important lemmas.

\begin{lemma}
\label{all_Q_estimate}
On the event $\mathcal{E}$, with Homeland condition in Section \ref{sec:EMDP-GA}, $\forall t \in [T],\ \forall h \in [H],\ \forall (s,a) \in \mathcal{S}\times \mathcal{A}$, we have,
\begin{align*}
&\ \ \ \ |\mathring{Q}_h^t(s,a)-r_h(s,a)-p_h\mathring{V}_{h,s,a}^t(s,a)|\\
&\le \sqrt{\frac{4}{\tilde{n}_h^t(s,a)} \sum_{k=t(s)}^t\chi_h^k(s,a)Var_{p_h}(V_{h+1}^{\pi^k})(s,a) \frac{\zeta}{\tilde{n}_h^t(s,a)}}+24H^3 \frac{\zeta log(T)}{\tilde{n}_h^t(s,a)}\\
&\ \ \ \ +\frac{2}{Hlog(T)\tilde{n}_h^t(s,a)}\sum_{k=t(s)}^t \chi_h^k(s,a) p_h(\mathring{V}^{k-1}_{h,s,a}-{V}_{h+1}^{\pi_k})(s,a).
\end{align*}
\end{lemma}

\begin{proof}
we first prove $|\mathring{Q}_h^t(s,a)| \le H^2$ by induction.
When $t=0$, the claim trivially holds.
Assume that the claim holds for $k \le t-1$.
For $s \in \mathcal{S}_t,\ n_h^t(s,a)=0$, we have $$|Q_h^t(s,a)|=|\mathring{Q}_h^t(s,a)| \le H^2.$$
For $s \in \mathcal{S}_t,\ n_h^t(s,a)>0$, we have
\begin{align*}
&\ \ \ \ |Q_h^t(s,a)|\\
&\le |r_h(s,a)|+ \\
&\ \ \ \ \left|\frac{1}{n_h^t(s,a)}\sum_{k=t(s)}^t\chi_h^k(s,a)\left[\mathring{p}_h^k \mathring{\overline{V}}_{h+1}^{k-1}(s,a)+\mathring{\gamma}_h^k(s,a)\mathring{p}_h^k(\mathring{\overline{V}}_{h+1}^{k-1}-\mathring{V}_{h,s,a}^{k-1})(s,a)\right]\right|\\
&\le 1+\left|\frac{1}{n_h^t(s,a)}\sum_{k=t(s)}^t\chi_h^k(s,a)\left[H+\frac{H^2}{2}\right] \right| \le H^2.
\end{align*}
For $s \in \mathcal{S}-\mathcal{S}_t$, we have
$$
|\mathring{Q}_h^t(s,a)|=|\frac{1}{|\mathcal{S}_t|}\sum_{\hat{s} \in \mathcal{S}_t} Q_h^t(s,a)| \le H^2.
$$
Therefore, we can conclude $\forall (s,a) \in \mathcal{S}\times \mathcal{A}:\ |\mathring{Q}_h^t(s,a)| \le H^2$.

For $n_h^t(s,a)=0$, the bound trivially holds because
$$
|\mathring{Q}_h^t(s,a)-r_h(s,a)-p_h\mathring{V}_{h,s,a}^t(s,a)|
\le |\mathring{Q}_h^t(s,a)|+|r_h(s,a)|+|p_h\mathring{V}_{h,s,a}^t(s,a)|
\le H^2+H+1.
$$
For $n_h^t(s,a)>0$, we have on the event $\mathcal{E}$ with Lemma \ref{Q_estimate}
\begin{align*}
&\ \ \ \ |\mathring{Q}_h^t(s,a)-r_h(s,a)-p_h\mathring{V}_{h,s,a}^t(s,a)|\\
&\le \sqrt{\frac{2}{n_h^t(s,a)} \sum_{k=t(s)}^t\chi_h^k(s,a)Var_{p_h}(\mathring{\overline{V}}_h^{k-1})(s,a) \frac{\zeta}{n_h^t(s,a)}}+20H^3 \frac{\zeta log(T)}{n_h^t(s,a)}\\
&\ \ \ \ +\frac{1}{4Hlog(T)n_h^t(s,a)}\sum_{k=t(s)}^t \chi_h^k(s,a)\mathring{\gamma}_h^k(s,a) p_h(\mathring{V}^{k-1}_{h,s,a}-\mathring{\overline{V}}_{h+1}^{k-1})(s,a).
\end{align*}

Applying $\mathring{V}_{h,s,a}^t \ge V^*_{h+1}\ge V_{h+1}^{\pi^k}$ and eq. (\ref{V_induction}), we can get
\begin{align}
&\ \ \ \ \frac{1}{n_h^t(s,a)} \sum_{k=t(s)}^t \chi_h^k(s,a) \mathring{\gamma}_h^k(s,a) p_h (\mathring{V}_{h,s,a}^{k-1}-\mathring{\overline{V}}_{h+1}^{k-1})(s,a) \nonumber\\
&\le \frac{1}{n_h^t(s,a)} \sum_{k=t(s)}^t \chi_h^k(s,a)  p_h (\mathring{\overline{V}}_{h+1}^{k-1}-{V}_{h+1}^{\pi^k})(s,a). \label{part1}
\end{align}

Applying $H \ge \mathring{\overline{V}}_h^k \ge V^*_h \ge V_h^{\pi^{k+1}}$ and Lemma \ref{Variance_inequality}, we have
\begin{align*}
&\ \ \ \ \frac{1}{n_h^t(s,a)} \sum_{k=t(s)}^t \chi_h^k(s,a) Var_{p_h}(\mathring{\overline{V}}_{h+1}^{k-1})(s,a) \\
&\le \frac{2}{n_h^t(s,a)} \sum_{k=t(s)}^t \chi_h^k(s,a) Var_{p_h}({V}_{h+1}^{\pi^k})(s,a) \\
&\ \ \ \ + \frac{2H}{n_h^t(s,a)} \sum_{k=t(s)}^t \chi_h^k(s,a) p_h (\mathring{\overline{V}}_{h+1}^{k-1}-V_{h+1}^{\pi^k})(s,a).
\end{align*}

Applying $\sqrt{x+y} \le \sqrt{x}+ \sqrt{y}$ and $\sqrt{xy} \le x+y$, we get
\begin{align}
&\ \ \ \ \sqrt{\frac{2}{n_h^t(s,a)} \sum_{k=t(s)}^t \chi_h^k(s,a) Var_{p_h}(\mathring{\overline{V}}_{h+1}^{k-1})(s,a) \frac{\zeta}{n_h^t(s,a)}}\nonumber\\
& \le \sqrt{\frac{4}{n_h^t(s,a)} \sum_{k=t(s)}^t \chi_h^k(s,a) Var_{p_h}({V}_{h+1}^{\pi^k})(s,a)\frac{\zeta}{n_h^t(s,a)}} \nonumber\\
&\ \ \ \ +\sqrt{\frac{4H}{n_h^t(s,a)} \sum_{k=t(s)}^t \chi_h^k(s,a) p_h (\mathring{\overline{V}}_{h+1}^{k-1}-V_{h+1}^{\pi^k})(s,a)\frac{\zeta}{n_h^t(s,a)}}\nonumber\\
&\le \sqrt{\frac{4}{n_h^t(s,a)} \sum_{k=t(s)}^t \chi_h^k(s,a) Var_{p_h}({V}_{h+1}^{\pi^k})(s,a)\frac{\zeta}{n_h^t(s,a)}}+4H^2\frac{log(T)\zeta}{n_h^t(s,a)} \nonumber\\
&\ \ \ \ + \frac{1}{Hlog(T)n_h^t(s,a)} \sum_{k=t(s)}^t \chi_h^k(s,a) p_h (\mathring{\overline{V}}_{h+1}^{k-1}-V_{h+1}^{\pi^k})(s,a)\frac{\zeta}{n_h^t(s,a)}. \label{part2}
\end{align}

Combining eq. (\ref{part1}) and eq. (\ref{part2}), we can conclude
\begin{align*}
&\ \ \ \ |\mathring{Q}_h^t(s,a)-r_h(s,a)-p_h\mathring{V}_{h,s,a}^t(s,a)|\\
&\le \sqrt{\frac{4}{n_h^t(s,a)} \sum_{k=t(s)}^t\chi_h^k(s,a)Var_{p_h}(V_{h+1}^{\pi^k})(s,a) \frac{\zeta}{n_h^t(s,a)}}+24H^3 \frac{\zeta log(T)}{n_h^t(s,a)}\\
&\ \ \ \ +\frac{2}{Hlog(T)n_h^t(s,a)}\sum_{k=t(s)}^t \chi_h^k(s,a) p_h(\mathring{\overline{V}}^{k-1}_{h+1}-V_{h+1}^{\pi^k})(s,a).
\end{align*}
\end{proof}

Then, we give an upper bound on the bonus.

\begin{lemma}
\label{beta_upper_bound}
On the event $\mathcal{E}$, under Homeland condition in Section \ref{sec:EMDP-GA}, $\forall t \in [T],\ \forall h \in [H],\ \forall (s,a) \in \mathcal{S} \times \mathcal{A}$, the following holds
\begin{align*}
\beta_h^t(s,a) &\le 2 \sqrt{\frac{3}{\tilde{n}_h^t(s,a)} \sum_{k=t(s)}^t\chi_h^k(s,a)Var_{p_h}(V_{h+1}^{\pi^k})(s,a)\frac{\zeta}{\tilde{n}_h^t(s,a)}}\\
&\ \ \ \ + 106H^3\frac{log(T)\zeta}{\tilde{n}_h^t(s,a)}\\
&\ \ \ \ +\frac{3}{Hlog(T)\tilde{n}_h^t(s,a)}\sum_{k=t(s)}^t\chi_h^k(s,a) p_h(\mathring{\overline{V}}_{h+1}^{k-1}-V_{h+1}^{\pi^k})(s,a).
\end{align*}
\end{lemma}
\begin{proof}
If $n_h^t(s,a)=0$, the bound is trivially true because $\beta_h^t(s,a)=H$.
We now consider the case $n_h^t(s,a)>0$.
Using eq. (\ref{var_1}), eq. (\ref{var_2}) and $H \ge \mathring{\overline{V}}_h^k \ge V^*_h$, we obtain
\begin{align*}
W_h^t(s,a)&\le \frac{1}{n_h^t(s,a)} \sum_{k=t(s)}^t \chi_h^k(s,a) {p}_h(\mathring{\overline{V}}_{h+1}^{k-1})^2(s,a)\\
&\ \ \ \ -(\frac{1}{n_h^t(s,a)}\sum_{k=t(s)}^t \chi_h^k(s,a) {p}_h\mathring{\overline{V}}_{h+1}^{k-1}(s,a))^2\\
&\ \ \ \ +\frac{1}{2n_h^t(s,a)}\sum_{k=t(s)}^t \chi_h^k(s,a)Var_{p_h}(\mathring{\overline{V}}_{h+1}^{k-1})(s,a)+88H^2\frac{\zeta}{n_h^t(s,a)}\\
&=Var_{p_h}(V_{h+1}^*)(s,a)\\
&\ \ \ \ +\frac{1}{n_h^t(s,a)} \sum_{k=t(s)}^t \chi_h^k(s,a) {p}_h((\mathring{\overline{V}}_{h+1}^{k-1})^2-(V_{h+1}^*)^2)(s,a)\\
&\ \ \ \ +(p_hV_{H+1}^*(s,a))^2-(\frac{1}{n_h^t(s,a)}\sum_{k=t(s)}^t \chi_h^k(s,a) {p}_h\mathring{\overline{V}}_{h+1}^{k-1}(s,a))^2 \\
&\ \ \ \ +\frac{1}{2n_h^t(s,a)}\sum_{k=t(s)}^t \chi_h^k(s,a)Var_{p_h}(\mathring{\overline{V}}_{h+1}^{k-1})(s,a)+88H^2\frac{\zeta}{n_h^t(s,a)}\\
&\le Var_{p_h}(V_{h+1}^*)(s,a)  +\frac{1}{2n_h^t(s,a)}\sum_{k=t(s)}^t \chi_h^k(s,a)Var_{p_h}(\mathring{\overline{V}}_{h+1}^{k-1})(s,a)\\
&\ \ \ \ \frac{2H}{n_h^t(s,a)} \sum_{k=t(s)}^t \chi_h^k(s,a)p_h(\mathring{\overline{V}}_{h+1}^{k-1}-V^*_{h+1})(s,a) +88H^2\frac{\zeta}{n_h^t(s,a)}.
\end{align*}

Applying $H \ge \mathring{\overline{V}}_h^k \ge V_h^* \ge V_h^{\pi^{k+1}}$ and applying Lemma \ref{Variance_inequality} to the terms $Var_{p_h}(V^*_{h+1})(s,a)$ and $Var_{p_h}(\mathring{\overline{V}}_{h+1}^{k-1})(s,a)$, we obtain
\begin{align}
W_h^t(s,a)&\le \frac{3}{n_h^t(s,a)}\sum_{k=t(s)}^t \chi_h^k(s,a) Var_{p_h}(V_{h+1}^{\pi^k})(s,a) \nonumber\\
&\ \ \ \ +\frac{2H}{n_h^t(s,a)}\sum_{k=t(s)}^t \chi_h^k(s,a) p_h (V^*_{h+1}-V_{h+1}^{\pi^k})(s,a)\nonumber\\
&\ \ \ \ +\frac{H}{n_h^t(s,a)}\sum_{k=t(s)}^t \chi_h^k(s,a) p_h (\mathring{\overline{V}}_{h+1}^{k-1}-V_{h+1}^{\pi^k})(s,a)\nonumber\\
&\ \ \ \ +\frac{2H}{n_h^t(s,a)} \sum_{k=t(s)}^t \chi_h^k(s,a)p_h(\mathring{\overline{V}}_{h+1}^{k-1}-V^*_{h+1})(s,a) +88H^2\frac{\zeta}{n_h^t(s,a)}\nonumber\\
&\le \frac{3}{n_h^t(s,a)}\sum_{k=t(s)}^t \chi_h^k(s,a) Var_{p_h}(V_{h+1}^{\pi^k})(s,a)\nonumber\\
&\ \ \ \ +\frac{5H}{n_h^t(s,a)} \sum_{k=t(s)}^t \chi_h^k(s,a)p_h(\mathring{\overline{V}}_{h+1}^{k-1}-V^{\pi^k}_{h+1})(s,a) +88H^2\frac{\zeta}{n_h^t(s,a)}. \label{W_upper_bound}
\end{align}
Combine eq. (\ref{W_upper_bound}) with $\sqrt{x+y} \le \sqrt{x}+ \sqrt{y}$ and $\sqrt{xy} \le x+y$ we can upper-bound the variance term of the bonus
\begin{align}
&\ \ \ \ 2\sqrt{W_h^t(s,a)\frac{\zeta}{n_h^t(s,a)}} \nonumber\\
&\le 2 \sqrt{\frac{3}{n_h^t(s,a)}\sum_{k=t(s)}^t \chi_h^k(s,a) Var_{p_h}(V_{h+1}^{\pi^k})(s,a) \frac{\zeta}{n_h^t(s,a)}} 
\nonumber\\
&\ \ \ \ +\sqrt{\frac{20H}{n_h^t(s,a)} \sum_{k=t(s)}^t \chi_h^k(s,a)p_h(\mathring{\overline{V}}_{h+1}^{k-1}-V^{\pi^k}_{h+1})(s,a) \frac{\zeta}{n_h^t(s,a)}}\nonumber\\
&\ \ \ \ + 2\sqrt{88}H\frac{\zeta}{n_h^t(s,a)}\nonumber\\
&\le 2 \sqrt{\frac{3}{n_h^t(s,a)}\sum_{k=t(s)}^t \chi_h^k(s,a) Var_{p_h}(V_{h+1}^{\pi^k})(s,a) \frac{\zeta}{n_h^t(s,a)}}\nonumber\\
&\ \ \ \ +\frac{1}{Hlog(T)n_h^t(s,a)}\sum_{k=t(s)}^t \chi_h^k(s,a)p_h(\mathring{\overline{V}}_{h+1}^{k-1}-V^{\pi^k}_{h+1})(s,a)+ \frac{20H^2\zeta}{n_h^t(s,a)}\nonumber\\
&\ \ \ \ + 19H\frac{\zeta}{n_h^t(s,a)}\nonumber\\
&\le 2 \sqrt{\frac{3}{n_h^t(s,a)}\sum_{k=t(s)}^t \chi_h^k(s,a) Var_{p_h}(V_{h+1}^{\pi^k})(s,a) \frac{\zeta}{n_h^t(s,a)}}\nonumber\\
&\ \ \ \ +\frac{1}{Hlog(T)n_h^t(s,a)}\sum_{k=t(s)}^t \chi_h^k(s,a)p_h(\mathring{\overline{V}}_{h+1}^{k-1}-V^{\pi^k}_{h+1})(s,a)\nonumber\\
&\ \ \ \ + \frac{39H^2\zeta}{n_h^t(s,a)}. \label{beta_part1}
\end{align}

Applying Lemma \ref{E_lemma} and eq. (\ref{part1}), we have
\begin{align}
&\ \ \ \ \frac{1}{Hlog(T)n_h^t(s,a)}\sum_{k=t(s)}^t \chi_h^k(s,a) \mathring{\gamma}_h^k(s,a) \mathring{p}_h^k(\mathring{V}_{h,s,a}^{k-1}-\mathring{\overline{V}}_{h+1}^{k-1})(s,a)\nonumber\\
&\le 14H^3 \frac{log(T)\zeta}{n_h^t(s,a)}\nonumber\\
&\ \ \ \ +\frac{5}{4Hlog(T)n_h^t(s,a)}\sum_{k=t(s)}^t \chi_h^k(s,a) \mathring{\gamma}_h^k(s,a) p_h(\mathring{V}_{h,s,a}^{k-1}-\mathring{\overline{V}}_{h+1}^{k-1})(s,a)\nonumber\\
&\le 14H^3 \frac{log(T)\zeta}{n_h^t(s,a)}+\frac{5}{4Hlog(T)n_h^t(s,a)}\sum_{k=t(s)}^t \chi_h^k(s,a) p_h(\mathring{\overline{V}}_{h+1}^{k-1}-V_{h+1}^{\pi^k})(s,a). \label{beta_part2}
\end{align}

Combining eq. (\ref{beta_part1}) and eq. (\ref{beta_part2}) we obtain the result of this lemma.
\end{proof}

Now we are ready to prove Theorem \ref{theorem_1}.

\begin{proof}[Proof of Theorem \ref{theorem_1}]
We assume the event $\mathcal{D}$ and Homeland condition in Section \ref{sec:EMDP-GA}.

% \jt{?}
From Lemma \ref{P(D)}, we have that $\mathcal{D}$ and Homeland condition in Section \ref{sec:EMDP-GA} hold with probability at least $1-\delta$.
Fix $(s,a,h) \in \mathcal{S}\times \mathcal{A} \times [H]$ and $t \in [T]$.

\textbf{Step 1: Upper-bound $(\mathring{\overline{Q}}_h^t-Q_h^{\pi^{t+1}})(s,a)$}\ \ \ \ Applying Lemmas \ref{all_Q_estimate} and \ref{beta_upper_bound}, we get
\begin{align}
(\mathring{\overline{Q}}_h^t-Q_h^{\pi^{t+1}})(s,a) \le& p_h(\mathring{V}^t_{h,s,a}-V_{h+1}^{\pi^{t+1}})(s,a)+130H^3\frac{log(T)\zeta}{\tilde{n}_h^t(s,a)} \nonumber\\
&+\sqrt{\frac{30}{\tilde{n}_h^t(s,a)}\sum_{k=t(s)}^t\chi_h^k(s,a)Var_{p_h}(V_{h+1}^{\pi^k})(s,a)\frac{\zeta}{\tilde{n}_h^t(s,a)}}\nonumber\\
&+\frac{5}{Hlog(T)\tilde{n}_h^t(s,a)}\sum_{k=t(s)}^t\chi_h^k(s,a)p_h(\mathring{\overline{V}}_{h+1}^{k-1}-V_{h+1}^{\pi^k})(s,a).\label{local_regret_term_bound}
\end{align}
\textbf{Step 2: Upper-bound the local optimistic regret $\hat{R}_h^T(s,a)$}\ \ \ \ We define
\begin{align*}
\hat{R}_h^T(s,a) &\triangleq \sum_{t=0}^{T-1} \chi_h^{t+1}(s,a)(\mathring{\overline{Q}}_h^t-Q_h^{\pi^{t+1}})(s,a)\\
&=\sum_{t=t(s)-1}^{T-1} \chi_h^{t+1}(s,a)(\mathring{\overline{Q}}_h^t-Q_h^{\pi^{t+1}})(s,a).
\end{align*}

Based on this, eq. (\ref{local_regret_term_bound}) yields the follows,
\begin{align*}
&\ \ \ \ \hat{R}_h^T(s,a) \\
&\le \sum_{t=t(s)-1}^{T-1} \chi_h^{t+1}(s,a)p_h(\mathring{V}^t_{h,s,a}-V_{h+1}^{\pi^{t+1}})(s,a)+130H^3\sum_{t=t(s)-1}^{T-1} \chi_h^{t+1}(s,a)\frac{log(T)\zeta}{\tilde{n}_h^t(s,a)} \\
&+\sum_{t=t(s)-1}^{T-1} \chi_h^{t+1}(s,a)\sqrt{\frac{30}{\tilde{n}_h^t(s,a)}\sum_{k=t(s)}^t\chi_h^k(s,a)Var_{p_h}(V_{h+1}^{\pi^k})(s,a)\frac{\zeta}{\tilde{n}_h^t(s,a)}}\\
&+\sum_{t=t(s)-1}^{T-1} \chi_h^{t+1}(s,a)\frac{5}{Hlog(T)\tilde{n}_h^t(s,a)}\sum_{k=t(s)}^t\chi_h^k(s,a)p_h(\mathring{\overline{V}}_{h+1}^{k-1}-V_{h+1}^{\pi^k})(s,a).
\end{align*}

For the first term, we have the decomposition
\begin{align}
p_h(\mathring{V}^t_{h,s,a}-V_{h+1}^{\pi^{t+1}})(s,a)=p_h(\mathring{V}^t_{h,s,a}-V^*_{h+1})(s,a)+p_h(V^*_{h+1}-V_{h+1}^{\pi^{t+1}})(s,a).\label{term1_decomposition}
\end{align}

Then, applying eq. (\ref{V_induction_2}) and Lemma \ref{tilde_eta_results}, we obtain
\begin{align}
&\ \ \ \ \sum_{t=t(s)-1}^{T-1} \chi_h^{t+1}(s,a)p_h(\mathring{V}^t_{h,s,a}-V_{h+1}^*)(s,a)\nonumber\\
&=\sum_{t=t(s)-1}^{T-1} \chi_h^{t+1}(s,a) \mathbb{I}_{\{n_h^t(s,a)=0\}}p_h(\mathring{V}^t_{h,s,a}-V_{h+1}^*)(s,a)\nonumber\\
&\ \ \ \ +\sum_{t=t(s)-1}^{T-1} \chi_h^{t+1}(s,a) \mathbb{I}_{\{n_h^t(s,a)>0\}}\sum_{k=t(s)}^{t}p_h(\mathring{\overline{V}}^{k-1}_{h+1}-V_{h+1}^*)(s,a)\nonumber\\
&\le H+\sum_{k=t(s)}^{T-1} (\sum_{t=k}^{T-1}\chi_h^{t+1}(s,a) \tilde{\eta}_h^{t,k}(s,a))p_h(\mathring{\overline{V}}^{k-1}_{h+1}-V_{h+1}^*)(s,a)\nonumber\\
&\le H+(1+\frac{1}{H})\sum_{t=t(s)-1}^{T-1}\chi_h^{t+1}(s,a)p_h(\mathring{\overline{V}}_{h+1}^{t-1}-V^*_{h+1})(s,a).\label{term1_upper_bound}
\end{align}

Combining eq. (\ref{term1_decomposition}), eq. (\ref{term1_upper_bound}), and $V^*_{h+1} \ge V_{h+1}^{\pi^{k+1}}$, we get
\begin{align}
&\ \ \ \ \sum_{t=0}^{T-1} \chi_h^{t+1}(s,a)p_h(\mathring{V}^t_{h,s,a}-V_{h+1}^{\pi^{t+1}})(s,a)  \nonumber\\
&\le \sum_{t=0}^{T-1} \chi_h^{t+1}(s,a)p_h(V^*_{h+1}-V_{h+1}^{\pi^{t+1}})(s,a)\nonumber\\
&\ \ \ \ +H+(1+\frac{1}{H})\sum_{t=0}^{T-1}\chi_h^{t+1}(s,a)p_h(\mathring{\overline{V}}_{h+1}^{t-1}-V^*_{h+1})(s,a)\nonumber\\
&\le H+(1+\frac{1}{H})\sum_{t=0}^{T-1}\chi_h^{t+1}(s,a)p_h(\mathring{\overline{V}}_{h+1}^{t-1}-V_{h+1}^{\pi^{t+1}})(s,a). \label{term1_bound}
\end{align}

For the fourth term, we can proceed in a similar way but using Lemma \ref{correction_term_bound} to get
\begin{align*}
&\ \ \ \ \sum_{t=0}^{T-1}\frac{\chi_h^{t+1}(s,a)}{\tilde{n}_h^t(s,a)} \sum_{k=t(s)}^t \chi_h^k(s,a) p_h(\mathring{\overline{V}}_{h+1}^{k-1}-V_{h+1}^{\pi^k})(s,a)\\
&\le \sum_{k=t(s)}^{T-1} (\sum_{t=k}^{T-1} \frac{\chi_h^{t+1}(s,a)}{\tilde{n}_h^t(s,a)}) \chi_h^k(s,a)p_h(\mathring{\overline{V}}_{h+1}^{k-1}-V_{h+1}^{\pi^k})(s,a)\\
& \le 8log(T) \sum_{k=1}^{T-1} \chi_h^k(s,a) p_h(\mathring{\overline{V}}_{h+1}^{k-1}-V_{h+1}^{\pi^k})(s,a).
\end{align*}

Thus, we have
\begin{align}
\ \ \ \ &\sum_{t=t(s)-1}^{T-1}\frac{5\chi_h^{t+1}(s,a)}{Hlog(T)\tilde{n}_h^t(s,a)} \sum_{k=t(s)}^t \chi_h^k(s,a) p_h(\mathring{\overline{V}}_{h+1}^{k-1}-V_{h+1}^{\pi^k})(s,a) \nonumber\\
&\le \frac{40}{H} \sum_{t=t(s)-1}^{T-1}\chi_h^{t+1}(s,a) p_h(\mathring{\overline{V}}_{h+1}^t-V_{h+1}^{\pi^{t+1}})(s,a). \label{term4_bound}
\end{align}

For the third term, combining Cauchy-Schwarz inequality, Lemmas \ref{correction_term_bound} and \ref{weight_lemma}, we get
\begin{align}
&\ \ \ \ \sum_{t=t(s)-1}^{T-1} \chi_h^{t+1}(s,a)\sqrt{\frac{30}{\tilde{n}_h^t(s,a)}\sum_{k=t(s)}^t\chi_h^k(s,a)Var_{p_h}(V_{h+1}^{\pi^k})(s,a)\frac{\zeta}{\tilde{n}_h^t(s,a)}} \nonumber\\
&\le \sqrt{30\sum_{t=t(s)-1}^{T-1}\frac{\chi_h^{t+1}(s,a)}{\tilde{n}_h^t(s,a)}\sum_{k=t(s)}^t Var_{p_h}(V_{h+1}^{\pi^k})(s,a)} \sqrt{\sum_{t=t(s)-1}^{T-1}\chi_h^{t+1}(s,a) \frac{\zeta}{\tilde{n}_h^t(s,a)}}\nonumber\\
&\le44log(T) \sqrt{\zeta \sum_{t=t(s)-1}^{T-1} \chi_h^{t+1}(s,a)Var_{p_h}(V_{h+1}^{\pi^{t+1}})(s,a)}.\label{term3_bound}
\end{align}

For the second term, we use Lemma \ref{weight_lemma} to get
\begin{align}
130H^3\sum_{t=t(s)-1}^{T-1} \chi_h^{t+1}(s,a)\frac{log(T)\zeta}{\tilde{n}_h^t(s,a)} \le 1040H^3log(T)^2\zeta. \label{term2_bound}
\end{align}

Finally, we combine eq. (\ref{term1_bound}), eq. (\ref{term4_bound}), eq. (\ref{term3_bound}) and eq. (\ref{term2_bound}) to conclude
\begin{align}
&\ \ \ \ \hat{R}_h^T(s,a) \nonumber\\
&\le 44log(T)\sqrt{\zeta \sum_{t=t(s)-1}^{T-1} \chi_h^{t+1}(s,a)Var_{p_h}(V_{h+1}^{\pi^{t+1}})(s,a)}+1041H^3log(T)^2\zeta \nonumber\\
&+(1+\frac{41}{H}) \sum_{t=t(s)-1}^{T-1}\chi_h^{t+1}(s,a) p_h(\mathring{\overline{V}}_{h+1}^t-V_{h+1}^{\pi^{t+1}})(s,a).\label{step2_result}
\end{align}

\textbf{Step 3: Replace $\chi_h^t$ with $\overline{p}_h^t$ in the upper-bound on $\hat{R}_h^T(s,a)$}\ \ \ \ Since $\mathcal{G}$ holds, we have
\begin{align*}
&\ \ \ \ \ \sqrt{\sum_{t=t(s)-1}^{T-1}\chi_h^{t+1}(s,a)Var_{p_h}(V_{h+1}^{\pi^{t+1}})(s,a)} \\
&\le \sqrt{2\sum_{t=t(s)-1}^{T-1}\overline{p}_h^{t+1}(s,a)Var_{p_h}(V_{h+1}^{\pi^{t+1}})(s,a)} +\sqrt{8\zeta}H
\end{align*}
and
\begin{align*}
&\ \ \ \ \sum_{t=t(s)-1}^{T-1}\chi_h^{t+1}(s,a)p_h(\mathring{\overline{V}}_{h+1}^t-V_{h+1}^{\pi^{t+1}})(s,a) \\
&\le (1+\frac{1}{H})\sum_{t=t(s)-1}^{T-1} \overline{p}_h^{t+1}(s,a)p_h(\mathring{\overline{V}}_{h+1}^t-V_{h+1}^{\pi^{t+1}})(s,a)+14H^2\zeta.
\end{align*}

Plugging the two inequalities in eq. (\ref{step2_result})
yields
\begin{align}
\hat{R}_h^T(s,a) \le &63log(T)\sqrt{\zeta \sum_{t=t(s)-1}^{T-1} \overline{p}_h^{t+1}(s,a)Var_{p_h}(V_{h+1}^{\pi^{t+1}})(s,a)} \nonumber\\
&+1754H^3log(T)^2\zeta\nonumber\\
&+(1+\frac{83}{H}) \sum_{t=t(s)-1}^{T-1}\overline{p}_h^{t+1}(s,a) p_h(\mathring{\overline{V}}_{h+1}^t-V_{h+1}^{\pi^{t+1}})(s,a).\label{local_regret_result2}
\end{align}

\textbf{Step 4: Upper-bound the step $h$ regret $\hat{R}_h^T$}\ \ \ \ We define
$$
\hat{R}_h^T \triangleq \sum_{s\in S}\sum_{t= max\{t(s)-1, 0\}}^{T-1}\overline{p}_h^{t+1}(s)(\mathring{\overline{V}}_h^{t}-V_h^{\pi^{t+1}})(s).
$$
We use the definition of $\mathring{\overline{V}}_h^k(s)$ and the facts that $\mathcal{D}$ holds, $\mathring{\overline{Q}}_h^k \ge 0$ on $\mathcal{D}$ and  that for all $x>1$ the following holds $\frac{1}{1-\frac{1}{4x}} \le 1+\frac{1}{x}$ to get
\begin{align*}
&\ \ \ \ \sum_{max\{t(s)-1, 0\}}^{T-1}\overline{p}_h^{t+1}(s)(\mathring{\overline{V}}_h^{t}-V_h^{\pi^{t+1}})(s)\\
% &=\sum_{t=0}^{(t(s)-2) \vee 0}\overline{p}_h^{t+1}(s)(\mathring{\overline{V}}_h^{t}-V_h^{\pi^{t+1}})(s)+\sum_{t=t(s)-1}^{T-1}\overline{p}_h^{t+1}(s)(\mathring{\overline{V}}_h^{t}-V_h^{\pi^{t+1}})(s)\\
&\le\frac{1}{1-\frac{1}{4H}}\sum_{t=max\{t(s)-1, 0\}}^{T-1}\chi_h^{t+1}(s)(\mathring{\overline{V}}_h^{t}-V_h^{\pi^{t+1}})(s)+\frac{56}{3}H^2\zeta\\
&\le (1+\frac{1}{H})\sum_{t=max\{t(s)-1, 0\}}^{T-1}\chi_h^{t+1}(s)(\mathring{\overline{V}}_h^{t}-V_h^{\pi^{t+1}})(s)+19H^2\zeta\\
&\le (1+\frac{1}{H})\sum_{t=max\{t(s)-1, 0\}}^{T-1}\chi_h^{t+1}(s)\pi_h^{t+1}(\mathring{\overline{Q}}_h^{t}-Q_h^{\pi^{t+1}})(s)+19H^2\zeta.
\end{align*}

Using Cauchy-Schwarz inequality,eq. (\ref{local_regret_result2}) and the fact that the policy $\pi^t$ is deterministic, we obtain
\begin{align}
\hat{R}_h^T&\le (1+\frac{1}{H}) \sum_{s\in S}\sum_{t=max\{t(s)-1, 0\}}^{T-1}\chi_h^{t+1}(s)\pi_h^{t+1}(\mathring{\overline{Q}}_h^{t}-Q_h^{\pi^{t+1}})(s)+19H^2S\zeta \nonumber\\
&=(1+\frac{1}{H}) \sum_{(s,a)\in S\times A}\sum_{t=max\{t(s)-1, 0\}}^{T-1}\chi_h^{t+1}(s,a)(\mathring{\overline{Q}}_h^{t}-Q_h^{\pi^{t+1}})(s,a)+19H^2S\zeta\nonumber\\
&=(1+\frac{1}{H})\sum_{(s,a)\in S\times A} \hat{R}_h^T(s,a)+19H^2S\zeta\nonumber\\
&\le 126log(T)\sqrt{\zeta}\sum_{(s,a)\in S\times A}\sqrt{\sum_{t=max\{t(s)-1, 0\}}^{T-1}\overline{p}_h^{t+1}(s,a)Var_{p_h}(V_{h+1}^{\pi^{t+1}})(s,a)}\nonumber\\
&\ \ \ \ +(1+\frac{167}{H})\sum_{(s,a)\in S\times A} \sum_{t=max\{t(s)-1, 0\}}^{T-1}\overline{p}_h^{t+1}(s,a) p_h(\mathring{\overline{V}}_{h+1}^t-V_{h+1}^{\pi^{t+1}})(s,a)\nonumber\\
&\ \ \ \ +3527SAH^3log(T)^2\zeta\nonumber\\
&\le HS\sqrt{T}+3527SAH^3log(T)^2\zeta +(1+\frac{167}{H})\hat{R}_{h+1}^T\nonumber\\
&\ \ \ \ +126log(T)\sqrt{\zeta SA \sum_{(s,a)\in S\times A}\sum_{t=max\{t(s)-1, 0\}}^{t-1}\overline{p}_h^{t+1}(s,a)Var_{p_h}(V_{h+1}^{\pi^{t+1}})(s,a)}.\label{step_h_bound}
\end{align}

For the last inequality, we have
\begin{align*}
&\ \ \ \ \sum_{(s,a)\in S\times A} \sum_{t=max\{t(s)-1, 0\}}^{T-1}\overline{p}_h^{t+1}(s,a) p_h(\mathring{\overline{V}}_{h+1}^t-V_{h+1}^{\pi^{t+1}})(s,a)\\
&=\sum_{(s,a,s')\in S\times A \times S} \sum_{t=max\{t(s)-1, 0\}}^{T-1}\overline{p}_h^{t+1}(s,a) p_h(s'|s,a)(\mathring{\overline{V}}_{h+1}^t-V_{h+1}^{\pi^{t+1}})(s')\\
&\le \sum_{s'\in S}\sum_{(s,a)\in S\times A} \sum_{t=0}^{T-1}\overline{p}_h^{t+1}(s,a) p_h(s'|s,a)(\mathring{\overline{V}}_{h+1}^t-V_{h+1}^{\pi^{t+1}})(s')\\
&= \sum_{s'\in S} \sum_{t=0}^{T-1} \overline{p}_{h+1}^{t+1}(s')(\mathring{\overline{V}}_{h+1}^t-V_{h+1}^{\pi^{t+1}})(s')=\hat{R}_{h+1}^T.
\end{align*}

\textbf{Step 5: Upper-bound the regret $\mathcal{R}^T$}\ \ \ \ Unfolding eq. (\ref{step_h_bound}), using Cauchy-Schwarz inequality, Lemma \ref{law_of_total_variance} and the fact that $\hat{R}^T_{H+1}=0$, we get
\begin{align*}
\hat{R}_1^T &\le\sum_{h=1}^H(1+\frac{167}{H})^{H-h}\\
&\ \ \ \ \cdot[126log(T)\sqrt{\zeta SA \sum_{(s,a)\in S\times A}\sum_{t=t(s)-1}^{t-1}\overline{p}_h^{t+1}(s,a)Var_{p_h}(V_{h+1}^{\pi^{t+1}})(s,a)}\\
&\ \ \ \ +3527SAH^3log(T)^2\zeta]\\
&\le e^{167}126log(T)\sqrt{\zeta SAH \sum_{(s,a,h)\in S\times A \times [H]}\sum_{t=0}^{t-1}\overline{p}_h^{t+1}(s,a)Var_{p_h}(V_{h+1}^{\pi^{t+1}})(s,a)}\\
&\ \ \ \ +3527e^{167}SAH^4log(T)^2\zeta\\
&=e^{167}126log(T)\sqrt{\zeta SAH\sum_{t=0}^{T-1} \mathbb{E}_{\pi^{t+1}}[(\sum_{h=1}^Hr_h(s_h,a_h)-V_1^{\pi^{t+1}}(s_1))^2]}\\
&\ \ \ \ +3527e^{167}SAH^4log(T)^2\zeta\\
&\le \tilde{\mathcal{O}}(\sqrt{\zeta H^3SAT})+\tilde{\mathcal{O}}(\zeta H^4SA).
\end{align*}

Applying Lemma \ref{optimism_lemma}, we have
$$
V^*_1(s_1)-V_h^{\pi^{t+1}}(s_1) \le \mathring{\overline{V}}^t_1(s_1)-V_h^{\pi^{t+1}}(s_1).
$$

This allow us to conclude
$$
\mathcal{R}^T \le \hat{R}_1^T \le \tilde{\mathcal{O}}(\sqrt{\zeta H^3SAT})+\tilde{\mathcal{O}}(\zeta H^4SA)+\tilde{\mathcal{O}}(H^2S\sqrt{T}).
$$
\end{proof}

\section{Preparation Lemmas}
The proof of the following results can be seen in M\'{e}nard et al.\cite{menard2021ucb}.
For a deterministic policy $\pi$, we define Bellman-type equations for the variances as follows
$$
\sigma Q_h^{\pi}(s,a) \triangleq Var_{p_h}V_{h+1}^{\pi}(s,a)+p_h \sigma V_{h+1}^{\pi}(s,a)
$$
$$
\sigma V_h^{\pi}(s) \triangleq \sigma Q_h^{\pi}(s,\pi(s))
$$
$$
\sigma V_{H+1}^{\pi}(s) \triangleq 0,
$$
where $Var_{p_h}(f)(s,a) \triangleq\mathbb{E}_{s' \sim p_h(\cdot|s,a)}[(f(s')-p_hf(s,a))^2]$.

The definition above indicates that
$$
\sigma V_1^{\pi}(s_1)=\sum_{h=1}^H \sum_{s,a} p_h^{\pi}(s,a)Var_{p_h}(V_{h+1}^{\pi})(s,a).
$$

We have the following lemmas.

\begin{lemma}
\label{law_of_total_variance}
For any deterministic policy $\pi$ and for all $h\in [H]$,
$$
\mathbb{E}_{\pi}[(\sum_{h'=h}^Hr_{h'}(s_{h'},a_{h'})-Q_h(s_h,a_h))^2|(s_h,a_h)=(s,a)]=\sigma Q_h^{\pi}(s,a).
$$
Particularly,
$$
\mathbb{E}_{\pi}[(\sum_{h=1}^Hr_h(s_h,a_h)-V_1^{\pi}(s_1))^2]=\sigma V_1^{\pi}(s_1)=\sum_{h=1}^H \sum_{s,a}p_h^{\pi}(s,a)Var_{p_h}(V_{h+1}^{\pi})(s,a).
$$
\end{lemma}

\begin{lemma}
\label{weight_lemma}
For $T \in \mathbb{N}^*$ and $(u_t)_{t \in \mathbb{N}^*}$, for a sequence where $u_t \in [0,1]$ and $U_t \triangleq \sum_{l=1}^t u_l$, we get
$$
\sum_{t=1}^T \frac{u_{t+1}}{U_t \vee1} \le 4log(U_{T+1}+1).
$$
In particular if $T+1 \ge 2$,
$$
\sum_{t=1}^T \frac{u_{t+1}}{U_t \vee1} \le 8log(T+1).
$$
\end{lemma}

\begin{lemma}
\label{tilde_eta_results}
For all $(s,a) \in S \times A$ the following holds
$$
\sum_{k=l}^t\chi_h^{k+1}(s,a)\tilde{\eta}_h^{k,l}(s,a) \le (1+\frac{1}{H})\chi_h^l(s,a),
$$
$$
\sum_{k=t(s)}^t \tilde{\eta}_h^{t,k}(s,a)=1\ \ \ \ \text{if $n_h^t(s,a)>0$}.
$$
\end{lemma}

\begin{lemma}
\label{correction_term_bound}
For all $(s,a) \in S \times A$ and $t \le T-1$ (with $T \ge 2$), the following holds
$$
\chi_h^l(s,a)\sum_{k=l}^t \frac{\chi_h^{k+1}(s,a)}{\tilde{n}_h^k(s,a)} \le \sum_{k=0}^{T-1}\frac{\chi_h^{k+1}(s,a)}{\tilde{n}_h^k(s,a)}=\sum_{k=t(s)-1}^{T-1}\frac{\chi_h^{k+1}(s,a)}{\tilde{n}_h^k(s,a)} \le 8log(T).
$$
\end{lemma}

\begin{lemma}
\label{Variance_inequality}
For $p,q \in \sum_S$, for $f,g: S \mapsto[0,b]$ two functions defined on $S$, we have that
$$
Var_p(f)\le2Var_p(g)+2bp|f-g|,\ \ \ 
Var_p(f^2) \le 2bVar_p(f),
$$
where we denote the absolute operator by $|f|(s)=|f(s)|$ for all $s \in S$.
\end{lemma}





%% For numbered reference style

% \end{thebibliography}
\end{document}

%% file: 1.introduction.tex
\section{Introduction}
\label{sec::introduction}
Reinforcement learning aims to train an agent that takes sequential actions, moving her state in an environment, 
%In reinforcement learning, an agent interacts with its environment and optimizes a policy to 
for maximising her cumulative rewards received from interactions with the environment \cite{sutton1998reinforcement}. 
A major model in reinforcement learning is Markov decision process (MDP), assuming the state transitions and rewards from actions do not depend on the history given the current state  \cite{white1989markov}.
Further, episodic MDP (EMDP) allows varying rewards from executing the same action at the same state but in different timings \cite{zimin2013online}.
%Existing research typically assumes that the agent possesses complete knowledge of the environment from the outset.
%However, in many practical applications, it is unrealistic to have full information prior to extensive exploration.
In parallel, value-based algorithms %, a large family of reinforcement learning algorithms, 
learn value functions $Q$ and $V$ over the space of $(\text{state}, \text{action})$-pairs and the state space, respectively, which then deliver a policy model for the agent, taking the actions that maximise $Q$ \cite{sutton1998reinforcement}. The domain of value functions is naturally restricted by the awareness of all possible states -- mathematically, they are defined over all aware states, termed as {\it aware domain} in this paper. Correspondingly, the exterior of the aware domain is termed as {\it unaware domain}.

%\begin{wrapfigure}{r}{0.4\textwidth} \label{sail}
%    \centering
%    \includegraphics[width=0.4\textwidth]{figures/sail.pdf}
    
    % \caption{A ship departs from a harbor, only aware of the surroundings -- the light area. The ship knows nothing about the dark area and needs to explore the world.}
%\end{wrapfigure}

Consider a scenario that is completely untouched in the literature of reinforcement learning -- an agent might surprisingly find she has reached {an unknown state} where she is even unaware of. The former US Secretary of Defense Donald Rumsfeld referred to this scenario as an unknown unknown in a news conference \cite{press}. %\jt{?}. %This scenario has been found in broad, security-critical areas, including, %enhancing decision-making under uncertainty and improving the performance of diverse systems, ranging from autonomous navigation and intelligent sensor networks to adaptive control systems.
Unknown unknowns broadly exist in security-critical applications, such as autonomous vehicles \cite{van2018autonomous}, %smart-home devices \cite{yang2021towards}, 
quantitative finance \cite{liu2021finrl} and economics \cite{chang2020decentralized}, when the environment is extremely under-explored and uncertain.
%An example of UU is shown in the figure. Consider a ship departing from a harbor with awareness limited to its immediate vicinity—marked as the light area. The vessel has no knowledge of the dark area, which is an UU, and must explore to acquire information.
% \fh{This scenario has been found in broad, security-critical areas, including xxx}
% \jt{addressed}
%In this work, we consider a setting in which the agent continuously gathers information during exploration—a phenomenon we term 'growing awareness'.
%During exploration, the agent may encounter an unknown state that it is unaware of - an unknown unknown (UU).
%Consequently, a natural question arises:
%\begin{quote}
%\textit{Can we design an algorithm that effectively handles scenarios involving UU states?}
%\end{quote}
% Consequently, the agent must dynamically adapt its policy based on this evolving understanding, presenting a novel challenge.

%Value-based methods are widely employed in reinforcement learning, where the policy is derived from the value functions $Q$ and $V$.
%In our setting, the agent has no information regarding UU states and is even unaware of their existence.
%The agent's awareness of the state space determines the domains of $Q$ and $V$; for UU states, we provide prior estimates for these functions.
%Thus, we can view $Q$ and $V$ as being defined over the union of their known domains and the set of UU states.
%At each state, the agent chooses an action  that maximizes the expected reward based on $Q$ and $V$ and it may subsequently transition into a UU state outside these domains.

\begin{figure}[t]
\centering
\includegraphics[scale=0.45]{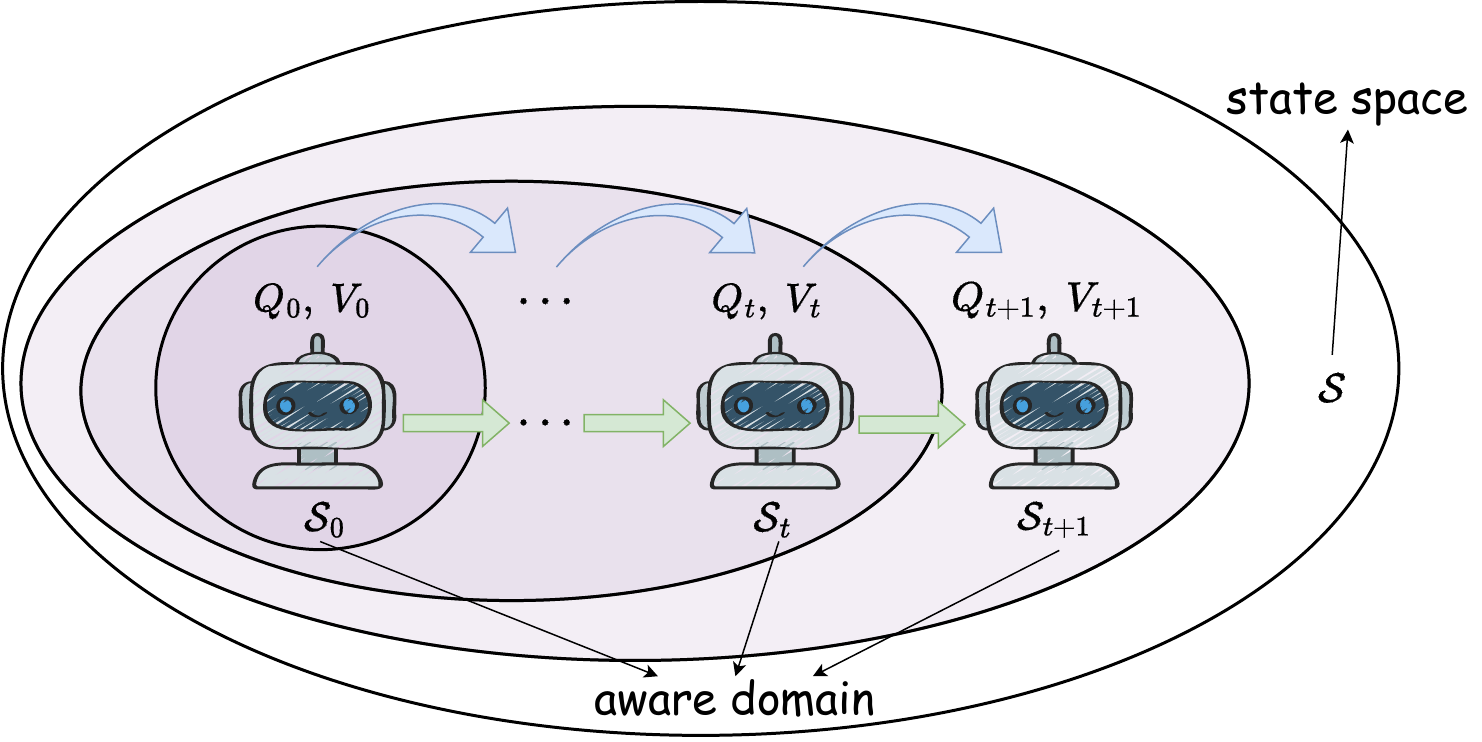}
\caption{Growing awareness}
\label{growing awareness}
\end{figure}

This paper proposes the first mathematical model to ground unknown unknowns in reinforcement learning. Although actions are calculated by maximising value functions $Q$ and $V$ that are {\it defined} on the aware domain, some of them could move the agent to {unknown unknown states in the unaware domain}.
%To empower the agent to handle this surprise, 
Based upon this, this paper designs a novel model, {\it episodic Markov decision {process} with growing awareness} (EMDP-GA), which %, that facilitates the awareness in an extremely under-explored world.
%When the agent goes to a UU state, outside of the aware domains, EMDP-GA 
employs a newly designed {\it noninformative value expansion} (NIVE) approach to expand value functions to the newly aware states, when the agent surprisingly reaches there. %goes to a UU state, outside of the aware domains, EMDP-GA. %, through which expands the definition domains of value functions.
{Figure \ref{growing awareness} illustrates the expansion of the aware domain, where $\mathcal{S}$ is the whole state space and $\mathcal{S}_t$ is the aware domain after exploring $t$ episodes.}
The initial values are chosen as a noninformative prior \cite{tibshirani1989noninformative}, such as the {averaged} values on the previous aware domain, out of respect for the extreme lack of knowledge in this unknown unknown state.
%Consequently, the policy—based on the current knowledge—must be continuously adapted to incorporate newly discovered states.
%In order to address the challenges posed by EMDP-GA, we initialize the value functions for UU states using average values, as no information is available for these states.
%Based upon Upper Confidence Bound Momentum Q-learning (UCBMQ), w
We further adapt %a new algorithm for training EMDP-GA, named 
%{\it upper confidence bound momentum Q-learning with growing awareness} (UCBMQ-GA) generalised upon 
upper confidence bound momentum Q-learning (UCBMQ) \cite{menard2021ucb} to the growing awareness for training the EMDP-GA, termed as UCBMQ-GA.
UCBMQ %is an adaption of Q-learning, using the 
employs a momentum term to correct the bias in Q-learning, but relies on the upper bounds of value functions $Q$ and $V$. 
{We prove that our NIVE can preserve this `upper‑bound' property when expanding the value functions, which will be invalidated by multiplying NIVE's expansion by any constant $d < 1$.} This confirms the efficiency of NIVE as well as our adaptation made in EMDP-GA and UCBMQ-GA.

%mitigating the negative impact on the regret. 
%\fh{This adaptation is non-trivial, because xxx.}
%UCBMQ-GA is a value-based method and the value functions $Q$ and $V$ are defined on the aware domain.
%The main idea %of UCBMQ-GA 
%is when the aware domains grow, the average values over the old aware domains are used to initialise $Q$ and $v$ over new states.
%The main idea of UCBMQ-GA is combining NIVE and UCBMQ to train EMDP-GA.
%UCBMQ-GA adapts an iterative two-step approach: (1) expanding value functions to newly discovered area for updating awareness, and (2) Updating functions using average values over known states.
%The core innovation of UCBMQ-GA is the use of NIVE for function updates, providing stability and reliability.
%For the training process, UCBMQ-GA follows the same methodology as UCBMQ.

%UCBMQ-GA extends the original Upper Confidence Bound Momentum Q-learning (UCBMQ)—designed for traditional reinforcement learning settings with complete state space information—to accommodate the progressive expansion of the agent’s awareness.
% \fh{describing your UCBMQ-GA. don't be too abstract}\jt{addressed}
%This UCBMQ-GA is inspired by upper confidence bound momentum Q-learning (UCBMQ), but can
%Compared to UCBMQ, the primary advantage of UCBMQ-GA is its capacity to 
%automatically adapt the policy as the agent’s aware domain expands.

Extensive theoretical analysis is conducted for verifying our approach.
%
% \begin{itemize}[leftmargin=12pt]
%     \item \textbf{Uncertainty in exposure to unknown unknown.}
% We define an {\it aware confidence} measure for quatifying the uncertainty of our agent. It is defined as %the estimates to values and optimal policy model, as 
% the gap between the estimators and the hidden optimal models on the awareness domain.
% The gap quantifies the margin available for policy improvement and thus serves as a measure of the agent’s confidence.
% Because confidence is inherently subjective and limited by the agent’s awareness, it is defined only over the current awareness domain.
% We prove that encountering an unknown unknown, an agent's aware confidence to her understanding over the whole universe suddenly drops. This is inevitable, but we theoretically justify that our NIVE has a decent property. 
%\begin{itemize}[leftmargin=12pt]
%    \item \textbf{Decent property of NIVE.}
%\jt{We theoretically justify that our NIVE has a decent property.
%In our framework, we maintain an upper bound on the optimal value function, and NIVE preserves this upper‑bound property when expanding the functions.
%Furthermore, we show that using any scalar multiple $d <1$ of the average values to perform the expansion may violate the upper‑bound guarantee.
%Therefore, employing the average values directly offers a robust and principled expansion strategy.}
%
%\item \textbf{Asymptotically consistent regret.} 
We analyse the regret over the whole course, which is defined as { the cumulative reward loss {over the whole learning process} against the hidden optimal policy}. The regret has an upper bound in order of 
%Afterwards, the  regret 
%Notably, the regret bound of UCBMQ-GA remains sublinear in $T$, i.e. 
$\tilde{\mathcal{O}}(\sqrt{ H^3SAT}+ H^4SA+H^2S\sqrt{T})$, where $S$ is the size of the state space, $A$ is the size of the action space, $T$ is the number of episodes, and $H$ is the length of an episode. %, and an episode is a sequence of interaction with the environment. 
This regret bound is sublinear with respect to $T$, and comparable with the state of the art (SOTA), including upper confidence bound-advantage \cite{zhang2020almost}, upper confidence bound momentum Q-learning \cite{menard2021ucb} and monotonic value propagation \cite{zhang2024settling}, despite our exposure to unknown unknowns, while they are not. %\fh{episode is defined as xxx}\jt{addressed} %and $T$ is the number of episodes.
%This result is consistent with the regret bound of UCBMQ.
%\item \textbf{Decent speed and affordable cost.} 
{The computing complexity and space complexity of employing the UCBMQ-GA {approach} for training the EMDP-GA model are of order $\mathcal{O}(H(S+A)T)$ and $\mathcal{O}(HS^2A)$, respectively, which are comparable with the SOTA \cite{menard2021ucb}.} %without exposure to unknown unknowns.
%\end{itemize}
 %so it dramatically hurts confidence and certainty of an agent, 
%
These results collectively suggest that although unknown unknowns are surprising, they can nonetheless be discovered asymptotically properly with a decent speed and an affordable cost.

%% file: 2.related_works.tex
\section{Related Works}
%\label{sec::related_works}

\paragraph{Episodic Markov decision process (EMDP)}
A large volume of literature has been on developing EMDP over the last decade.
Neu et al. study the adversarial stochastic shortest path problem in EMDPs \cite{neu2012adversarial}.
Zimin and Neu investigate EMDPs with a layered structure \cite{zimin2013online} .
Dann and Brunskill study EMDPs from the perspective of probably approximately correct (PAC) learning \cite{dann2015sample}.
% Typically, we note the well-known minimax regret lower bound, $\mathop{\min} \{ \sqrt{SAH^3T},HT \}$, established by \cite{domingues2021episodic} to offer a theoretical benchmark.
%We now review several related theoretical works on 
EMDP algorithms can be roughly classified into model-based and model-free approaches.
Model-based approaches learn the EMDP model prior to or concurrently with policy optimisation and require $\Omega(S^2AH)$ memory space to store the value function estimators \cite{moerland2023model}, in which recent developments have achieved lower regret and burn-in costs \cite{azar2017minimax, dann2019policy, zanette2019tighter, efroni2019tight, zhang2020almost, zhang2021reinforcement, zhang2024settling, li2024breaking}.
Here, burn-in costs are the
minimum sample size needed for an algorithm to operate sample-optimally.
In contrast, model-free approaches estimate the optimal value function or optimise the policy directly, without explicit model estimation, and require only $O(SAH)$ memory to store these estimators.
Extensive theoretical analyses of model-free approaches have also been conducted \cite{zhang2020almost, menard2021ucb, mao2021near, dann2021provably, tiapkin2022dirichlet, li2023breaking, agrawal2024optimistic}.
Amongst them, a remarkable result is a minimax regret lower bound, $\mathop{\min} \{ \sqrt{SAH^3T},HT \}$, established by Domingues et al. \cite{domingues2021episodic}, offering a theoretical benchmark.

\paragraph{Awareness and reverse Bayes}
% Awareness has been used to describe the knowledge about the decision universe, including acts and preference relations.
% Reverse Bayesianism has been proposed and developed to model the expanding universe of decision as the awareness grows, which has good theoretical results and lay the theoretical basis of later works about awareness(\cite{karni2013reverse}; \cite{karni2017awareness}; \cite{karni2015probabilistic}).
% Belief formation and reactions to unforeseen events according to reverse Bayesianism has also been experimentally studied (\cite{becker2020reverse}).
% This framework has been applied in some traditional economic scenarios, such as intertemporal asset pricing problem (\cite{viero2021intertemporal}).
% However, awareness has not been introduced in reinforcement learning.

Similar terms are seen in the literature of %Awareness has been used to characterize the knowledge of the decision universe, encompassing both actions and preference relations.
%Considering the growing awareness, there have been many works about 
choice theory \cite{grant2015preference,  piermont2017introspective, kochov2018behavioral}. 
Reverse Bayesianism is proposed to model the expanding decision universe as awareness grows %, yielding robust theoretical results that underpin subsequent work on awareness 
\cite{karni2013reverse,  karni2015probabilistic, karni2017awareness}.
Belief formation and responses to unforeseen events, as predicted by reverse Bayesianism, are experimentally studied \cite{becker2020reverse}.
This framework was then applied in economic theory \cite{viero2021intertemporal, chakravarty2022reverse, karni2021reverse} and epistemology \cite{zednik2016bayesian}.
However, no practical {approach} has been proposed to train an agent to be able to handle unknown unknowns.
%\jt{However, the concept of awareness has not yet been introduced in reinforcement learning.
%Methods for modeling growing awareness in reinforcement learning differ fundamentally from those in Reverse Bayesianism.
%Reverse Bayesianism does not learn the data gathered from interaction with the environment, while reinforcement learning methods concerning growing awareness are supposed to train an agent with gathered data to make decisions.
%Furthermore, the state space and the action space in reinforcement learning are independent of each other, which helps describe the characteristics of the environment and the agent.}

% \jt{However, the concept of awareness has not yet been introduced in reinforcement learning.
% Growing awareness in reinforcement learning differs technically from that in Reverse Bayesianism due to the distinct definitions of action in these frameworks.
% In Reverse Bayesianism, an action is a function mapping states to states, whereas in reinforcement learning, the action space is a set independent of the state space.
% Consequently, as awareness of the state space expands, it does not imply a corresponding increase in awareness of the action space.
% }

%% file: 3.preliminaries.tex
\section{Notations and Preliminaries}
\label{sec::preliminaries}

\paragraph{Episodic Markov decision process (EMDP)}
An EMDP is defined by a tuple $M=(\mathcal{S},\mathcal{A},P,r,H,T)$ \cite{fiechter1994efficient, kaufmann2021adaptive}, described below.  %which models an agent’s interaction with its environment.
{$\mathcal{S}$ and $\mathcal{A}$ are finite 
state space and finite action space, respectively; mathematically, they %$\mathcal{S}$ and $\mathcal{A}$ 
are subsets of Euclidean spaces.}
If an agent at a state $s \in \mathcal{S}$ takes an action $a \in \mathcal{A}$, it will be transited to a new state $s' \in \mathcal{S}$. 
An EMDP agent takes {\it episodic} actions for learning -- after an episode, the agent comes back to the initial state $s_1$, and starts the next episode.
%We call a sequence of state–action pairs as an episode; 
Usually, episodes involve the same number of time steps, termed as {\it horizon} $H \in \mathbb{N}$. 
{$T \in \mathbb{N}$ is the number of episodes in a learning process.}
% Suppose there are $T \in \mathbb{N}$ episodes in a learning process. 
We can thus denote the episode $t$ by the trajectory $(s_1^t, a_1^t, \cdots, s_H^t, a_H^t, s_{H+1}^t)$.
For any episode, at time step $h \in [H] \triangleq \{1, \cdots, H\}$, a transition function $P_h: \mathcal{S} \times \mathcal{A} \times \mathcal{S} \rightarrow [0,1] $ defines the probability $P_h(s'|s,a)$ of transition from state $s$ to state $s'$ when taking action $a$. This interaction leads to an immediate reward $r_h: \mathcal{S} \times \mathcal{A} \rightarrow \mathbb{R}$. %assigns a reward $r_h(s,a)$ for taking action $a$ at state $s$.
%If the initial state $s_1$ is drawn from a distribution $\mu$, an artificial state $s_0$ can be introduced with a transition function defined as $p(s|s_0,a)=\mu(s)$. 
%Thus, we analyse the case of fixed initial state without loss of generality.
Let $\pi_h(a|s)$ denote the probability of taking action $a$ at state $s$. A policy model $\pi=(\pi_1, \cdots, \pi_H) \in \Pi$ characterises the agent's action course. All possible policy models constitute a  policy space $\Pi= \left \{ \pi=(\pi_1, \cdots, \pi_H)| \pi_h : \mathcal{S} \times \mathcal{A} \rightarrow [0,1] \right \}$. We denote the policy used to get episode $t$ by $\pi^t \in \Pi$. %, where $\pi_h(\cdot|s)$ {is a distribution over} $\mathcal{A} $.
%For a policy $\pi=(\pi_1, \cdots, \pi_H) \in \Pi$, $\pi_h$ is the policy at time step $h$ and $\pi_h(a|s)$ denotes the probability of taking action $a$ at state $s$.
%In the learning process, an agent interacts with the environment over $T$ episodes.
%Episode $t$ is the trajectory $(s_1^t, a_1^t, \cdots, s_H^t, a_H^t, s_{H+1}^t)$ and the policy used during this episode is denoted by $\pi^t \in \Pi$.
%After observing episode $t$, the policy is updated from $\pi^t$ to $\pi^{t+1}$, which is then employed in episode $t+1$.
%This process repeats until $T$ episodes have been completed.

\paragraph{Value functions and regret}
At time step $h \in [H]$, for any state $s \in \mathcal{S}$ and action $a \in \mathcal{A}$, the Q-value function of a policy $\pi$ is defined as
$
Q_h^{\pi}(s,a)=r_h(s,a)+\mathbb{E}_{p,\pi}\left [\sum_{l=h+1}^H r_l(s_l,a_l)|s_h=s,a_h=a \right ]
$,
%where the expectation is taken over the randomness induced by the environment and the policy.
and the value function is defined as
$
V_h^{\pi}(s)=\mathbb{E}_{p,\pi}\left [ \sum_{l=h}^H r_l(s_l,a_l)|s_h=s \right ]
$.
Their relationship %between $Q_h^{\pi}$ and $V_h^{\pi}$, known as 
is governed by the Bellman equation %, is given by 
$$
V_h^{\pi}(s)=\sum_{a \in A} \pi(a|s)Q_h^{\pi}(s,a),\ \ 
Q_h^{\pi}(s,a)=r_h(s,a)+\sum_{s'\in S} p(s'|s,a)V_h^{\pi}(s')
.$$
% \paragraph{Agent}
%\paragraph{Optimal policy and regret}
The objective of reinforcement learning is to learn a policy that maximises the cumulative rewards, denoted by $V_1^{\pi}(s)$. 
Shreve and Bertsekas prove that, under certain conditions, there exists an optimal deterministic policy $\pi$ %such that, for any $s \in \mathcal{S}, a\in \mathcal{A}, h\in [H]$, %the corresponding Q-value function and value function 
satisfy 
$
V_h^{*}(s) \triangleq V_h^{\pi^*}(s)=\mathop{\max}_{\pi} V_h^{\pi}(s),\ 
Q_h^*(s,a) \triangleq Q_h^{\pi^*}(s,a)=\mathop{\max}_{\pi} Q_h^{\pi}(s,a)
$ \cite{shreve1978alternative}.
This result leads to the Bellman optimality equation:
$$
V_h^{*}(s)=\mathop{\max}_{a \in A} Q_h^*(s,a),\ \ 
Q_h^*(s,a)=r_h(s,a)+\sum_{s'\in S} p(s'|s,a)V_h^{*}(s').
$$
The performance of an algorithm is often evaluated by its regret against the optimal policy 
%We define the regret as
$
\mathcal{R}^T=\sum_{t=1}^T \left [V^*_1(s_1)-V_1^{\pi^t}(s_1) \right ]
$.
Note that $V_1^{\pi^t}(s_1) \le V^*_1(s_1)$ for any $t \in [T]$.

% \paragraph{Learning process}
% In the learning process, an agent interacts with the environment over $T$ episodes
% Episode $t$ is the trajectory $(s_1^t, a_1^t, \cdots, s_H^t, a_H^t, s_{H+1}^t)$ and the policy used during this episode is denoted by $\pi^t \in \Pi$.
% After observing episode $t$, the policy is updated from $\pi^t$ to $\pi^{t+1}$, which is then employed in episode $t+1$.
% This process repeats until $T$ episodes have been completed.

% \begin{figure}[h]
% \centering
% \includegraphics[scale=0.47]{figures/UCBMQ update.pdf}
% \caption{Updating value functions in UCBMQ}
% \label{UCBMQ update}
% \end{figure}

\paragraph{Upper confidence bound momentum Q-learning (UCBMQ)}

%Our approach builds upon the algorithm UCBMQ proposed by M\'{e}nard et al. .
UCBMQ is designed for training an EMDP model \cite{menard2021ucb}. %, on the basis of Q-learning
%A key technical ingredient of UCBMQ is the momentum term, which is employed to correct bias.
Its objective is to approximate the optimal policy by controlling the regret. It optimises four functions: {estimated}  Q-value function $Q_h^t(s,a)$,  bias-value function $V_{h,s,a}^{t}(s')$,  upper bound on the optimal Q-value function $\overline{Q}_h^t(s,a)$, and  upper bound on the optimal value function $\overline{V}_h^t(s)$.
In iteration $t$, the agent interacts with the environment to collect episode $t$, updating the visitation counts $n_h^t(s,a)$ for each state-action pair $(s,a)$.
Next, we update the count‐dependent parameters, including the bonus term $\beta_h^t(s,a)$.
Finally, we update the four value functions.
The core method of updating $Q_h^t(s,a)$ retains the standard Q-learning update, which is inspired by the Bellman optimality equation, and adds a momentum correction to remove the bias.
$V_{h,s,a}^{t}(s')$ is updated as a convex combination of the previous bias‐value and the optimistic next‐step value, $\overline{Q}_h^t(s,a)$ is updated by adding $\beta_h^t(s,a)$ to $Q_h^t(s,a)$ and $\overline{V}_h^t(s)$ is updated according to Bellman optimality equation.
The procedure of UCBMQ is described in \ref{details_of_UCBMQ}.
UCBMQ can guarantee a regret bound of $\tilde{\mathcal{O}}(\sqrt{ H^3SAT}+ H^4SA)$, where $H$, $S$, $A$ and $T$ are the horizon, the size of the state space, the size of the action space and the number of learning episodes, respectively.
This regret bound matches the lower bound $\Omega(\sqrt{SAH^3T})$ \cite{domingues2021episodic} for large enough $T$.

%% file: 4.awareness_of_agent_in_an_uncertain_environment.tex
\section{Growing Awareness in Extremely Under-explored, Uncertain Environment}
\label{sec::awareness_of_agent}

The awareness of an agent can be considerably restricted and problematic in an extremely under-explored and uncertain environment. %the extent of an agent's knowledge about its environment, which 
We are aiming to grow the awareness along the agent's course of exploring the environment. %This section mathematically grounds the awareness, starting with the definitions of the aware domain and unaware domain.

\subsection{Mathematically Grounding Awareness in Reinforcement Learning}

\paragraph{An example of a spacecraft}

Consider a spacecraft controlled by a reinforcement learning agent. She has explored her homeland (or home planet), which is not necessarily fully completed. 
%Here, the earth can be viewed as the homeland of the spacecraft.
%It is now tasked with cleaning the entire house.
The agent's value functions are naturally defined on her homeland, which is her current aware domain.
% Some day, value functions tell the agent she needs to go beyond her homeland and explore the space.
{After taking an action calculated from the value functions defined on the aware domain}, the spacecraft may go out of her home planet. %, the {\it aware domain}. 
The spacecraft then starts to explore the previously unaware domain. %, using its sensors to gather information.
The exploration requires an initialisation for the value functions to kick off. This calls for an expansion strategy for the value functions to newly discovered states. %-- a noninformative guess is the most `reasonable' estimator given all accessible information, and has been sufficient. 
The value functions are optimised afterwards. %using further gathered new information.
After sufficient exploration, a good spacecraft is expected to develop an (near-)optimal policy -- this is exactly the goal of this paper.
Figure \ref{explore space} illustrates this example. %, where %the earth is the home planet, the red line indicates the spacecraft's trajectory, the grey area represents the unaware domain and the light area represents the aware domain.}

% \begin{wrapfigure}{r}{0.45\textwidth}
%     \centering
%     \includegraphics[width=0.43\textwidth]{figures/awareness set.pdf}
% \end{wrapfigure}

Mathematically, we translate awareness to the domain of value functions, as the definitions below.

\begin{definition}[aware domain]
The value functions are defined over the awareness domain.
In this domain, the agent is aware of each state’s existence but needs to fully explore it to ascertain its value.
\end{definition}

\begin{definition}[unaware domain]
The unaware domain comprises the states whose existence and values are both unknown to the agent.
\end{definition}

\begin{remark}
% The existence and values of UU states are unavailable to the agent.
% KU states, by contrast, are recognized but their values remain uncertain until fully explored.
In our framework, once a state in the unaware domain is explored, it moves to the aware domain and never reverts to the unaware domain.
States in the awareness domain remain there permanently.
Hence, the agent’s awareness monotonically increases without forgetting previously discovered states.
\end{remark}

\subsection{Episodic Markov Decision Processes with Growing Awareness}
\label{sec:EMDP-GA}

We now define our {\it episodic Markov decision processes with growing awareness} (EMDP-GA) model as a tuple $M_{GA}=(\mathcal{S},\mathcal{S}_0,\mathcal{A},P,r,H,T)$, where $\mathcal{S}_0 \subseteq \mathcal{S}$ is the initial aware domain.
% , while
% \fh{$\mathcal{S}$, $\mathcal{A}$, $P$, $r$, %defines the infinite state space and 
% $H$, and $T$ are the infinite state space, horizon, and the number of learning episodes, respectively, the same as EMDP.}
% \jt{?}
Note that the initial state of each learning episode is fixed as $s_1 \in \mathcal{S}_0$.
%Now we are ready to model awareness in reinforcement learning.
%Our model is named  
% Suppose the agent is modelled as an EMDP, %with growing awareness (EMDP-GA).
% %defined as a tuple 
% $M =(\mathcal{S},\mathcal{A},P,r,H,T)$. %where $\mathcal{S}$, $\mathcal{A}$, $P$, $r$, $H$, and $T$ are the infinite state space, infinite action space, transition function, immediate reward function, horizon, and the number of learning episodes, respectively.
% %$\mathcal{A}$ is the finite action space.
% %The transition function $P_h(s'|s,a)$ defines the probability of transitioning from state $s$ to $s'$ when taking action $a$ at time step $h$.
% %The immediate reward function $r_h: \mathcal{S} \times \mathcal{A} \rightarrow \mathbb{R}$ defines the immediate reward $r_h(s,a)$ for taking action $a$ in state $s$ at time step $h$. 
% We denote by $\mathcal{S}_0 \subseteq \mathcal{S}$ her initial aware domain. %about the state space.
% Note that the initial state is naturally located in this initial aware domain; i.e., $s_1 \in \mathcal{S}_0$.
If the agent reaches an unknown unknown state $s$ after taking an action, she will be aware of the existence of $s$, so that the aware domain grows by the state $s$.
Let $\mathcal{S}_t$ denote the aware domain after exploring $t$ episodes.
At episode $t+1$, the agent's trajectory is $ \left ( s^{t+1}_1,\cdots,s^{t+1}_H \right )$, and the aware domain is updated accordingly as $\mathcal{S}_{t+1}=\mathcal{S}_{t} \cup \left\{ s_{1}^{t+1},  \cdots,  s_H^{t+1} \right\}$.
Note that
$
s_1 \in \mathcal{S}_0 \subseteq \cdots \subseteq \mathcal{S}_T \subseteq \mathcal{S}
$.

\begin{figure}[h]
\centering
\includegraphics[scale=0.175]{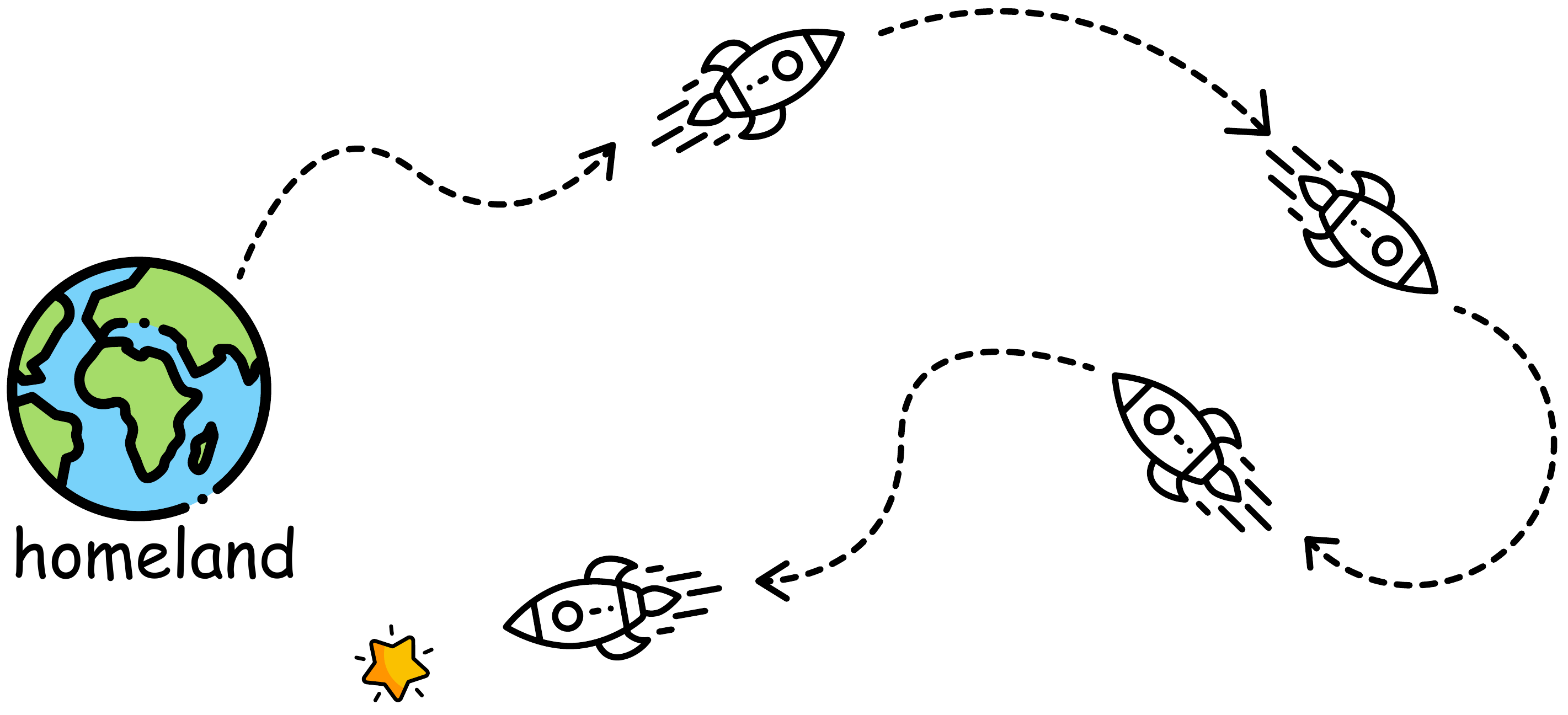}
\caption{Illustration of a spacecraft}
\label{explore space}
\end{figure}

{%The agent acknowledges the existence of unknown unknown states.
After $t$ episodes, we denote by $U_t$ the unknown unknown states.
%This means when the agent encounters any state outside the aware domain $S_t$, it will make decision according to the policy defined on $U_t$.
The extended aware domain is thus $\tilde{\mathcal{S}}_t=\mathcal{S}_t \cup U_t$. %, where $U_t$ generally represents any unknown unknown state.
}
% We denote by $U_t$ the set of unknown unknown states after exploring $t$ episodes; i.e., $U_t=\mathcal{S}-\mathcal{S}_t$.
% \jt{The agent recognises that there are some unknown unknown states.
% The extended aware domain is defined as $\tilde{\mathcal{S}}_t=\mathcal{S}_t \cup U_t$, where $U_t$ generally represents any unknown unknown state.}
% The growing awareness is then defined as \fh{$\tilde{\mathcal{S}}_t=\mathcal{S}_t \cup U_t$}.
%Note that while $U_t$ represents the states currently unknown to the agent, it also serves as a placeholder for any new state that may be discovered.
% \fh{In our model, the agent selects actions from the action space $\mathcal{A}$ even when it is placed at new states.}
{EMDP usually defines state space and action space independently. Inheriting this convention, in this paper, the agent selects actions from a fixed action space $\mathcal{A}$, despite the fact that the aware domain is growing, %like in an EMDP, 
even when the agent is placed at new states.}
{Note that the policy $\pi^t$, which is used to get episode $t$, is defined on the aware domain after exploring $t-1$ episodes.
Thus, the policy $\pi^t=(\pi_1^t,\cdots,\pi_H^t)$ is defined on $\tilde{\mathcal{S}}_{t-1}$; i.e., $\pi_h^t: \tilde{\mathcal{S}}_{t-1} \times \mathcal{A} \rightarrow[0,1]$ for each $h$.
Here, $\pi_h^t(\cdot|s)$ is a distribution over $\mathcal{A}$.}
Consequently, $\pi^t$ effectively constitutes a policy over the entire state space $\mathcal{S}$.

{
% We mathematically formulate {\it homeland} as the following definition.

% \begin{definition}[homeland]
% \label{optimal_policy_assumption}
% We assume the agent has a homeland which is fully explored.
% For the optimal policy, we assume that, with probability at least $1-\frac{\delta}{3}$, the average value of the known states exceeds that of any UU state:
% $$\forall s \in \mathcal{S}-\mathcal{S}_t, Q^*_h(s,a) \le \frac{1}{|\mathcal{S}_{t}|}\sum_{\hat{s} \in \mathcal{S}_{t}} Q^*_h(\hat{s},a),\ V^*_h(s) \le \frac{1}{|\mathcal{S}_{t}|}\sum_{\hat{s} \in \mathcal{S}_{t}} V^*_h(\hat{s}).$$
% \end{definition}

%\begin{remark}
\paragraph{Homeland condition}
% \label{homeland_condition}
%To note, we do not aim to explore a completely 
Exploring in an extremely unexplored, uncertain world requires a `homeland'. 
Space exploration needs an earth.
%The vision is an agent needs a safe base for an extremely high-risk mission.
We assume that with high probability, unknown unknowns are no more valuable than the average on known states; a very high-value target still exists but with a high risk. Mathematically, with probability at least $1-\frac{2\delta}{3}$, the average value of the known states exceeds that of any unknown unknown state:
$$\forall s \in \mathcal{S}-\mathcal{S}_t, Q^*_h(s,a) \le \frac{1}{|\mathcal{S}_{t}|}\sum_{\hat{s} \in \mathcal{S}_{t}} Q^*_h(\hat{s},a),\ V^*_h(s) \le \frac{1}{|\mathcal{S}_{t}|}\sum_{\hat{s} \in \mathcal{S}_{t}} V^*_h(\hat{s}).$$
It coincides with the fact that exploring in an extremely under-explored, uncertain environment is a high-risk, high-gain mission.
%For example, exploring an unknown planet is much riskier than remaining on the earth but may also yield high rewards. Consequently, we assume that for the optimal policy, 

% We mathematically formulate this assumption.
% We mathematically formulate this assumption as the following assumption.

% Definition \ref{optimal_policy_assumption} suggests that, with high probability, unknown unknowns are no more valuable than the average of known states; a very high-value target still exists but with a high risk.
% It suggests that exploring in an extremely under-explored, uncertain environment is a high-risk, high-gain mission.

% \begin{assumption}
% \label{optimal_policy_assumption}
% We assume the agent has a homeland which is fully explored.
% For the optimal policy, we assume that, with probability at least $1-\frac{\delta}{3}$, the average value of the known states exceeds that of any UU state:
% $$\forall s \in \mathcal{S}-\mathcal{S}_t, Q^*_h(s,a) \le \frac{1}{|\mathcal{S}_{t}|}\sum_{\hat{s} \in \mathcal{S}_{t}} Q^*_h(\hat{s},a),\ V^*_h(s) \le \frac{1}{|\mathcal{S}_{t}|}\sum_{\hat{s} \in \mathcal{S}_{t}} V^*_h(\hat{s}).$$
% \end{assumption}

}

%$\mathcal{A}$ is the finite action space.
%The transition function $P_h(s'|s,a)$ defines the probability of transitioning from state $s$ to $s'$ when taking action $a$ at time step $h$.
%The immediate reward function $r_h: \mathcal{S} \times \mathcal{A} \rightarrow \mathbb{R}$ defines the immediate reward $r_h(s,a)$ for taking action $a$ in state $s$ at time step $h$. 
%Notably, the sole difference between EMDP and EMDP-GA is the inclusion of an additional component—the initial aware domain $S_0$.
% Figure \ref{EMDP} illustrates the interaction between the agent and EMDP.

% \fh{xxx}

% \subsection{Oracle Agent}

%% file: 5.UCBMQ-GA_algorithm.tex
%\section{Learning an Agent with Growing Awareness}
%\label{sec::UCBMQ-GA_algorithm}
% We introduce UCBMQ and then present our algorithm UCBMQ-GA.
%In this section, we introduce our framework for EMDP-GA.
%We first describe our domain-expansion method and then present our framework, Upper Confidence Bound Momentum Q-learning with Growing Awareness (UCBMQ-GA).

% \begin{figure}[t]
% \centering
% \includegraphics[scale=0.45]{figures/EMDP.pdf}
% \caption{{Episodic Markov decision processes with growing awareness by noninformative value expansion}.}
% \label{EMDP}
% \end{figure}

\subsection{Noninformative Prior for Domain Expansion to Newly Aware States}
% The value functions are defined on the aware domain since the agent has no information about the space outside the aware domain, even the existence.
% When the agent reaches a newly discovered state that is not in the domain of value functions, we need to expand the domain to this state.
% The newly discovered state and it may be very different from any known state and we have no prior information about it.
% Consider the case that taking an action $a$ can always lead to a high reward, whatever the state is, we naturally tend to think $a$ is a valuable action and can lead to a high reward in the newly discovered state with high probability.
% So we use the average values over the aware domain to estimate the value functions on newly discovered state.

%As discussed, the domains of value functions $Q$ and $V$ are restricted by the awareness to the environment -- the value functions are defined on the agent’s currently aware domain. %, since the agent has no information, even of existence, about states outside this space.

When an agent encounters an unknown unknown state that lies outside the currently aware domain, the reinforcement learning algorithm becomes invalid. %we must extend the domain to include this state.
The unknown unknown status implies that it exhibits properties that significantly differ from any known state, %(including known unknown states) 
and prior information is completely absent.
Out of respect for this extreme lack of knowledge,
%For example, if an action $a$ consistently yields high rewards across known states, it is reasonable to infer that $a$ will also likely yield high rewards in the new state.
%Therefore, we estimate 
we design a {\it noninformative value expansion} (NIVE) approach to expand value functions to the newly discovered states -- the value functions are initialised using the averages over the current aware domain.
We introduce NIVE with more details below.

\paragraph{Algorithm of NIVE}
NIVE updates three value functions: 
%\begin{itemize}[leftmargin=12pt]
%    \item 
%\end{itemize}
estimated Q-value function $Q_h^t(s,a)$, bias-value function $V_{h,s,a}^{t}(s')$, and upper bound on the optimal value function $\overline{V}_h^t(s)$.
When the aware domain grows from $\mathcal{S}_{t-1}$ to $\mathcal{S}_t$, we would like to expand the domains of three value functions: $Q^{t-1}_h(s,a)$, $\overline{V}^{t-1}_h(s)$ and $V^{t-1}_{h,s,a}(s')$, accordingly.
%For every state in $\mathcal{S}_t-\mathcal{S}_{t-1}$, no direct information is available; however, prior estimates can be inferred from known data.
%Consequently, the framework 
NIVE calculates the average values over 
% \fh{xxx}
{the aware domain} $\mathcal{S}_{t-1}$ as  estimators for states in $\mathcal{S}_t-\mathcal{S}_{t-1}$:
%In UCBMQ-GA, average values for $Q_h^t(s,a)$, $\overline{V}_h^t(s)$ and $V_{h,s,a}^t(s')$ over $\mathcal{S}_t$ are defined as: for $t \in [T]$, 
\begin{gather*}
    Q_{h,avg}^t(a)=\frac{1}{|\mathcal{S}_t|} \sum_{s \in \mathcal{S}_t} Q_h^t(s,a),\ \ \ 
\overline{V}_{h,avg}^t=\frac{1}{|\mathcal{S}_t|} \sum_{s \in \mathcal{S}_t} \overline{V}_h^t(s),\\ 
V_{h,s,a,avg}^t=\frac{1}{|\mathcal{S}_t|} \sum_{s' \in \mathcal{S}_t} V_{h,s,a}^t(s'), \ \ \ 
V_{h,a,avg}^t(s')=\frac{1}{|\mathcal{S}_t|} \sum_{s \in \mathcal{S}_t} V_{h,s,a}^t(s'),\\ \text{and}\ \ \ 
V_{h,a,avg}^t=\frac{1}{|\mathcal{S}_t|^2} \sum_{s,s' \in \mathcal{S}_t} V_{h,s,a}^t(s').
\end{gather*}

%When using the UCBMQ-GA to train an EMDP-GA agent, %the agent selects 
%actions are determined according to 
%We define 
% \fh{xxx}
%{an upper-bound function} $\overline{Q}_h^t(s,a)$ for %, which upper-bounds 
%the optimal Q-value function $Q^*_h(s,a)$.
% \fh{The functions $\overline{Q}_h^t(s,a)$ and $\overline{V}_h^t(s)$ thus serve as estimators to optimal value functions $Q^*_h(s,a)$ and $\overline{V}_h^t(s)$.}
%In our approach, and also 
In many reinforcement learning algorithms, such as UCBMQ \cite{menard2021ucb} used later, upper bounds $\overline{Q}_h^t(s,a)$ and $\overline{V}_h^t(s)$ are used to calculate actions. 
% We thus view them as estimators of the optimal value functions $Q^*_h(s,a)$ and ${V}^*(s)$.}
%Our domain‐expansion technique, noninformative value expansion (
After the surprising encounter, the NIVE expands 
% \fh{value functions (and their estimators)} 
{the estimators of value functions} to newly discovered states by assigning them the average values computed over known states.
Crucially, the expanded functions $\tilde{\overline{Q}}_h^t(s,a)$ and $\tilde{\overline{V}}_h^t(s)$ remain valid upper bounds for $Q^*_h(s,a)$ and $V_h^*(s)$, respectively, as the following lemma.  The proof is given in   \ref{optimism} (implied in Lemma \ref{optimism_lemma}). 
%\fh{We fix $\delta \in (0,1)$.}

\begin{lemma}
{For any time step $t \in [T]$ in episode $h \in [H]$,} the following inequality holds with probability at least $1-\frac{\delta}{2}$,
$$
\tilde{\overline{Q}}_h^t(s,a) \ge Q^*_h(s,a),\ \tilde{\overline{V}}_h^t(s) \ge V_h^*(s).
$$
\end{lemma}

%\begin{remark}
    % This lemma justifies that our NIVE preserves the crucial upper-bound property of the estimators, required for UCBMQ and its variants. 
%\end{remark}

\begin{figure}[t]
\centering
\includegraphics[scale=0.5]{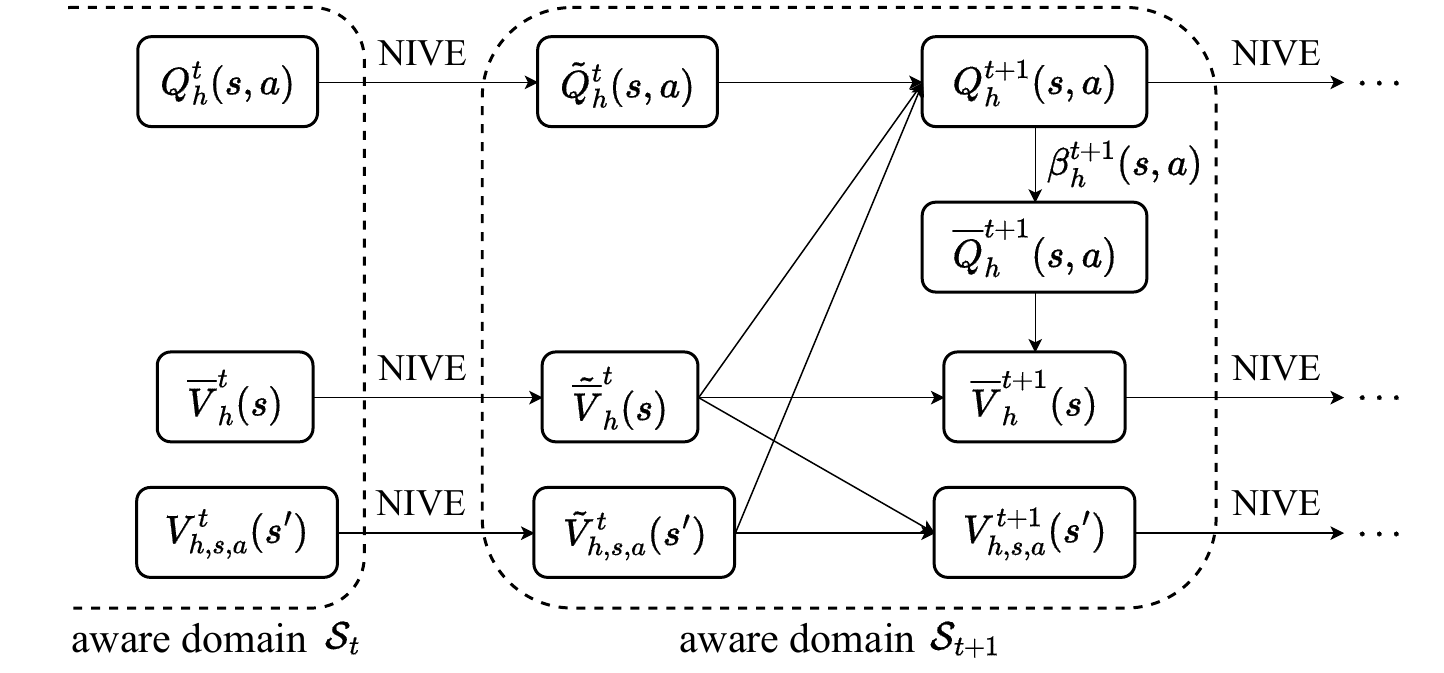}
\caption{Illustration of updating functions in UCBMQ-GA}
\label{UCBMA-GA update}
\end{figure}

Further, we may prove the following corollary. The proof is given in  \ref{proof_of_scalar}.
% \fh{The proof is given in xxx.}\jt{need to add}

\begin{corollary}
\label{corollary:scalar}
If one were to use any scalar multiple $d < 1$ of the average values to expand the functions, the upper‐bound property could become invalid.
In some cases, assigning $d \overline{V}_{h,avg}^t,\ (d \neq 1)$ to a newly discovered state would produce an expanded function that no longer upper bounds $V_h^*(s)$.
\end{corollary}

This corollary confirms the efficiency of NIVE -- multiplying any constant $d < 1$ to NIVE will make it invalid.

\subsection{Learning an Agent with Growing Awareness to Unknown Unknowns}

We now design an optimisation method for training our EMDP-GA model for obtaining an agent with growing awareness to handle the unknown unknowns.
We adapt upper confidence bound momentum Q-learning (UCBMQ) \cite{menard2021ucb} to the growing awareness, %for training the EMDP-GA model. This training algorithm is 
termed as UCBMQ-GA.
%In its current form, UCBMQ does not address the problem of growing awareness.
%However, UCBMQ-GA is able to deal with EMDP-GA and is more broadly applicable.

\paragraph{Algorithm of UCBMQ-GA}
{%Like UCBMQ, 
UCBMQ-GA {optimises} four functions: {estimated}  Q-value function $Q_h^t(s,a)$,  bias-value function $V_{h,s,a}^{t}(s')$,  upper bound on the optimal Q-value function $\overline{Q}_h^t(s,a)$, and  upper bound on the optimal value function $\overline{V}_h^t(s)$.}

Let $\chi_h^t(s,a)=\mathbb{I}(s_h^t=s, a_h^t=a)$ denote the indicator function for the occurrence of $(s,a)$ at time step $h$ in episode $t$. Thus, $n_h^t(s,a)=\sum_{l=1}^t \chi_h^l(s,a)$ as the visitation count for the pair $(s,a)$ {appearing at time step $h$ over $t$ episodes}.
Define $\chi_h^t(s,a)=n_h^t(s,a)=0$ if state $s \notin \mathcal{S}_t$. Following the convention, and to the brevity, we define $0 \times \infty =0$ and $\frac{1}{0}=\infty$.

For any transition function $p$ and any function $f: S \rightarrow \mathbb{R}$, we define 
$
pf(s,a)=\mathbb{E}_{s'\sim p(\cdot|s,a)}[f(s')].
$
We denote by $p_h^t$ the Dirac distribution  concentrated at $(s_{h+1}^t)$, and then, $p_h^tf(s,a)=f(s_{h+1}^t)$.
The procedure of UCBMQ-GA for training EMDA-GA is also explained below. More details are shown in  \ref{UCBMQ-GA_appendix}.

\paragraph{(0) Initialisation} %and data collection}
%Line 1-10 describe this step.
%The initialisation of UCBMQ-GA is identical to that of UCBMQ: f
For state $s \in \mathcal{S}_0 \subseteq \mathcal{S}$, we set $V^0_{h,s,a}=\overline{V}^0_h=H$, and $Q^0_h=0$. Then, go to (1).

\paragraph{(1) Decision making}
In episode $t$, if the current state is aware (i.e., $s \in \mathcal{S}_{t-1}$), the agent takes an action that maximises {$\overline{Q}_{h}^{t-1}(s,a)$.
Conversely, if the agent encounters an unknown unknown state (i.e., $s \notin \mathcal{S}_{t-1}$), it chooses an action that maximises ${Q}_{h,avg}^{t-1}(a)$.}

{The algorithm terminates here if the termination conditions have been met. Otherwise, if the agent go beyond the aware domain, go to (2); otherwise, go to (3). }

\paragraph{(2) Embedded NIVE for growing awareness} %and expanding upper bound}
%Line 11-12 describe this step.
After completing an episode, the awareness set is updated from $\mathcal{S}_{t-1}$ to $\mathcal{S}_t$.
%The key technique that makes our framwork able to deal with EMDP-GA is that our framework 
UCBMQ-GA then employs NIVE to expand $\tilde{Q}^{t-1}_h(s,a)$, $\tilde{\overline{V}}^{t-1}_h(s)$ and  $\tilde{V}^{t-1}_{h,s,a}(s')$ to the newly discovered area, 
%This is the primary difference between UCBMQ and UCBMQ-GA.
%The expanded functions employed in our framework are defined 
as follows,

\begin{align*}
\tilde{Q}^{t-1}_h(s,a)=
    \begin{cases}
        Q^{t-1}_h(s,a), & s \in \mathcal{S}_{t-1},\\
        Q^{t-1}_{h,avg}(a), & s \in \mathcal{S}_t-\mathcal{S}_{t-1},
    \end{cases} 
\end{align*}
\begin{align*}
\tilde{\overline{V}}^{t-1}_h(s)=
    \begin{cases}
        \overline{V}^{t-1}_h(s), & s\in \mathcal{S}_{t-1},\\
        \overline{V}_{h,avg}^{t-1}, & s \in \mathcal{S}_t-\mathcal{S}_{t-1},
    \end{cases}
% \label{expanded_2}
\end{align*}
\begin{align*}
\text{and  }
\tilde{V}^{t-1}_{h,s,a}(s')=
    \begin{cases}
        V^{t-1}_{h,s,a}(s'), & s,s' \in \mathcal{S}_{t-1},\\
        V_{h,s,a,avg}^{t-1}, & s\in \mathcal{S}_{t-1}, s'\in \mathcal{S}_t-\mathcal{S}_{t-1},\\
        V_{h,a,avg}^{t-1}(s'), & s\in \mathcal{S}_t-\mathcal{S}_{t-1}, s'\in \mathcal{S}_{t-1},\\
        V_{h,a,avg}^{t-1}, & s,s'\in \mathcal{S}_t-\mathcal{S}_{t-1}.
    \end{cases}
% \label{expanded_3}
\end{align*}

The exploration bonus is then calculated with the expanded functions:
%$$
%\text{if } n_h^t(s,a)=0 \text{ or } s \notin \mathcal{S}_t\text{, we have } \beta_h^t(s,a)=H;
%$$
%otherwise
if $n_h^t(s,a)=0$ or $s \notin \mathcal{S}_t$, $\beta_h^t(s,a)=H$; otherwise,
\begin{align*}
\beta_h^t(s,a)=&2 \sqrt{\frac{\zeta W_h^t(s,a)}{n_h^t(s,a)}}+53H^3\frac{\zeta log(T)}{n_h^t(s,a)}\\
&+\sum_{k=1}^t\frac{\chi_h^k(s,a)\mathring{\gamma}_h^k(s,a)}{Hlog(T)n_h^t(s,a)}p_h^k \left (\tilde{V}_{h,s,a}^{k-1}-\tilde{\overline{V}}_{h+1}^{k-1} \right )(s,a),
\end{align*}
where $W_h^t(s,a)=\sum_{k=1}^t \frac{\chi_h^k(s,a)}{n_h^k(s,a)}p_h^k \left (\tilde{\overline{V}}_{h+1}^{k-1}-\sum_{l=1}^t \frac{\chi_h^l(s,a)}{n_h^l(s,a)}p_h^l\tilde{\overline{V}}_{h+1}^{l-1} \right )^2(s,a)$.

Then, go to (3).

{\paragraph{(3) Optimising value functions}
% describe this step.
The value functions $Q_h^t(s,a)$, $V_{h,s,a}^{t}(s')$, $\overline{Q}_h^t(s,a)$ and $\overline{V}_h^t(s)$ are then optimised as illustrated in Figure \ref{UCBMA-GA update}. %updated in a manner similar to UCBMQ.
The learning rate and momentum term are chosen as 
\begin{gather*}
\alpha_h^t(s,a)=\frac{\chi_h^t(s,a)}{n_h^t(s,a)},\ 
\gamma_h^t(s,a)=\chi_h^t(s,a) \frac{H}{H+n_h^t(s,a)} \frac{n_h^t(s,a)-)}{n_h^t(s,a)}, \text{ and} \\
\eta_h^t(s,a)=\alpha_h^t(s,a)+\gamma_h^t(s,a).  
\end{gather*}

$\zeta$ is the exploration threshold which helps in regulating the exploration bonus.}

Then, go to (1).

%\begin{remark}[

%% file: 6.theory.tex
\section{Theoretical Analysis}
\label{sec::theoretical_analysis}

This section %first defines the optimal policy in EMDP-GA and then 
presents theoretical analysis of our algorithm. %, which guarantees the performance.

\subsection{Asymptotical Consistency of Learning Unknown Unknowns}%in Exposure of Unawareness}

\paragraph{Oracle agent}

The theoretical analysis aims to compare our {approach} against the (unknown) optimal policy $
\pi^*=(\pi_1^*,\cdots,\pi_H^*)$, 
$\pi_h^*: \mathcal{S}\times \mathcal{A} \rightarrow [0,1]$. This is characterised by the regret of {UCBMQ-GA} against this optimal policy, defined as %$\pi^*$ is defined as
$
\mathcal{R}^T=\sum_{t=1}^T [V^{\pi^*}(s_1)-V^{\pi^t}(s_1)]
$. 
To characterise this optimal policy, we assume the existence of an oracle agent that is aware of the entire state space $\mathcal{S}$ and can maximise the value function accordingly.
%This assumption is motivated by the fact that, in many applications, the objective is to approximate the optimal policy in the real world despite the absence of complete information prior to exploration.
%Mathemarically, the optimal policy is defined as
%$
%\pi^*=(\pi_1^*,\cdots,\pi_H^*),\ \ \ 
%\pi_h^*: \mathcal{S}\times \mathcal{A} \rightarrow [0,1],
%$

% We make the following assumption about the transition function.
% \begin{assumption}[non-zero probability of transition]
% \label{transition_function_assumption}
% For the transition function, We assume that $1-\left (\frac{\delta}{3} \right)^{\frac{1}{H[\sqrt{T}]}} \le \frac{1}{S}$, and the following holds
% $$\forall (s,a,s') \in \mathcal{S} \times \mathcal{A} \times \mathcal{S}, P(s'|s,a) \ge 1-\left (\frac{\delta}{3} \right)^{\frac{1}{H[\sqrt{T}]}}.$$
% \end{assumption}

% \begin{remark}
% When $T$ is large enough, $1-\left (\frac{\delta}{3} \right)^{\frac{1}{H[\sqrt{T}]}} \le \frac{1}{S}$ holds and guarantees that $P(\cdot|s,a)$ can be a distribution over $\mathcal{S}$.
% Assumption \ref{transition_function_assumption} means that the agent can probably reach any state from the current state when taking any action.
% \end{remark}

We then define an {\it aware moment} of a state as follows.

\begin{definition}[aware moment]
    For any state $s \in \mathcal{S}$, define its aware moment as 
    $
    t(s)=\mathop{min}\limits_{t \in [T]}\{t:s \in \mathcal{S}_t\}
    $; i.e., the first episode after which the agent is aware of $s$.
\end{definition}

% We can have a basic result establishing an upper bound on $t(s)$: $\forall s \in S, t(s)\le \sqrt{T}$ with probability at least $1-\frac{\delta}{3}$.
% The proof of this result is detailed in Appendix \ref{proof_of_first_existence_lemma}.
% This result implies that, with high probability, every state can be aware of by the agent within $\sqrt{T}$ episodes.
% This result is instrumental in deriving regret upper bounds, as the knowledge about the environment have influence on the regret.
% Specifically, increased environmental knowledge leads to lower regret and limited knowledge results in higher regret.
%We establish the following basic result, which provides 
Naturally, $t(s) \ge 0$. We then have an upper bound on the aware moment. The proof is detailed in  \ref{proof_of_first_existence_lemma}.

% \begin{theorem}
% \label{first_existence_lemma}
%     With probability at least $1-\frac{\delta}{3}$, the aware moment is upper-bounded: %$\forall s \in \mathcal{S}, 
%     $t(s)\le \sqrt{T}$, where $T$ is the number of learning episodes.
% \end{theorem}

% \begin{remark}
%     This theorem implies that, with high probability, the agent becomes aware of all states within $\sqrt{T}$ episodes.
% \end{remark}

This finding is instrumental in deriving regret upper bounds. %, since intuitively the knowledge of environment significantly influences regret. 
%Specifically, more environmental knowledge leads to lower regret, whereas limited knowledge leads to higher regret.
% \fh{We assume in this section that $T >3$ and $\zeta=log \left (\frac{96eHSA(2T+1)}{\delta} \right)$.}
Then, we obtain the regret bound as follows. A detailed proof is given in \ref{proof_of_regret_bound}. %are ready to present the main result of our paper, which is proved 

\begin{theorem}
\label{theorem_1}
%For UCNMQ-GA, w
We assume that $T >3$ and $\zeta=log \left (\frac{96eHSA(2T+1)}{\delta} \right)$.
When the UCNMQ-GA is used to train EMDP-GA, with probability at least $1-\delta$, the regret satisfies
$$
\mathcal{R}^T \le \tilde{\mathcal{O}}(\sqrt{ H^3SAT}+ H^4SA+H^2S\sqrt{T}),
$$
where $H$ is the horizon, $S$ is the size of the state space, $A$ is the size of the action space and $T$ is the number of learning episodes.
\end{theorem}

{
\begin{remark}
This regret bound is sublinear with respect to the number of episodes $T$. 
It matches the SOTA, despite that they are not exposed to unknown unknowns, but we are.
\end{remark}

\begin{remark}
    We have now proved that unknown unknowns could be surprising, but can be asymptotically consistently discovered.
\end{remark}

% \fh{\paragraph{Comparison with existing methods without the awareness problem.} This theorem suggests that the regret is comparable with the state of the art. xxx}
\paragraph{Comparison with existing methods, but without exposure to unknown unknowns.} 
%{This theorem suggests that the regret is comparable with the state of the art.
Compared with SOTA, for $H \le \frac{A}{S}$, the regret bound of our UCBMQ-GA matches UCBMQ ($ \mathcal{O}(\sqrt{SAH^3T}+SAH^4)$) \cite{menard2021ucb}.
For $T \ge H^5SA$ and $H \le \frac{A}{S}$, the regret bound of UCBMQ-GA matches monotonic value propagation (MVP, $\mathcal{O}(\sqrt{H^3SAT})$) \cite{zhang2024settling}. Other methods, upper Bayes-confidence bound value iteration (Bayes-UCBVI) \cite{tiapkin2022dirichlet} %using $B$ Monte-Carlo samples to approximate one quantile, 
and upper confidence bound-advantage(UCB-Adv) \cite{zhang2020almost}, are worse than the three methods in either case.
Table \ref{comparisons} shows the comparisons.
% Numerous works have introduced novel techniques to improve the regret bound.
% For example, a recent work by \cite{zhang2024settling} proposed Monotonic Value Propagation (MVP), achieving a regret bound that matches the minimax lower bound established by \cite{domingues2021episodic}.
% Unfortunately, none of these works have considered the issue of growing awareness—a natural challenge in real-world applications.
}

\begin{table}[t]
  \caption{Comparisons with the State of the Art}
  \label{comparisons}
  \centering
  \small
  \begin{tabular}{cccc}
    \toprule
    Algorithm     & Regret Bound     & Comput. Complex. & {Spac. Complex.} \\
    \midrule
    
    \thead{UCBMQ \\ \cite{menard2021ucb}}& \thead{$ \tilde{\mathcal{O}}(\sqrt{H^3SAT}+H^4SA)$}  & $\mathcal{O}(H(S+A)T)$& $\mathcal{O}(HS^2A)$\\

    \thead{MVP \\ \cite{zhang2024settling}} & $\tilde{\mathcal{O}}(\sqrt{H^3SAT})$ & $\mathcal{O}(HSAT)$ &$\mathcal{O}(HS^2A)$\\

    \thead{UCB-Adv\\ \cite{zhang2020almost}}& \thead{$ \tilde{\mathcal{O}}(\sqrt{H^3SAT}$\\$+H^{\frac{33}{4}}S^2A^{\frac{3}{2}}T^{\frac{1}{4}})$}  & $\mathcal{O}(HAT)$&$\mathcal{O}(HSA)$\\

    \thead{Bayes-UCBVI \\ \cite{tiapkin2022dirichlet}} & \thead{$\tilde{\mathcal{O}}(\sqrt{H^3SAT}+H^3S^2A)$} & $\mathcal{O}(BHS^2AT)$ &$\mathcal{O}(HS^2A)$    \\

    \thead{\textbf{UCBMQ-GA}\\ \textbf{(ours)}} & \thead{$\tilde{\mathcal{O}}(\sqrt{H^3SAT}+H^4SA$\\$+H^2S\sqrt{T})$}  & $\mathcal{O}(H(S+A)T)$&$\mathcal{O}(HS^2A)$\\

    % \thead{UCBVI \\ \cite{azar2017minimax}} & $\mathop{\min} \{\sqrt{SAH^3T}+S^2AH^3 , HT \}$  & NO\\
    % \thead{ORLC \\ \cite{dann2019policy}}& $\mathop{\min} \{\sqrt{SAH^3T}+S^2AH^4 , HT \}$  & NO\\
    % \thead{EULER \\ \cite{zanette2019tighter}} & $\mathop{\min} \{\sqrt{SAH^3T}+S^{\frac{3}{2}}AH^3(\sqrt{S}+\sqrt{H}) , HT \}$  & NO\\
    % \thead{UCB-Adv \\ \cite{zhang2020almost}}& $\mathop{\min} \{\sqrt{SAH^3T}+S^2A^{\frac{3}{2}}H^{\frac{33}{4}}T^{\frac{1}{4}} , HT \}$  & NO\\
    % \thead{UCBMQ \\ \cite{menard2021ucb}}& $\mathop{\min} \{\sqrt{SAH^3T}+SAH^4 , HT \}$  & NO\\
    % \thead{UCB-Q-Hoeffding \\ \cite{jin2018q}}& $\sqrt{H^4SAT} , HT \}$  & NO\\
    % \thead{Q-EarlySettled-Advantage \\ \cite{li2023breaking}} & $\mathop{\min} \{\sqrt{SAH^3T}+SAH^6 , HT \}$  & NO\\
    % \thead{MVP \\ \cite{zhang2024settling}} & $\mathop{\min} \{\sqrt{SAH^3T} , HT \}$  & NO\\
    % \thead{UCBMQ-GA \\ this work} & $\mathop{\min} \{\sqrt{SAH^3T}+SAH^4+H^2S\sqrt{T} , HT \}$  & YES\\
    \bottomrule
  \end{tabular}
\end{table}

\paragraph{Proof sketch} %of Theorem \ref{theorem_1}}

The proof is in five steps.
$\mathring{\overline{Q}}_h^t(s,a)$ and $\mathring{\overline{V}}_h^{t}(s)$ are defined in \ref{notation_appendix}.

\paragraph{Step 1: Upper bound $(\mathring{\overline{Q}}_h^t-Q_h^{\pi^{t+1}})(s,a)$} 
%For a state-action pair $(s,a)$, we upper bound the difference between $\mathring{\overline{Q}}_h^t$ and $Q_h^{\pi^{t+1}}$.
Combining the estimation of $\mathring{\overline{Q}}_h^t(s,a)$ in Lemma \ref{all_Q_estimate} and the upper-bound on $\beta_h^t(s,a)$ in Lemma \ref{beta_upper_bound}, we obtain the upper-bound of $(\mathring{\overline{Q}}_h^t-Q_h^{\pi^{t+1}})(s,a)$.

\paragraph{Step 2: Upper bound the local optimistic regret $\hat{R}_h^T(s,a)$}
We define the local optimistic regret as
\begin{align*}
\hat{R}_h^T(s,a)& \triangleq \sum_{t=0}^{T-1} \chi_h^{t+1}(s,a)(\mathring{\overline{Q}}_h^t-Q_h^{\pi^{t+1}})(s,a)\\
&=\sum_{t=t(s)-1}^{T-1} \chi_h^{t+1}(s,a)(\mathring{\overline{Q}}_h^t-Q_h^{\pi^{t+1}})(s,a).
\end{align*}
With the result in \textbf{Step 1}, we can decompose $\hat{R}_h^T(s,a)$ and upper bound each term.

\paragraph{Step 3: Replace $\chi_h^t$ with $\overline{p}_h^t$ in the upper-bound on $\hat{R}_h^T(s,a)$}
Following \textbf{Step 2}, we modify the upper bound on the local optimistic regret using $\overline{p}_h^t(s,a)$, the probability to reach $(s,a)$ at time step $h$ in the episode $t$ as below,
\begin{align*}
\hat{R}_h^T(s,a) \le &63log(T)\sqrt{\zeta \sum_{t=t(s)-1}^{T-1} \overline{p}_h^{t+1}(s,a)Var_{p_h}(V_{h+1}^{\pi^{t+1}})(s,a)} \nonumber\\
&+1754H^3log(T)^2\zeta\nonumber\\
&+(1+\frac{83}{H}) \sum_{t=t(s)-1}^{T-1}\overline{p}_h^{t+1}(s,a) p_h(\mathring{\overline{V}}_{h+1}^t-V_{h+1}^{\pi^{t+1}})(s,a).
\end{align*}

\paragraph{Step 4: Upper bound the regret $\hat{R}_h^T$ at step $h$}
We define the step $h$ regret as
$
\hat{R}_h^T \triangleq \sum_{s\in S}\sum_{t=0}^{T-1}\overline{p}_h^{t+1}(s)(\mathring{\overline{V}}_h^{t}-V_h^{\pi^{t+1}})(s)$.
For each $s$, the composition of the step $h$ regret follows by $t(s)$.
Employing the Cauchy-Schwarz inequality, we get the upper bound.

\paragraph{Step 5: Upper bound the regret $\mathcal{R}^T$}
With all the results in the previous steps, we can eventually prove that the regret $\mathcal{R}^T$ can be upper-bounded by $\hat{R}_h^T$.
%Eventually, we upper bound the regret $\mathcal{R}^T$ by bounding $\hat{R}_h^T$}.

\subsection{Computational Complexity and Space Complexity}

We then prove the computational complexity of employing UCBMQ-GA to train EMDP-GA as below.

\begin{theorem}
    \label{computation_lemma}
The computational complexity of employing UCBMQ-GA to train EMDP-GA is of order $\mathcal{O}(H(S+A)T)$, {where $H$ is the horizon, $S$ is the size of the state space, $A$ is the size of the action space and $T$ is the number of learning episodes}.
\end{theorem}
% The space complexity is $\mathcal{O}(HS^2A)$ since we need to store the all the functions, the same as UCBMQ.
The proof is given in \ref{complexity_proof}. Here, we present an intuitive sketch. 
{The updates of $\overline{Q}_h^t$, $\overline{V}_h^t$ and $V_{h,s,a}^t$ are carried out in an online manner.
At time step $h$ of the episode $t$, if $(s_h^t,a_h^t) \neq (s,a)$, both learning rate $\alpha_H^t(s,a)$ and momentum term $\gamma_h^t(s,a)$ become zero.
Thus, we do not need to update the functions on such $(s,a)$ pairs.
Furthermore, the expansions of $Q^{t-1}_h$, $\overline{V}^{t-1}_h$, and $V^{t-1}_{h,s,a}$ in each episode depend on their corresponding averaged values on $\mathcal{S}_{t-1}$.
}

In parallel, we have the following result for space complexity.

\begin{theorem}
\label{space_lemma}
    The space complexity of employing UCBMQ-GA to train EMDP-GA is $\mathcal{O}(HS^2A)$. 
\end{theorem}

A detailed proof is presented in  \ref{complexity_proof}. 
Intuitively, this lemma holds because we need to store all the bias-value function $V_{h,s,a}^t(s')$.

\paragraph{Comparison with existing results} 
% \fh{review existing results}
{
As shown in Table \ref{comparisons}, the computational complexity of UCBMQ-GA is of the same order as our baseline UCBMQ,
%Compared with the SOTA, the computational complexity of UCBMQ-GA is 
smaller than MVP ($\mathcal{O}(HSAT)$) and Bayes-UCBVI ($\mathcal{O}(BHS^2AT)$), %using $B$ Monte-Carlo samples to approximate one quantile, 
and a bit larger than UCB-Adv ($\mathcal{O}(HAT)$).
{The space complexity of UCBMQ-GA is of the same order as UCBMQ, MVP, and Bayes-UCBVI, and still a bit larger than that of UCB-Adv.} 
However, the complexity advances of UCB-Adv are in cost of regret, especially when  %the regret bound of UCBMQ-GA is smaller that of UCB-Adv when 
state space size $S$ and action space size $A$ are large compared with training time $T$, which is the case in an under-explored environment.
% There is a family of algorithms based Upper Confidence Bound (UCB): UCBVI \cite{azar2017minimax}, UCB-Adv \cite{zhang2020almost}, UCBMQ \cite{menard2021ucb}, Bayes-UCBVI \cite{tiapkin2022dirichlet} and LSVI-UCB \cite{jin2023provably}
% These approaches follow the same idea of the upper confidence bound but differ in computational complexity.
% We can conclude that the computational complexity of UCBMQ-GA is of the SOTA.
%\begin{remark}
% The computational complexity of UCBMQ-GA is smaller than $\mathcal{O}(BHS^2AT)$ for Bayes-UCBVI using $B$ Monte-Carlo samples to approximate one quantile, but larger than $\mathcal{O}(HAT)$ for UCB-Adv and LSVI-UCB.
}
%\end{remark}

% the same as UCBMQ.

% There are two common criteria for classifying an approach as model-free or model-based.
% One common criterion is analysing how the approach updates the policy.
% UCBMQ-GA does not estimate the EMDP-GA model or he transition function $P$.
% Instead, UCBMQ-GA directly approximates the value functions $Q$ and $V$ without the estimation about the model and updates the policy based on $Q$.
% Consequently, UCBMQ-GA falls into the model-free category.
% The other criterion is related to the space complexity.
% The space complexity of UCBMQ-GA is the same size of $HS^2A$.
% According to the definition of Jin et al. \cite{jin2018q}, UCBMQ-GA can therefore be classified as model-based approach.

% \paragraph[Model-free or model-based?]
% The space complexity of UCBMQ-GA is the same size of $HS^2A$.
% According to the definition by Jin et al. \cite{jin2018q}, UCBMQ-GA can therefore be classified as a model-based approach.
% However, UCBMQ-GA does not estimate the EMDP-GA model or he transition function $P$.
% Consequently, UCBMQ-GA falls into the model-free category.

%% file: 8.conclusion.tex
\section{Conclusion}
\label{sec::conclusion}
% In this paper, we propose a novel model, EMDP-GA, to model the growing awareness and unknown unknowns in reinforcement learning.
% In each episode, the agent selects an action from the action space based on the value functions defined on the aware domain.
% The aware domain expands when the agent reaches a new state.
% UCBMA-GA employs NIVE to expand the value functions when the aware domain expands, and uses a similar way in UCBMQ to train the agent.
% We find that NIVE leads to a decrease in aware confidence and the regret of UCBMQ-GA is consistent with SOTA.
% Moreover, our results raise several interesting open questions for further research, such as exactly model-based or model-free algorithms for EMDP-GA and the settings with a continuous state space or action space.

We mathematically ground the concept of unknown unknowns in reinforcement learning through our proposed episodic Markov decision process with growing awareness (EMDP-GA).
By expanding value functions $Q$ and $V$ to newly encountered states with the noninformative value expansion (NIVE) method, our approach effectively addresses the challenge of unknown unknowns.
We adapt upper confidence bound momentum Q-learning (UCBMQ) to train the EMDP-GA model at an affordable cost, achieving a regret bound competitive with existing methods not exposed to unknown unknowns.

\paragraph{Applicability} This paper focuses on problems with finite state space and action space. Future works include extending it to %Future research directions include designing an algorithm for EMDP-GA obtaining a lower regret bound and studying the growing awareness in the settings with 
continuous state and action spaces.

%% file: elsarticle-template-num.bbl
\begin{thebibliography}{10}

\bibitem{agrawal2024optimistic}
Priyank Agrawal and Shipra Agrawal.
\newblock {Optimistic Q-learning} for average reward and episodic reinforcement learning.
\newblock {\em arXiv preprint arXiv:2407.13743}, 2024.

\bibitem{azar2017minimax}
Mohammad~Gheshlaghi Azar, Ian Osband, and R{\'e}mi Munos.
\newblock Minimax regret bounds for reinforcement learning.
\newblock In {\em International Conference on Machine Learning}, pages 263--272. PMLR, 2017.

\bibitem{becker2020reverse}
Christoph~K Becker, Tigran Melkonyan, Eugenio Proto, Andis Sofianos, and Stefan~T Trautmann.
\newblock {Reverse Bayesianism}: Revising beliefs in light of unforeseen events.
\newblock 2020.

\bibitem{chakravarty2022reverse}
Surajeet Chakravarty, David Kelsey, and Joshua~C Teitelbaum.
\newblock {Reverse Bayesianism} and act independence.
\newblock {\em Journal of Economic Theory}, 203:105495, 2022.

\bibitem{chang2020decentralized}
Michael Chang, Sid Kaushik, S~Matthew Weinberg, Tom Griffiths, and Sergey Levine.
\newblock Decentralized reinforcement learning: Global decision-making via local economic transactions.
\newblock In {\em International Conference on Machine Learning}, pages 1437--1447. PMLR, 2020.

\bibitem{dann2015sample}
Christoph Dann and Emma Brunskill.
\newblock Sample complexity of episodic fixed-horizon reinforcement learning.
\newblock {\em Annual Conference on Neural Information Processing Systems}, 28, 2015.

\bibitem{dann2019policy}
Christoph Dann, Lihong Li, Wei Wei, and Emma Brunskill.
\newblock Policy certificates: Towards accountable reinforcement learning.
\newblock In {\em International Conference on Machine Learning}, pages 1507--1516. PMLR, 2019.

\bibitem{dann2021provably}
Christoph Dann, Mehryar Mohri, Tong Zhang, and Julian Zimmert.
\newblock A provably efficient model-free posterior sampling method for episodic reinforcement learning.
\newblock {\em Annual Conference on Neural Information Processing Systems}, 34:12040--12051, 2021.

\bibitem{domingues2020regret}
Omar~D Domingues, Pierre M{\'e}nard, Matteo Pirotta, Emilie Kaufmann, and Michal Valko.
\newblock Regret bounds for kernel-based reinforcement learning.
\newblock In {\em International Conference on Machine Learning}. PMLR, 2020.

\bibitem{domingues2021episodic}
Omar~Darwiche Domingues, Pierre M{\'e}nard, Emilie Kaufmann, and Michal Valko.
\newblock Episodic reinforcement learning in finite mdps: Minimax lower bounds revisited.
\newblock In {\em International Conference on Algorithmic Learning Theory}, pages 578--598. PMLR, 2021.

\bibitem{efroni2019tight}
Yonathan Efroni, Nadav Merlis, Mohammad Ghavamzadeh, and Shie Mannor.
\newblock Tight regret bounds for model-based reinforcement learning with greedy policies.
\newblock {\em Annual Conference on Neural Information Processing Systems}, 32, 2019.

\bibitem{fiechter1994efficient}
Claude-Nicolas Fiechter.
\newblock Efficient reinforcement learning.
\newblock In {\em Annual Conference on Computational Learning Theory}, pages 88--97, 1994.

\bibitem{grant2015preference}
Simon Grant and John Quiggin.
\newblock A preference model for choice subject to surprise.
\newblock {\em Theory and Decision}, 79:167--180, 2015.

\bibitem{karni2021reverse}
Edi Karni, Quitz{\'e} Valenzuela-Stookey, and Marie-Louise Vier{\o}.
\newblock Reverse bayesianism: A generalization.
\newblock {\em The B.E. Journal of Theoretical Economics}, 21(2):557--569, 2021.

\bibitem{karni2013reverse}
Edi Karni and Marie-Louise Vier{\o}.
\newblock {“Reverse Bayesianism”}: A choice-based theory of growing awareness.
\newblock {\em American Economic Review}, 103(7):2790--2810, 2013.

\bibitem{karni2015probabilistic}
Edi Karni and Marie-Louise Vier{\o}.
\newblock Probabilistic sophistication and reverse bayesianism.
\newblock {\em Journal of Risk and Uncertainty}, 50:189--208, 2015.

\bibitem{karni2017awareness}
Edi Karni and Marie-Louise Vier{\o}.
\newblock Awareness of unawareness: A theory of decision making in the face of ignorance.
\newblock {\em Journal of Economic Theory}, 168:301--328, 2017.

\bibitem{kaufmann2021adaptive}
Emilie Kaufmann, Pierre M{\'e}nard, Omar~Darwiche Domingues, Anders Jonsson, Edouard Leurent, and Michal Valko.
\newblock Adaptive reward-free exploration.
\newblock In {\em International Conference on Algorithmic Learning Theory}, pages 865--891. PMLR, 2021.

\bibitem{kochov2018behavioral}
Asen Kochov.
\newblock A behavioral definition of unforeseen contingencies.
\newblock {\em Journal of Economic Theory}, 175:265--290, 2018.

\bibitem{li2023breaking}
Gen Li, Laixi Shi, Yuxin Chen, and Yuejie Chi.
\newblock Breaking the sample complexity barrier to regret-optimal model-free reinforcement learning.
\newblock {\em Information and Inference-A Journal of the IMA}, 12(2):969--1043, 2023.

\bibitem{li2024breaking}
Gen Li, Yuting Wei, Yuejie Chi, and Yuxin Chen.
\newblock Breaking the sample size barrier in model-based reinforcement learning with a generative model.
\newblock {\em Operations Research}, 72(1):203--221, 2024.

\bibitem{liu2021finrl}
Xiao-Yang Liu, Hongyang Yang, Jiechao Gao, and Christina~Dan Wang.
\newblock {FinRL}: Deep reinforcement learning framework to automate trading in quantitative finance.
\newblock In {\em International Conference on AI in Finance}, pages 1--9. ACM, 2021.

\bibitem{mao2021near}
Weichao Mao, Kaiqing Zhang, Ruihao Zhu, David Simchi-Levi, and Tamer Basar.
\newblock Near-optimal model-free reinforcement learning in non-stationary episodic mdps.
\newblock In {\em International Conference on Machine Learning}, pages 7447--7458. PMLR, 2021.

\bibitem{menard2021ucb}
Pierre M{\'e}nard, Omar~Darwiche Domingues, Xuedong Shang, and Michal Valko.
\newblock {UCB Momentum Q-learning: Correcting the bias without forgetting}.
\newblock In {\em International Conference on Machine Learning}, pages 7609--7618. PMLR, 2021.

\bibitem{moerland2023model}
Thomas~M Moerland, Joost Broekens, Aske Plaat, Catholijn~M Jonker, et~al.
\newblock Model-based reinforcement learning: A survey.
\newblock {\em Foundations and Trends{\textregistered} in Machine Learning}, 16(1):1--118, 2023.

\bibitem{neu2012adversarial}
Gergely Neu, Andras Gyorgy, and Csaba Szepesv{\'a}ri.
\newblock The adversarial stochastic shortest path problem with unknown transition probabilities.
\newblock In {\em International Conference on Artificial Intelligence and Statistics}, pages 805--813. PMLR, 2012.

\bibitem{press}
United States~Department of~Defense.
\newblock {Defense.gov Transcript: DoD News Briefing - Secretary Rumsfeld and Gen. Myers}, 2002.

\bibitem{piermont2017introspective}
Evan Piermont.
\newblock Introspective unawareness and observable choice.
\newblock {\em Games and Economic Behavior}, 106:134--152, 2017.

\bibitem{shreve1978alternative}
Steven~E Shreve and Dimitri~P Bertsekas.
\newblock Alternative theoretical frameworks for finite horizon discrete-time stochastic optimal control.
\newblock {\em SIAM Journal on Control and Optimization}, 16(6):953--978, 1978.

\bibitem{sutton1998reinforcement}
Richard~S Sutton, Andrew~G Barto, et~al.
\newblock {\em Reinforcement learning: An introduction}, volume~1.
\newblock MIT press Cambridge, 1998.

\bibitem{tiapkin2022dirichlet}
Daniil Tiapkin, Denis Belomestny, Eric Moulines, Alexey Naumov, Sergey Samsonov, Yunhao Tang, Michal Valko, and Pierre M{\'e}nard.
\newblock From dirichlet to rubin: Optimistic exploration in rl without bonuses.
\newblock In {\em International Conference on Machine Learning}, pages 21380--21431. PMLR, 2022.

\bibitem{tibshirani1989noninformative}
Robert Tibshirani.
\newblock Noninformative priors for one parameter of many.
\newblock {\em Biometrika}, 76(3):604--608, 1989.

\bibitem{van2018autonomous}
Jessica Van~Brummelen, Marie O’brien, Dominique Gruyer, and Homayoun Najjaran.
\newblock Autonomous vehicle perception: The technology of today and tomorrow.
\newblock {\em Transportation Research Part C: Emerging Technologies}, 89:384--406, 2018.

\bibitem{viero2021intertemporal}
Marie-Louise Vier{\o}.
\newblock An intertemporal model of growing awareness.
\newblock {\em Journal of Economic Theory}, 197:105351, 2021.

\bibitem{white1989markov}
Chelsea~C White~III and Douglas~J White.
\newblock Markov decision processes.
\newblock {\em European Journal of Operational Research}, 39(1):1--16, 1989.

\bibitem{zanette2019tighter}
Andrea Zanette and Emma Brunskill.
\newblock Tighter problem-dependent regret bounds in reinforcement learning without domain knowledge using value function bounds.
\newblock In {\em International Conference on Machine Learning}, pages 7304--7312. PMLR, 2019.

\bibitem{zednik2016bayesian}
Carlos Zednik and Frank J{\"a}kel.
\newblock Bayesian reverse-engineering considered as a research strategy for cognitive science.
\newblock {\em Synthese}, 193:3951--3985, 2016.

\bibitem{zhang2024settling}
Zihan Zhang, Yuxin Chen, Jason~D Lee, and Simon~S Du.
\newblock Settling the sample complexity of online reinforcement learning.
\newblock In {\em Annual Conference on Learning Theory}, pages 5213--5219. PMLR, 2024.

\bibitem{zhang2021reinforcement}
Zihan Zhang, Xiangyang Ji, and Simon Du.
\newblock Is reinforcement learning more difficult than bandits? a near-optimal algorithm escaping the curse of horizon.
\newblock In {\em Annual Conference on Learning Theory}, pages 4528--4531. PMLR, 2021.

\bibitem{zhang2020almost}
Zihan Zhang, Yuan Zhou, and Xiangyang Ji.
\newblock Almost optimal model-free reinforcement learningvia reference-advantage decomposition.
\newblock {\em Annual Conference on Neural Information Processing Systems}, 33:15198--15207, 2020.

\bibitem{zimin2013online}
Alexander Zimin and Gergely Neu.
\newblock Online learning in episodic markovian decision processes by relative entropy policy search.
\newblock {\em Annual Conference on Neural Information Processing Systems}, 26, 2013.

\end{thebibliography}
